\documentclass{article}

\usepackage{arxiv}
\usepackage{amsmath}
\usepackage[utf8]{inputenc}
\usepackage[T1]{fontenc}
\usepackage{url}
\usepackage{booktabs}
\usepackage{amsfonts}
\usepackage{nicefrac}
\usepackage{microtype}
\usepackage{graphicx}
\usepackage{natbib}
\setcitestyle{authoryear,round,citesep={;},aysep={,},yysep={;}}
\usepackage{amssymb}
\usepackage{amsthm}

\usepackage{bm}
\usepackage{array}
\usepackage{longtable}
\usepackage{float}
\usepackage{xcolor}
\usepackage{hyperref}

\newtheoremstyle{paperstatementstyle}
  {\medskipamount}{\medskipamount}{\itshape}{}{\bfseries}{.}{0.5em}
  {\thmname{#1}\thmnumber{ #2}\thmnote{\normalfont\space(#3)}}
\theoremstyle{paperstatementstyle}
\newtheorem{theorem}{Theorem}
\newtheorem{definition}{Definition}
\newtheorem{lemma}{Lemma}
\newtheorem{corollary}{Corollary}
\newtheorem{proposition}{Proposition}
\newtheorem{assumption}{Assumption}
\newtheorem{condition}{Condition}
\newcommand{\restatedtheoremlabel}{}
\newtheorem*{theoremrestatementinner}{Theorem~\ref{\restatedtheoremlabel}}
\newenvironment{theoremrestatement}[2][]{\def\restatedtheoremlabel{#2}\begin{theoremrestatementinner}[#1]}{\end{theoremrestatementinner}}
\newcommand{\restatedcorollarylabel}{}
\newtheorem*{corollaryrestatementinner}{Corollary~\ref{\restatedcorollarylabel}}
\newenvironment{corollaryrestatement}[2][]{\def\restatedcorollarylabel{#2}\begin{corollaryrestatementinner}[#1]}{\end{corollaryrestatementinner}}
\newcommand{\restatedpropositionlabel}{}
\newtheorem*{propositionrestatementinner}{Proposition~\ref{\restatedpropositionlabel}}

\newtheoremstyle{paperproseformalstyle}
  {\medskipamount}{\medskipamount}{\normalfont}{}{\bfseries}{.}{0.5em}
  {\thmname{#1}\thmnumber{ #2}\thmnote{\normalfont\space(#3)}}
\theoremstyle{paperproseformalstyle}
\newtheorem{remark}{Remark}
\newtheorem{procedure}{Procedure}
\theoremstyle{paperstatementstyle}
\newtheorem{paperrule}{Rule}
\newenvironment{resultparts}[1][roman]{\begin{enumerate}\renewcommand{\theenumi}{\csname #1\endcsname{enumi}}}{\end{enumerate}}
\title{A Theory of Reliable Self-Evolution for Agent Harnesses}
\date{}

\renewcommand{\thefootnote}{$\dagger$}
\author{Qianshu Cai\textsuperscript{1,2}\enspace
  Yonggang Zhang\textsuperscript{2}\thanks{Corresponding authors.}\enspace
  Jun Nie\textsuperscript{1,3}\enspace
  Maohao Ran\textsuperscript{3} \\
  \bfseries\rule{0pt}{24pt}Huajiang Zheng\textsuperscript{2}\enspace
  Jun Song\textsuperscript{3}\enspace
  Xinmei Tian\textsuperscript{1}\enspace
  Yike Guo\textsuperscript{2}\enspace
  Wei Xue\textsuperscript{2}\footnotemark[1] \\[6pt]
  \textsuperscript{1}University of Science and Technology of China \\
  \textsuperscript{2}Hong Kong Generative AI Research \& Development Center; \\
  The Hong Kong University of Science and Technology \\
  \textsuperscript{3}Hong Kong Baptist University
}

\hypersetup{
  hidelinks,
  pdftitle={A Theory of Reliable Self-Evolution for Agent Harnesses},
  pdfsubject={cs.LG},
  pdfkeywords={Harness self-evolution, reachability, modification generation, validation, measurement, reliable adoption, user-task distribution},
}
\hypersetup{pdfauthor={Qianshu Cai, Yonggang Zhang, Jun Nie, Maohao Ran, Huajiang Zheng, Jun Song, Xinmei Tian, Yike Guo, Wei Xue}}

\begin{document}
\addtocontents{toc}{\protect\setcounter{tocdepth}{-1}}
\maketitle

\renewcommand{\thefootnote}{\arabic{footnote}}
\setcounter{footnote}{0}

\begin{abstract}
In harness self-evolution, agents modify their own prompts, code, tools, and orchestration while keeping the underlying language model fixed. Recent work has shown that agents can improve themselves in response to task failures and achieve substantial performance gains. However, gains on failed tasks do not automatically ensure that performance on previously successful tasks is preserved, raising concerns about reliable adoption.
In this work, we provide a theoretically grounded condition under which a self-evolved harness can be reliably adopted.
We then propose a validation rule to make the evolved system satisfy the reliable adoption condition with theoretical guarantees.
Consequently, the system can achieve progressive improvement through evolution.
This leads to a natural question: Does reliable self-evolution have a performance ceiling, and which factors govern this ceiling? Our theoretical results show that the performance ceiling is determined by the costs of verification and evaluation.
On the other hand, self-evolution may stall in practice. In this regard, we show that experimental evidence on agent performance shifts can be used to identify the sources of stagnation.
Our work thus establishes a theoretical framework for understanding and advancing reliable harness self-evolution.
\end{abstract}

\section{Introduction}
\label{sec:introduction}

Harness self-evolution allows an agent to modify its prompts, code, tools, and orchestration in response to task feedback while keeping the underlying language model fixed. These modifications persist across subsequent tasks. G{\"o}del Agent \citep{yin2024godel} and SICA \citep{robeyns2025sica} demonstrated that agents can autonomously rewrite their own logic and code to improve task performance. DGM \citep{zhang2025dgm} uses evidence from failed task executions to improve its tools and workflows, achieving substantial performance gains. HGM \citep{wang2025hgm} advances search over successive agent modifications, obtaining further gains in coding performance. These advances motivate applying harness self-evolution in production systems. MOSS \citep{cai2026moss} brings source-level self-evolution to production-grade agent frameworks such as OpenClaw \citep{steinberger2026openclaw}, testing modifications against failures encountered during use. It requires user approval before deployment and supports rollback if the updated system fails health checks. Even with these safeguards, a concern about reliable adoption remains: the adopted agent may improve on failed tasks while performing worse on previously successful tasks. Recent systems also consider performance beyond the failed tasks that motivate a modification. HarnessX \citep{chen2026harnessx} includes checks for observed regressions on previously solved tasks. Self-Harness \citep{zhang2026selfharness} requires aggregate performance to improve on at least one of its development and held-out task sets without declining on either. We emphasize that a self-evolved harness can be reliably adopted when it improves expected reward over the full user-task distribution. We formulate a sufficient condition for this improvement in terms of gains on failed tasks and controlled changes in the expected rewards on previously successful tasks.

\begin{figure}[t]
\centering
\includegraphics[width=\textwidth]{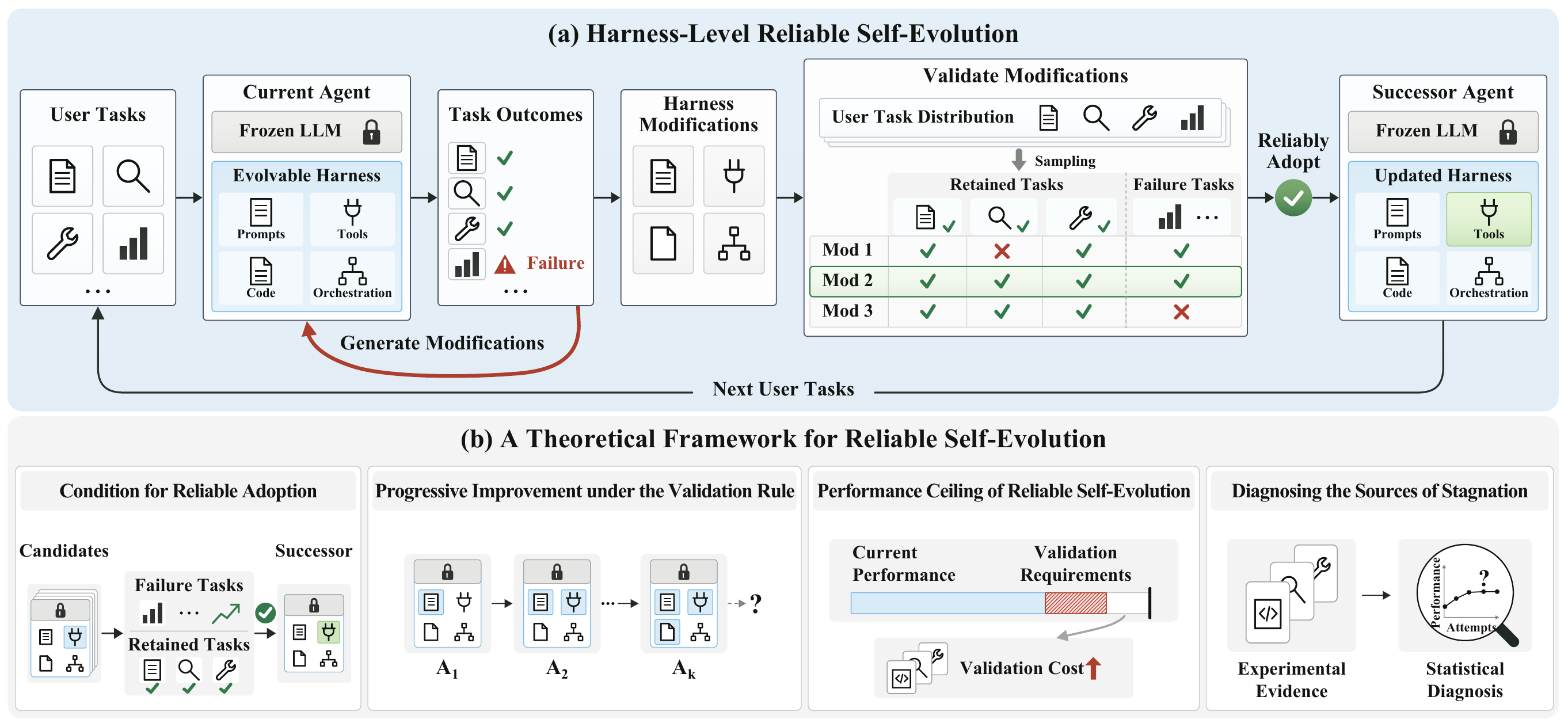}
\caption{\textbf{Harness-level reliable self-evolution and our theoretical framework.} \textbf{(a)} The harness self-evolution process (Section~\ref{sec:formulation}). \textbf{(b)} A condition for reliable adoption (Section~\ref{sec:qualified-modifications}), progressive improvement under the validation rule (Section~\ref{sec:certification}), the performance ceiling (Section~\ref{sec:limits}), and diagnosis of stagnation from experimental evidence (Section~\ref{sec:diagnosis}).}
\label{fig:harness-self-evolution}
\end{figure}

In the harness self-evolution process we study (Figure~\ref{fig:harness-self-evolution}(a)), the current agent generates harness modifications from task failures to produce a set of candidate agents. These candidates are evaluated on a validation set using validation rules. A candidate meeting the validation requirements is adopted as the updated agent and continues to handle user tasks.

In this work, we establish a theoretical framework for understanding and advancing reliable harness self-evolution (Figure~\ref{fig:harness-self-evolution}(b)). We first provide a theoretically grounded condition under which a self-evolved harness can be reliably adopted. We analyze how much improvement on failed tasks is sufficient, with controlled changes in performance on previously successful tasks, to improve expected reward over the full user-task distribution and thereby enable reliable adoption. We then propose a validation rule with theoretical guarantees for reliable adoption under this condition. We also bound the probability of achieving a reliable update through candidate generation and validation. The gains guaranteed for adopted updates accumulate across evolution, and we also bound the average magnitude of task-performance changes across the adopted sequence.

This raises a natural question: does reliable self-evolution have a performance ceiling, and what determines it? We identify a performance ceiling beyond which the proposed validation rule cannot reliably validate further improvement. However, further improvement may still be possible even when this ceiling is reached. We therefore propose an alternative validation rule that uses evidence about each candidate's losses on previously successful tasks and can reliably validate improvement beyond the original ceiling. These results show that the ceiling depends in part on how improvement is verified. We also show that, as expected reward approaches its upper bound, reliably validating smaller gains can require increasing evaluation cost. When reliable self-evolution stagnates in practice, we show how experimental evidence on candidate performance shifts relative to the current agent can help identify the sources of stagnation. We provide theoretical guarantees for estimating reachability---the probability of generating a modification that satisfies the improvement and retention requirements---from additional experimental data. Comparing this evidence with validation and adoption records helps distinguish a low probability of generating qualified modifications from insufficient validation evidence for adopting these modifications.

\section{Reliable Adoption of Self-Evolved Harnesses}\label{sec:safe-self-evolution}

\subsection{Harness Self-Evolution Process}\label{sec:formulation}

We study the following harness self-evolution process (Figure~\ref{fig:harness-self-evolution}(a)). Let \(\mathcal T_{\mathrm{user}}\) and \(\mathcal O\) be the user-task and output spaces, and \(\mathcal D_{\mathrm{user}}\) the user-task distribution. At step \(k\), the current agent \(A_k=(M,C_k)\) has a fixed language model \(M\) and an evolvable harness \(C_k\) of prompts, code, tools, and orchestration. On task \(t\sim\mathcal D_{\mathrm{user}}\), it produces \(o\sim A_k(\cdot\mid t)\). From observed failures and interaction history, the agent generates harness modifications. Let \(\mathcal M\) be the space of valid modifications; each generated \(c\in\mathcal M\) yields a candidate agent \(\widetilde A=A_k\oplus c\). The candidate agents are evaluated on a validation set; a validation rule identifies candidates with sufficient evidence for reliable adoption. Selecting a validated \(c\) gives \(A_{k+1}=A_k\oplus c\); otherwise, \(A_{k+1}=A_k\). The successor handles subsequent tasks and modification generation; Appendices~\ref{app:formal-setting} and~\ref{app:reachability-scope} formalize the assumptions and procedure.

For an agent \(A\) and fixed task reward \(r(t,o)\in[0,1]\), define \(V(A,t):=\mathbb E_{o\sim A(\cdot\mid t)}[r(t,o)]\) and \(J_{\mathrm{user}}(A):=\mathbb E_{t\sim\mathcal D_{\mathrm{user}}}[V(A,t)]\) as its expected rewards on task \(t\) and across user tasks, respectively. Reliable adoption requires \(J_{\mathrm{user}}(A_k\oplus c)>J_{\mathrm{user}}(A_k)\).

\subsection{A Sufficient Condition for Reliable Adoption}\label{sec:qualified-modifications}

To obtain a sufficient condition for reliable adoption, we decompose the change in overall expected reward into contributions from failed and previously successful tasks. Let the binary detector \(\phi(t,o)\in\{0,1\}\) identify a failed output (\(\phi=1\)), while \(r(t,o)\) scores task performance (Definition~\ref{def:failure-detector}, Appendix~\ref{app:failure-weighting}). For the current agent, define the task-level failure probability \(\psi_{A_k}(t):=\mathbb E_{o\sim A_k(\cdot\mid t)}[\phi(t,o)]\) and overall failure rate \(Z_k:=\mathbb E_{t\sim\mathcal D_{\mathrm{user}}}[\psi_{A_k}(t)]\). For \(0<Z_k<1\), the induced failure-task and retained-task distributions are
\begin{equation}
\mathcal D_{F,k}(dt)
:=\frac{\psi_{A_k}(t)}{Z_k}\,\mathcal D_{\mathrm{user}}(dt),
\qquad
\mathcal D_{R,k}(dt)
:=\frac{1-\psi_{A_k}(t)}{1-Z_k}\,\mathcal D_{\mathrm{user}}(dt).
\label{eq:formula-3-7}
\end{equation}
Because outputs can vary across runs, the distributions can overlap on the same task.

For candidate \(\widetilde A\), define the per-task expected-reward difference \(\mathrm{Adv}_{A_k}(\widetilde A,t):=V(\widetilde A,t)-V(A_k,t)\) and its failure-task contribution \(L_{A_k}(\widetilde A):=Z_k\mathbb E_{t\sim\mathcal D_{F,k}}[\mathrm{Adv}_{A_k}(\widetilde A,t)]\). By Lemma~\ref{lem:failure-decomposition} in Appendix~\ref{app:failure-weighting}, the overall expected-reward difference decomposes as
\begin{equation}
J_{\mathrm{user}}(\widetilde A)-J_{\mathrm{user}}(A_k)
=L_{A_k}(\widetilde A)
+(1-Z_k)\mathbb E_{t\sim\mathcal D_{R,k}}[\mathrm{Adv}_{A_k}(\widetilde A,t)].
\label{eq:formula-3-8}
\end{equation}
For retained tasks, let \(D_R(\widetilde A;A_k):=\mathbb E_{t\sim\mathcal D_{R,k}}|\mathrm{Adv}_{A_k}(\widetilde A,t)|\) be the mean absolute expected-reward change. For \(c\in\mathcal M\), write \(L_k(c):=L_{A_k}(A_k\oplus c)\) and \(D_k(c):=D_R(A_k\oplus c;A_k)\).

\begin{definition}[Qualified modification]
\label{def:qualified-modification}
For \(T=(\lambda,\delta)\), where \(\lambda>0\) sets the minimum failure-task contribution and \(\delta\ge0\) the maximum retained-task change, define the set of modifications qualified at step \(k\) by \(\mathcal P_k(T):=\{c\in\mathcal M:L_k(c)\ge\lambda,\ D_k(c)\le\delta\}\).
\end{definition}

\begin{theorem}[Gain bound for reliable adoption]
\label{thm:expected-reward-improvement}
Under Assumptions~\ref{ass:bounded-reward} and~\ref{ass:agent-independent-tasks}, every candidate agent \(\widetilde A\) with \(D_R(\widetilde A;A_k)\le\delta\) satisfies the expected-reward gain bound
\begin{equation}
J_{\mathrm{user}}(\widetilde A)-J_{\mathrm{user}}(A_k)
\ge L_{A_k}(\widetilde A)-(1-Z_k)\delta.
\label{eq:formula-3-18}
\end{equation}
\end{theorem}

\begin{corollary}[Sufficient condition for reliable adoption]
\label{cor:qualified-safe-improvement}
Under Assumptions~\ref{ass:bounded-reward} and~\ref{ass:agent-independent-tasks}, if \(T=(\lambda,\delta)\) satisfies \(\lambda>(1-Z_k)\delta\), then every \(c\in\mathcal P_k(T)\) satisfies
\[
J_{\mathrm{user}}(A_k\oplus c)-J_{\mathrm{user}}(A_k)
\ge\lambda-(1-Z_k)\delta>0.
\]
For \(\gamma>0\), setting \(\lambda=\lambda_k(\gamma,\delta):=\gamma+(1-Z_k)\delta\) ensures gain at least \(\gamma\) for every \(c\in\mathcal P_k(T)\).
\end{corollary}

The condition permits changes on retained tasks while ensuring that failure-task gains outweigh their possible losses. Appendix~\ref{app:proofs-decomposition} proves Theorem~\ref{thm:expected-reward-improvement} and Corollary~\ref{cor:qualified-safe-improvement}. Alternative retained-task conditions are in Appendices~\ref{app:tail-retention}--\ref{app:section-G-13} and~\ref{app:one-sided-deviations}. Section~\ref{sec:certification} develops a validation rule for this condition.

\section{Achieving Harness-Level Reliable Self-Evolution}
\label{sec:certification}

\subsection{A Validation Rule for Reliable Adoption}
\label{sec:task-stream-updates}

Assessing the reliable-adoption condition in Section~\ref{sec:qualified-modifications} requires evaluating the candidate's expected-reward improvement under the user-task distribution \(\mathcal D_{\mathrm{user}}\). In the harness self-evolution process of Section~\ref{sec:formulation}, we form a validation set from sampled user tasks and apply a validation rule to decide which generated modifications to adopt. We first give a rule with a high-probability guarantee for reliable adoption. We then bound the probability of achieving an improving update by generating a candidate pool and then validating its modifications.

For a generated modification \(c\in\mathcal M\), write its candidate agent as \(\widetilde A=A_k\oplus c\). Let \(\widehat L_k,\widehat D_R,\widehat Z_k\) be validation-set estimates of the failure-task contribution \(L_k(c)\), mean absolute retained-task change \(D_k(c)\), and current failure rate \(Z_k\), with error bounds \(\varepsilon_L,\varepsilon_D,\varepsilon_Z\). Fix thresholds \(\tau,\delta>0\) for the estimated failure-task contribution and retained-task change. Appendix~\ref{app:proofs-evaluation} gives the estimators and sampling conditions and derives the corresponding error bounds.

\begin{paperrule}[Two-Gate validation]
\label{alg:two-gate}
With a valid candidate and the required data, validate \(c\) if and only if
\begin{align}
\widehat L_k&\ge\tau,
\qquad \widehat D_R\le\delta,
\label{eq:formula-4-23}\\
\Delta_k&:=\tau-\varepsilon_L
-(1-\widehat Z_k+\varepsilon_Z)(\delta+\varepsilon_D)>0.
\label{eq:formula-4-28}
\end{align}
\end{paperrule}
The rule validates a modification only when the estimates support a sufficient failure-task contribution, controlled retained-task change, and a positive lower bound on overall expected-reward gain. In the one-candidate update, a validated \(c\) is adopted according to Section~\ref{sec:formulation}; Procedure~\ref{alg:finite-horizon-update} in Appendix~\ref{app:proofs-evaluation} specifies the validation procedure and the resulting agent update.

Let \(\mathsf{Acc}_k\) denote the event that the candidate is adopted, and \(I_k:=J_{\mathrm{user}}(A_{k+1})-J_{\mathrm{user}}(A_k)\) its true gain. Set \(\Delta_k=0\) when no valid candidate or required validation data are available, as in Definition~\ref{def:evolution-step} of Appendix~\ref{app:proofs-evaluation}. Let \(\beta_{\mathrm{step},k}\) denote the stepwise estimation failure bound (Eq.~\eqref{eq:step-risk}, Appendix~\ref{app:section-G-5}).
\begin{theorem}[Reliable adoption under Two-Gate validation]
\label{thm:one-step-guarantee}
Suppose Assumptions~\ref{ass:bounded-reward} and~\ref{ass:agent-independent-tasks} (Appendix~\ref{app:formal-setting}) and the evaluation conditions of Lemma~\ref{lem:improvement-estimation} (Appendix~\ref{app:proofs-evaluation}) and Lemma~\ref{lem:retained-change-estimation} (Appendix~\ref{app:section-G-5}) hold. Apply Rule~\ref{alg:two-gate} through the update in Procedure~\ref{alg:finite-horizon-update}. Then
\[
\mathbb P\bigl(\mathsf{Acc}_k\cap\{I_k<\Delta_k\}\bigr)
\le\beta_{\mathrm{step},k}.
\]
In particular, the positive-gain requirement gives \(\mathbb P(\mathsf{Acc}_k\cap\{I_k\le0\})\le\beta_{\mathrm{step},k}\).
\end{theorem}
Under the stated evaluation conditions, Rule~\ref{alg:two-gate} provides a positive lower bound \(\Delta_k\) on the gain from adoption: the probability of adopting a candidate whose true gain falls below this bound is at most \(\beta_{\mathrm{step},k}\). This guarantee applies the sufficient condition in Corollary~\ref{cor:qualified-safe-improvement}; its proof is in Appendix~\ref{app:section-G-5}. Remark~\ref{rem:stored-label-recollection} in Appendix~\ref{app:proofs-evaluation} explains why reference data must be refreshed after adoption.

\phantomsection\label{sec:certified-selection}\label{sec:reachability}
Having established the reliable-adoption guarantee for Rule~\ref{alg:two-gate}, we next quantify the probability that generation and validation yield a reliable update. Consider a pool of generated candidates validated against a target \(T=(\lambda,\delta)\) from Definition~\ref{def:qualified-modification}. Fix the interaction history \(\mathcal H_k\) and a generation context \(\Xi_k=\xi\) shared across generation runs. A complete run of the current agent's generation procedure returns \(C\sim\pi_{k,\xi}\) on \(\mathcal M_\bot:=\mathcal M\cup\{\bot\}\), where \(\bot\) means that no valid modification was returned.
\begin{definition}[Reachability]
\label{def:reachability}
A generation run produces a modification qualified for \(T\) with probability
\begin{equation}
P_k(T\mid\mathcal H_k,\xi)
:=\pi_{k,\xi}\bigl(\mathcal P_k(T)\bigr).
\label{eq:formula-3-25}
\end{equation}
\end{definition}

Let \(C_1,\ldots,C_N\) be the outputs of \(N\) conditionally independent generation runs under the same history and context. For each valid \(C_i\), compute confidence bounds for \(L_k(C_i)\) and \(D_k(C_i)\). Let the validation indicator \(G_i(T)=1\) if and only if the lower bound for \(L_k(C_i)\) reaches \(\lambda\) and the upper bound for \(D_k(C_i)\) is at most \(\delta\); these tests define the pool validation rule \(G\). Appendix~\ref{app:proofs-pool} specifies the bounds and the convention for invalid outputs.

Candidates whose true values meet the target after allowing for the confidence-interval widths pass validation whenever the intervals contain the true values. Before drawing candidates, fix deterministic \(\bar w_L,\bar w_D\) that bound the interval widths for all candidates and possible validation data. For evaluation sample sizes \(\mathbf n\) and error allocation \(\beta\), define the set of modifications
\begin{equation}
\mathcal P^+_{k,N,\mathbf n,\beta}(T)
:=\{c\in\mathcal M:
L_k(c)\ge\lambda+\bar w_L,\ 
D_k(c)\le\delta-\bar w_D\}.
\label{eq:formula-4-7}
\end{equation}
Write \(P_k^+(T\mid\mathcal H_k,\xi):=\pi_{k,\xi}(\mathcal P^+_{k,N,\mathbf n,\beta}(T))\) for the probability of generating a modification in this set, abbreviated as \(P_k^+(T)\) when the history and context are fixed. Let \(\beta_{\mathrm{stat}}\) bound the probability that any validation interval fails to contain the quantity it estimates.

\begin{theorem}[Probability of achieving harness-level reliable self-evolution]
\label{thm:certified-selection}
Fix \(T\) and \((\mathcal H_k,\Xi_k=\xi)\). Suppose the \(N\) generation runs are conditionally i.i.d. and all validation intervals cover their true values with conditional probability at least \(1-\beta_{\mathrm{stat}}\) given the pool. Suppose the evaluation data and interval widths satisfy the conditions in Appendix~\ref{app:proofs-pool}. Let a selector \(\sigma\) choose only validated candidates and choose one whenever any are available. Write \(R^{G,\sigma}_{k,N}(T\mid\mathcal H_k,\xi)\) for the conditional probability that the update adopts a modification qualified for \(T\). Then
\begin{equation}
(1-\beta_{\mathrm{stat}})
\bigl[1-(1-P_k^+(T\mid\mathcal H_k,\xi))^N\bigr]
\le R^{G,\sigma}_{k,N}(T\mid\mathcal H_k,\xi)
\le1-(1-P_k(T\mid\mathcal H_k,\xi))^N.
\label{eq:formula-4-9}
\end{equation}
In particular, for \(T=(\lambda_k(\gamma,\delta),\delta)=(\gamma+(1-Z_k)\delta,\delta)\) with \(\gamma>0\) as in Corollary~\ref{cor:qualified-safe-improvement}, the same lower bound applies to the probability of an update with overall expected-reward gain at least \(\gamma\).
\end{theorem}
For the calibrated target, \(P_k^+(T)>0\) and \(\beta_{\mathrm{stat}}<1\) establish a positive probability of an update with at least \(\gamma\) gain. Candidate generation also limits reliable adoption: validation can adopt a \(T\)-qualified modification only if the pool contains one. Appendix~\ref{app:proofs-pool} gives the proof and shows how to choose the target using an estimate of \(Z_k\) obtained before generation (Corollary~\ref{cor:pilot-certified-selection}). Proposition~\ref{prop:self-task-separation} in Appendix~\ref{app:proofs-generation} analyzes generation-dependent reachability. Proposition~\ref{prop:expected-selected-improvement} and Corollaries~\ref{cor:pool-confidence-allocation} and~\ref{cor:context-averaging} in Appendix~\ref{app:proofs-pool} address expected selected gain, confidence allocation, and shared context, respectively.

\subsection{Progressive Improvement Guarantees}
\label{sec:multiple-updates}
\label{sec:multistep-gains}

When Rule~\ref{alg:two-gate} is applied across \(K\) evolution steps, each adopted successor becomes the reference agent for subsequent generation and validation. We ask whether guaranteed gains from adoptions accumulate despite dependence between steps. Fix the reward and user-task distribution, run Procedure~\ref{alg:finite-horizon-update} at each step to obtain \(A_0,\ldots,A_K\), and let \(\mathcal A:=\{k<K:\mathsf{Acc}_k\}\) denote the adopted steps.
\begin{theorem}[Progressive improvement through evolution]
\label{thm:finite-run-guarantee}
Suppose the hypotheses of Theorem~\ref{thm:one-step-guarantee} hold at every step and the deterministic error bounds satisfy \(\sum_{k=0}^{K-1}\beta_{\mathrm{step},k}\le\beta_{\mathrm{run}}\). Then, with probability at least \(1-\beta_{\mathrm{run}}\), the adopted updates satisfy
\begin{equation}
J_{\mathrm{user}}(A_K)-J_{\mathrm{user}}(A_0)
\ge\sum_{k\in\mathcal A}\Delta_k.
\label{eq:formula-6-16}
\end{equation}
The bound does not require independence between evolution steps.
\end{theorem}
Each \(\Delta_k\) is computed before the corresponding adoption. Under the stated evaluation conditions and the overall error bound, the positive gains guaranteed for adopted updates accumulate, even when later steps depend on earlier ones. Appendix~\ref{app:proofs-multistep} gives the proof, the extension to the Measured-margin rule introduced in Section~\ref{sec:limits}, and the adopted-step bound under a uniform positive gain (Corollary~\ref{cor:accepted-step-count}). When user-task distributions change, Theorem~\ref{thm:drifting-task-guarantee} in Appendix~\ref{app:section-I-4-1} deducts cumulative total-variation drift; Theorem~\ref{thm:observation-lag-cumulative} in Appendix~\ref{app:section-G-15} addresses evaluation--deployment distribution mismatch.

An increase in overall expected reward does not describe how task-level performance changes along the sequence. We therefore bound the total change across adopted updates and the final agent's change relative to \(A_0\), both averaged over \(\mathcal D_{\mathrm{user}}\). For the fixed user-task distribution, let
\[
\mathcal V_K:=\sum_{k\in\mathcal A}
\mathbb E_{t\sim\mathcal D_{\mathrm{user}}}
|V(A_{k+1},t)-V(A_k,t)|
\]
be the accumulated mean absolute change along the adopted updates.
\begin{proposition}[Accumulated task-performance change]
\label{prop:cumulative-retained-change}
Under Assumptions~\ref{ass:bounded-reward} and~\ref{ass:agent-independent-tasks}, suppose \(Z_k\in(0,1)\) on adopted steps and each adoption satisfies the Two-Gate inequalities in Eq.~\eqref{eq:formula-4-23} with fixed \(\tau,\delta>0\). Write \(\varepsilon_L^{(k)},\varepsilon_D^{(k)}\) for the step-\(k\) error bounds. On the event that these bounds hold,
\begin{equation}
\begin{aligned}
\mathcal V_K\le{}&1-J_{\mathrm{user}}(A_0)
+\sum_{k\in\mathcal A}(Z_k-\tau+\varepsilon_L^{(k)})\\
&+2\sum_{k\in\mathcal A}(1-Z_k)(\delta+\varepsilon_D^{(k)}).
\end{aligned}
\label{eq:cumulative-change-basic}
\end{equation}
Consequently,
\(\mathbb E_{t\sim\mathcal D_{\mathrm{user}}}
|V(A_K,t)-V(A_0,t)|\le\min\{1,\mathcal V_K\}\).
\end{proposition}
The proposition also bounds the average magnitude of changes in per-task expected reward relative to \(A_0\), counting both gains and losses along the adopted sequence. Appendix~\ref{app:cumulative-change} proves the proposition and gives the corresponding result for the retained-task decrease rule in Eq.~\eqref{eq:one-sided-adoption-rule} of Appendix~\ref{app:one-sided-cumulative-rule}.

Further progress requires subsequent agents to generate modifications that can be validated. Proposition~\ref{prop:no-later-step-guarantee} in Appendix~\ref{app:later-step-proof} constructs a successful first update after which later reachability vanishes. Section~\ref{sec:limits} next examines the performance ceiling and its dependence on verification and evaluation.

\section{Performance Ceiling of Reliable Self-Evolution}
\label{sec:limits}

Rule~\ref{alg:two-gate} requires failure-task improvement to cover possible retained-task losses and evaluation error, while bounded rewards constrain the improvement available. We characterize this ceiling, then examine how verification can support further reliable improvement near it.

\subsection{Performance Ceiling}
\label{sec:improvement-limits}

Every candidate satisfies \(L_{A_k}(\widetilde A)\le1-J_{\mathrm{user}}(A_k)\) under Assumptions~\ref{ass:bounded-reward} and~\ref{ass:agent-independent-tasks}. Lemma~\ref{lem:failure-improvement-bound} and Corollary~\ref{cor:positive-reachability} in Appendix~\ref{app:positive-reachability} give the sharper failure-task bound and its consequence for positive reachability; Appendix~\ref{app:reachability-tail} analyzes reachability near the generation process's improvement endpoint.

Two-Gate requires \(\widehat L_k\ge\tau>(1-\widehat Z_k+\varepsilon_Z)(\delta+\varepsilon_D)+\varepsilon_L\). Let \(\mathcal E_L\) be the event that the required data are available and the error bounds for estimating \(L_k\) hold, and \(\mathcal R\) the set of parameter pairs \((\tau,\delta)\) satisfying this lower bound on \(\tau\) and the upper bounds on \(\widehat L_k\). Definition~\ref{def:admissible-region} in Appendix~\ref{app:proofs-limits} specifies both. Write \(\varepsilon_\Sigma:=2\varepsilon_V+\varepsilon_R+\varepsilon_Z\), where the component errors bound per-task reward estimation, the retained-task statistic, and failure-rate estimation, respectively (Appendices~\ref{app:proofs-evaluation} and~\ref{app:section-G-5}).

\begin{corollary}[Performance ceiling under Two-Gate verification]
\label{cor:fixed-tolerance-constraint}
Under Assumptions~\ref{ass:bounded-reward} and~\ref{ass:agent-independent-tasks}, fix the threshold \(\delta>0\) on estimated retained-task change in Two-Gate. If some \(\tau>0\) satisfies \((\tau,\delta)\in\mathcal R\), then, on \(\mathcal E_L\), the current agent's expected reward satisfies
\begin{equation}
J_{\mathrm{user}}(A_k)
<1-(1-Z_k)\delta-\varepsilon_\Sigma.
\label{eq:formula-5-15}
\end{equation}
\end{corollary}

On \(\mathcal E_L\), Rule~\ref{alg:two-gate} cannot validate further improvement when the current agent reaches or exceeds this ceiling: verification treats the full allowed retained-task change as a possible loss, and evaluation error further reduces the ceiling. Appendix~\ref{app:proofs-limits} gives the proof and finer parameter and sample-size conditions; Figure~\ref{fig:boundary-position} in Appendix~\ref{app:conditional-numerics} illustrates both effects.

Near this ceiling, further gains may be too small to cover the full allowed retained-task change, even when a candidate incurs little loss. We therefore use evidence about the candidate's own retained-task losses to determine whether further improvement can be reliably validated.

Define its mean per-task expected-reward decrease as
\[
D_R^-(\widetilde A;A_k):=\mathbb E_{t\sim\mathcal D_{R,k}}[(V(A_k,t)-V(\widetilde A,t))_+],
\]
and the gain lower bound \(M_{A_k}(\widetilde A):=L_{A_k}(\widetilde A)-(1-Z_k)D_R^-(\widetilde A;A_k)\). The reward decomposition gives \(J_{\mathrm{user}}(\widetilde A)-J_{\mathrm{user}}(A_k)\ge M_{A_k}(\widetilde A)\). Appendix~\ref{app:one-sided-deviations} gives the derivation and an estimator \(\widehat D^-_{R,k}\) of the decrease with error bound \(\varepsilon_D^-=\varepsilon_\Sigma/(1-\widehat Z_k)\).

\begin{paperrule}[Measured-margin validation]
\label{alg:measured-margin}
Use the evaluation data specified by Procedure~\ref{alg:finite-horizon-update} in Appendix~\ref{app:proofs-evaluation}. Given an allowed retained-task decrease \(\delta^->0\), when a valid candidate and the required validation data are available, define its estimated gain lower bound
\begin{equation}
\widehat M_k:=\widehat L_k-\varepsilon_L
-(1-\widehat Z_k+\varepsilon_Z)
(\widehat D^-_{R,k}+\varepsilon_D^-).
\label{eq:formula-4-24}
\end{equation}
Validate the modification if and only if \(\widehat D^-_{R,k}\le\delta^-\) and \(\widehat M_k>0\) hold. Set \(\widehat M_k=0\) when no valid candidate or required validation data are available.
\end{paperrule}

Under the assumptions of Theorem~\ref{thm:one-step-guarantee}, let \(\mathsf{Acc}^M_k\) be adoption under Rule~\ref{alg:measured-margin} through Definition~\ref{def:evolution-step} (Appendix~\ref{app:proofs-evaluation}). Proposition~\ref{prop:measured-margin-guarantee} in Appendix~\ref{app:one-sided-deviations} gives \(\mathbb P(\mathsf{Acc}^M_k\cap\{I_k<\widehat M_k\})\le\beta_{\mathrm{step},k}\). Thus, the rule supports reliable adoption using a gain lower bound informed by the candidate's own losses.

The condition \(L_{A_k}(\widetilde A)>(1-Z_k)\delta^-\) imposes the ceiling \(B_k:=1-(1-Z_k)\delta^-\) (Corollary~\ref{cor:two-gate-threshold}, Appendix~\ref{app:proofs-limits}). In that appendix, Proposition~\ref{prop:positive-margin-instances}(\ref{part:positive-margin-instance}) constructs instances with \(J_{\mathrm{user}}(A_k)\ge B_k\), no retained-task loss, and gain \((1-J_{\mathrm{user}}(A_k))/2>0\); Proposition~\ref{prop:certification-beyond-threshold} gives an evaluation design under which Rule~\ref{alg:measured-margin} validates this improvement when the required samples are available and the error bounds hold. The ceiling thus depends on the treatment of retained-task losses during verification.

\subsection{The Cost of Validating Further Improvement}
\label{sec:evaluation-cost}

As the remaining gains shrink, validating further improvement requires more precise evaluation. Let \(u_k:=1-J_{\mathrm{user}}(A_k)\) be the remaining gap to the reward upper bound. Every candidate satisfies \(M_{A_k}(\widetilde A)\le u_k\) (Proposition~\ref{prop:positive-margin-instances}(\ref{part:universal-margin-bound}), Appendix~\ref{app:proofs-limits}), so validating a positive gain as expected reward approaches one requires correspondingly small evaluation error.

For the preceding construction, \(m\) evaluations per agent per task, \(n_F\) failure tasks, and \(n_R\) reference observations with failure labels, each of order \(\widetilde O(u_k^{-2})\), suffice at fixed failure rate and confidence allocation under the stated sample and error conditions (Proposition~\ref{prop:certification-beyond-threshold}, Appendix~\ref{app:proofs-limits}). The following result shows that evaluation cost must increase in the worst case even for other validation procedures.

Consider any validation procedure that selects tasks adaptively and stops according to its observations. Call it level-\(\beta\) if, in every instance, the probability of validating a candidate with \(J_{\mathrm{user}}(\widetilde A)\le J_{\mathrm{user}}(A_k)\) is at most \(\beta\). Measure its cost by the number \(N_{\mathrm{eval}}\) of candidate agent runs, each returning one task reward. Appendix~\ref{app:section-G-21} specifies the instance and observation model.

\begin{theorem}[Cost of validating further improvement]
\label{thm:certification-sample-lower-bound}
Fix \(Z\in(0,1)\), \(\beta\le1/8\), and \(u\in(0,Z/2]\). There exist two instances with \(Z_k=Z\) and \(1-J_{\mathrm{user}}(A_k)=u\), identical in the distribution of every observable quantity except the candidate agent's reward distribution under the failure-task distribution. In one, \(J_{\mathrm{user}}(\widetilde A)-J_{\mathrm{user}}(A_k)=u/2\) and \(D_R(\widetilde A;A_k)=0\),
and in the other the overall reward change is \(-u/2\). Any level-\(\beta\) validation procedure validating the improving candidate with probability at least \(1/2\), under any task-selection policy and data-dependent stopping rule, satisfies
\[
\mathbb E_-[N_{\mathrm{eval}}]\ge\frac{Z}{10u},
\]
where the expectation is under the non-improving instance.
\end{theorem}

Even when the candidate preserves previously successful performance, verification must distinguish smaller gains from degradation. For fixed \(Z\), the theorem requires a worst-case expected evaluation cost of order at least \(u^{-1}\) as expected reward approaches one; its proof is in Appendix~\ref{app:section-G-21}.

\section{Diagnosing Stalled Harness Self-Evolution}
\label{sec:diagnosis}

Even with guarantees for reliable adoption, self-evolution may stall because generation rarely produces suitable modifications or validation lacks sufficient evidence to accept them. To diagnose stagnation, we evaluate candidate performance shifts relative to the current agent using additional data, including candidates not adopted. These evaluations support reachability estimation, which we combine with validation and adoption records to distinguish the two sources.

\subsection{Reachability Estimation}
\label{sec:acceptance-rates}
\label{sec:measurement}

Fix the interaction history \(\mathcal H_k\) and generation context \(\Xi_k=\xi\). We estimate reachability \(P_k(T\mid\mathcal H_k,\xi)\) from Section~\ref{sec:certified-selection} for improvement and retention targets \(T=(\lambda,\delta)\).

Draw \(N_{\mathrm{meas}}\) outputs \(C_j\) independently from \(\pi_{k,\xi}\). Repeatedly evaluate the current agent and each valid candidate on additional tasks sampled from \(\mathcal D_{\mathrm{user}}\) and withheld from generation and validation. Their per-task expected-reward estimates yield intervals \([\underline L_j,\overline L_j]\) for \(L_k(C_j)\) and \([\underline D_j,\overline D_j]\) for \(D_k(C_j)\). Appendix~\ref{app:drawn-measurement} gives the sampling conditions and interval construction; Appendix~\ref{app:section-I-3-1} gives the corresponding estimates of candidate effects on a fixed task suite.

\begin{definition}
\label{def:measurement-indicators}
For target \(T=(\lambda,\delta)\), define the fractions of generated outputs whose evaluation intervals confirm or permit qualification, respectively, by
\[
\widehat P_{\mathrm{in}}(T):=\frac{\sum_{j=1}^{N_{\mathrm{meas}}}\mathbf1\{\underline L_j\ge\lambda,\overline D_j\le\delta\}}{N_{\mathrm{meas}}},\qquad
\widehat P_{\mathrm{out}}(T):=\frac{\sum_{j=1}^{N_{\mathrm{meas}}}\mathbf1\{\overline L_j\ge\lambda,\underline D_j\le\delta\}}{N_{\mathrm{meas}}}.
\]
Invalid outputs contribute zero to both numerators and remain in the denominator.
\end{definition}

\begin{theorem}[Uniform confidence bounds for reachability]
\label{thm:uniform-reachability-bounds}
In the conditional experiment above, let \(\beta_{\mathrm{rect}},\beta_{\mathrm{emp}}\in(0,1)\) be error probabilities. Suppose all evaluation intervals contain their true values jointly with probability at least \(1-\beta_{\mathrm{rect}}\). Let \(r_{\mathrm{VC}}=r_{\mathrm{VC}}(N_{\mathrm{meas}},\beta_{\mathrm{emp}})\) be a valid uniform empirical error bound for the target sets \(\{(l,d):l\ge\lambda,d\le\delta\}\), a VC class of constant dimension. Then, with probability at least \(1-\beta_{\mathrm{rect}}-\beta_{\mathrm{emp}}\), simultaneously for every target \(T\),
\begin{equation}
\max\{0,\widehat P_{\mathrm{in}}(T)-r_{\mathrm{VC}}\}
\le P_k(T\mid\mathcal H_k,\xi)\le
\min\{1,\widehat P_{\mathrm{out}}(T)+r_{\mathrm{VC}}\}.
\label{eq:formula-6-8}
\end{equation}
\end{theorem}

The bounds estimate reachability even when few candidates pass validation and permit comparisons across targets from the same data. Appendix~\ref{app:proofs-measurement} proves Theorem~\ref{thm:uniform-reachability-bounds} and analyzes estimation precision.

For a target calibrated for reliable adoption by Corollary~\ref{cor:qualified-safe-improvement}, a small upper bound indicates that generation rarely meets the target. For rule \(G\) in Section~\ref{sec:certified-selection}, a substantial lower bound together with independently confirmed but unvalidated modifications supports insufficient validation evidence as a source of stagnation. A wide interval leaves the generation probability uncertain; precision depends on both the number of outputs and their evaluation uncertainty. Definition~\ref{def:certification-selection-probabilities} and Proposition~\ref{prop:certification-probability-bounds} in Appendix~\ref{app:proofs-measurement} formalize the diagnostic quantities, distinguishing validation from subsequent selection.

\subsection{Experimental Illustration}
\label{sec:empirical-studies}

We illustrate this diagnosis on DS-1000 \citep{lai2023ds1000} with persistent Python solver-harness modifications and a fixed language model. Independent generation, label-free revision, and outcome-feedback revision each produce 96 outputs in 24 pools, evenly divided among four contexts. We evaluate the current agent and every valid candidate, including those not adopted, on audit tasks withheld from generation and validation. Appendix~\ref{app:section-I-3-4} gives the generation processes, validation rule, and sampling design for the additional evaluations.

Each context has \(n=56\) audit tasks, with failure and retained groups \(F,R\) fixed by the original current-agent outcomes. On task \(t\), let \(\bar x_t,\bar y_t\) be the current and candidate agents' mean binary success outcomes from three fresh runs each. We estimate the failure-task contribution, mean absolute change on retained tasks, and overall expected-reward gain by
\begin{equation}
\widehat L=\frac1n\sum_{t\in F}(\bar y_t-\bar x_t),\qquad
\widehat D=\frac1{|R|}\sum_{t\in R}|\bar y_t-\bar x_t|,\qquad
\widehat\Delta=\frac1n\sum_t(\bar y_t-\bar x_t).
\label{eq:ds1000-repeated-effects}
\end{equation}
Write \(L,D,\Delta\) for the corresponding expected-reward effects on this fixed audit set, defined in Eq.~\eqref{eq:ds1000-fixed-effects} of Appendix~\ref{app:section-I-3-4}. At the prespecified target \(T_0=(0.05,0.50)\), Table~\ref{tab:ds1000-pqhr} compares the audit estimates with the original decisions based on 48 separate validation tasks per context.

\begin{table}[H]
\centering
\small
\setlength{\tabcolsep}{4pt}
\caption{\textbf{Repeated evaluation and validation on DS-1000.} Audit target membership at \(T_0=(0.05,0.50)\) uses mean outcomes from three fresh runs per agent and task. Brackets give 90\% hierarchical-bootstrap intervals for the empirical fractions. Validation and adoption counts use the original decisions based on 48 validation tasks per context.}
\label{tab:ds1000-pqhr}
\begin{tabular*}{\linewidth}{@{\extracolsep{\fill}}lccc@{}}
\toprule
Generation process & Audit point fraction & Validated outputs & Pools with adoption \\
\midrule
Independent & $0.385\ [0.135,0.646]$ & $2/96$ & $2/24$ \\
Label-free revision & $0.260\ [0.031,0.479]$ & $0/96$ & $0/24$ \\
Outcome-feedback revision & $0.344\ [0.125,0.563]$ & $0/96$ & $0/24$ \\
\bottomrule
\end{tabular*}
\end{table}

Across the three processes, substantial fractions meet the audit point target, yet only two outputs pass validation and are adopted. Both also meet the audit point target. Table~\ref{tab:ds1000-pqhr-full} in Appendix~\ref{app:section-I-3-4} gives the full output and pool summaries; Figure~\ref{fig:ds1000-pqhr} there compares contexts and validation budgets. The appendix also reports the results for the stricter target.

To assess whether an unadopted candidate offers an improvement opportunity, we retrospectively analyze all three candidates whose original audit bounds met the target. For this specified screen, the fresh runs give conditional 95\% joint confidence bounds supporting one outcome-feedback modification: \(L\ge 0.0550\), \(D\le 0.2766\), and \(\Delta\ge 0.0556\) on the fixed audit set (Table~\ref{tab:ds1000-audit-diagnosis}, Appendix~\ref{app:section-I-3-4}). These bounds support the target requirements and positive expected-reward gain. The additional runs provide evidence about the agents' expected performance, accounting for variation across executions of the same task. The candidate's original validation estimate of its failure-task contribution was 0.1458, above 0.05, but its lower bound was 0.0378, below 0.05, while the retention upper bound 0.4196 met the 0.50 requirement; its pool adopted no modification. Thus, the additional evaluation supports improvement on the audit set. The original validation record identifies insufficient evidence of the required improvement as the reason this candidate was not adopted.

\section{Related Work}
\label{sec:related}

\textbf{Harness self-evolution.}
Persistent agent changes \citep{yin2024godel,robeyns2025sica,zelikman2024stop,fernando2023promptbreeder,wang2024voyager,hu2025adas,zhang2025aflow} span source-code rewriting, prompt evolution, reusable skills, and workflow search. DGM \citep{zhang2025dgm} and HyperAgents \citep{zhang2026hyperagents} explore sequences of agent variants, while HGM \citep{wang2025hgm} uses descendant performance to guide search. HarnessX \citep{chen2026harnessx} and Self-Harness \citep{zhang2026selfharness} screen observed regressions when deciding whether to adopt modifications; AHE \citep{lin2026ahe} records predicted fixes and regressions for subsequent evaluation and rollback. MOSS \citep{cai2026moss} integrates source-level evolution into production-grade agent frameworks. These systems motivate studying when harness modifications can be reliably adopted and what determines the performance ceiling of reliable self-evolution.

\textbf{Theoretical guarantees for self-improvement.}
STOP \citep{zelikman2024stop} studies expected improver utility and generalization; \citet{wang2025statisticallimits} analyze learnability through policy-reachable hypothesis classes and establish risk-improvement guarantees under validation and capacity control. SGM \citep{wu2025sgm} controls harmful adoption across repeated tests, while PACE \citep{shawn2026pace} provides an anytime-valid test for an individual candidate under its conditional-fairness assumption. TTHE \citep{nie2026tthe} studies generation coverage and selection regret for harness candidates. We analyze how validation supports reliable adoption and progressive improvement, and how verification and evaluation determine a performance ceiling. We also use candidate performance shifts to diagnose whether generation or validation impedes further improvement.

\textbf{Safe policy improvement.}
HCPI \citep{thomas2015hcpi} uses a confidence bound to test whether a candidate policy meets a specified performance requirement. Our validation rule likewise accounts for evaluation uncertainty before accepting a modification. TRPO \citep{schulman2015trpo} derives sufficient conditions for improvement by bounding the effect of policy-distribution changes. Our sufficient condition combines gains on failed tasks with a bound on changes in expected reward on previously successful tasks to ensure overall improvement.

\section{Limitations}
\label{sec:limitations}

Our guarantees concern expected reward and average changes in task performance under the specified task distributions. The reward--utility analysis in Appendix~\ref{app:reward-fidelity} supports transferring reward guarantees to user utility when the discrepancy between reward and utility can be bounded. Appendix~\ref{app:tail-retention} extends average retention control to task-group protection; applying this extension with finite evaluation data requires corresponding estimation guarantees. When user demand responds to agent updates, Appendix~\ref{app:agent-responsive} gives a correction to the improvement guarantee. Using this correction for adoption requires a bound on the effect of the distribution change before deployment.

\section{Conclusion and Future Work}
\label{sec:conclusion}

We establish a theoretical foundation for reliable harness self-evolution, providing a condition for reliable adoption and showing how validated updates can yield progressive improvement with controlled changes in task performance. Our results show that a performance ceiling imposed by a validation rule need not imply that further improvement is impossible. An alternative rule can validate improvements beyond that ceiling, while reliably validating smaller gains may require greater evaluation cost. Experimental evidence on candidate performance shifts can help distinguish limitations in generating suitable modifications from limitations in validating them. These findings highlight the importance of considering generation and validation jointly in the design of reliable self-evolution systems. Future work can develop methods that use these diagnostic results to guide candidate generation and allocate evaluation budgets, and investigate conditions under which adopted updates preserve opportunities for subsequent improvement (Appendix~\ref{app:scope-open-problems}).

\bibliographystyle{iclr2027_conference}
\bibliography{references}

\clearpage
\appendix
\section*{Appendix Contents}
\addtocontents{toc}{\protect\setcounter{tocdepth}{2}}
\begingroup
\setlength{\parskip}{0pt}
\makeatletter
\@starttoc{toc}
\makeatother
\endgroup
\clearpage

\section{Formal Common Setting and Principal Symbols}
\label{app:formal-setting}

\subsection{Formal common setting}
\label{app:section-A-1}

\begin{definition}[Agent state]
\label{def:agent-state}
At step $k$, the agent state is
\begin{equation}
A_k=(M,C_k),
\label{eq:formula-2-1}
\end{equation}
where $M$ is the frozen LLM and $C_k$ is the evolvable textual or code state: prompts, tools, orchestration, and sub-agent structure. A valid modification persists upon adoption and satisfies
\begin{equation}
\Delta C\in\mathcal M\implies C_k\oplus\Delta C\in\{\text{valid code states}\}.
\label{eq:formula-2-2}
\end{equation}
The measurable modification space is extended to $\mathcal M_\bot:=\mathcal M\sqcup\{\bot\}$ with its disjoint-union measurable structure. One complete run of a predeclared generation procedure returns a valid modification or $\bot$, denoting failure to return one. Tests, feedback, repair, and stopping may be internal to that run. A pool of size $N$ counts all complete runs, including failures; only valid outputs define candidate agents and undergo task-performance evaluation.
\end{definition}

For target bookkeeping, set $L_k(\bot)=D_k(\bot)=0$, with $L_k$ and $D_k$ defined in Section~\ref{sec:qualified-modifications}.

Validity checks establish compile, parse, and runtime preconditions before task-performance evaluation. LLM-weight updates are outside the evolvable scope. The modification space has no prescribed metric, dependency graph, or subspace decomposition. Prompts are included in $C_k$, as in SICA \citep{robeyns2025sica}, the Darwin G\"odel Machine \citep{zhang2025dgm}, and ADAS \citep{hu2025adas}; structure relating modifications to their effects is discussed in Appendix~\ref{app:scope-open-problems}.

The task space is the measurable disjoint union
\begin{equation}
\mathcal T=\mathcal T_{\mathrm{user}}\sqcup\mathcal T_{\mathrm{self}},
\label{eq:formula-2-3}
\end{equation}
where self-modification tasks ask the current agent to generate a modification from failure material. Equip the task and outcome spaces with measurable structures $(\mathcal T,\Sigma_{\mathcal T})$ and $(\mathcal O,\Sigma_{\mathcal O})$, with $\Sigma_{\mathcal O}$ separable. Outcomes are user-facing outputs on $\mathcal T_{\mathrm{user}}$ and terminal generation outputs in $\mathcal M_\bot$ on $\mathcal T_{\mathrm{self}}$.

\begin{definition}[Agent as a Markov kernel]
\label{def:agent-kernel}
The agent acts on both task types through a Markov kernel
\begin{equation}
A_k:\mathcal T\to\Delta(\mathcal O):
\label{eq:formula-2-4}
\end{equation}
for each $t$, $A_k(\cdot\mid t)$ is a probability measure, and $t\mapsto A_k(B\mid t)$ is measurable for every $B\in\Sigma_{\mathcal O}$.
\end{definition}

If the measurable encoding $F\mapsto t_F\in\mathcal T_{\mathrm{self}}$ supplies failure material, its generation kernel is
\begin{equation}
\Pi_k(\cdot\mid F)=A_k(\cdot\mid t_F).
\label{eq:formula-2-5}
\end{equation}
The complete generation procedure in Appendix~\ref{app:reachability-scope} induces the terminal-output distribution $\pi_{k,\xi}$ of Definition~\ref{def:reachability}. The current agent generates its own modifications in the main setting; the sufficient condition for reliable adoption in Corollary~\ref{cor:qualified-safe-improvement} and the pool-selection bound also apply to supplied modifications under their stated conditions. Updating the current agent changes the agent from which later modifications are generated.

The jointly measurable reward $r=r^{\mathrm{tar}}:\mathcal T\times\mathcal O\to[0,1]$ is fixed in reward comparisons and targets.
\begin{assumption}[Bounded reward]
\label{ass:bounded-reward}
For every task--outcome pair, $r^{\mathrm{tar}}(t,o)\in[0,1]$.
\end{assumption}
Its per-task expected reward is
\begin{equation}
V(A,t):=\int_{\mathcal O}r^{\mathrm{tar}}(t,o)A(do\mid t)
=\mathbb E_{o\sim A(\cdot\mid t)}[r^{\mathrm{tar}}(t,o)],
\label{eq:formula-2-7}
\end{equation}
which is measurable in $t$ and lies in $[0,1]$. Let $\mathcal D_{\mathrm{user}}\in\Delta(\mathcal T_{\mathrm{user}})$.

\begin{assumption}[Agent-independent user-task distribution]
\label{ass:agent-independent-tasks}
The user-task distribution does not depend on the agent: the same $\mathcal D_{\mathrm{user}}$ is used in $J_{\mathrm{user}}(A)$ for every agent $A$ produced by the self-evolution process.
\end{assumption}
This assumption applies to the fixed-distribution comparisons that invoke it. The drifting and agent-responsive analyses specify their task distributions locally. Define
\begin{equation}
J_{\mathrm{user}}(A):=\mathbb E_{t\sim\mathcal D_{\mathrm{user}}}[V(A,t)]\in[0,1].
\label{eq:formula-2-8}
\end{equation}
Appendix~\ref{app:reward-fidelity} gives the additional condition for transferring reward guarantees to ideal utility.

\subsection{Principal symbols}
\label{app:section-A-2}

\begin{center}\small
\begin{tabular}{@{}p{0.26\textwidth}p{0.68\textwidth}@{}}
\toprule
Symbol & Meaning \\
\midrule
$A_k=(M,C_k)$ & Current agent, frozen LLM, and evolvable harness (Definition~\ref{def:agent-state}). \\
$\Delta C,\mathcal M_\bot,\bot$ & Modification, terminal-output space, and failed output (Definition~\ref{def:agent-state}). \\
$\mathcal T_{\mathrm{user}},\mathcal T_{\mathrm{self}}$ & User and self-modification task spaces Eq.~\eqref{eq:formula-2-3}. \\
$r,V,J_{\mathrm{user}}$ & Fixed reward, per-task and user-task expected rewards Eqs.~\eqref{eq:formula-2-7}--\eqref{eq:formula-2-8}. \\
$\psi_A,Z_k$ & Failure probability and the current agent's failure rate Eqs.~\eqref{eq:formula-3-1}--\eqref{eq:formula-3-2}. \\
$\mathcal D_{F,k},\mathcal D_{R,k}$ & Failure- and retained-task distributions Eq.~\eqref{eq:formula-3-7}. \\
$L_k(c),D_k(c)$ & Failure-task improvement contribution and retained-task absolute expected-reward change (Section~\ref{sec:qualified-modifications}). \\
$T=(\lambda,\delta),\mathcal P_k(T)$ & Target and qualified modifications (Definition~\ref{def:qualified-modification}). \\
$\Pi_k,\pi_{k,\xi}$ & Self-task kernel and complete-run distribution; see Eq.~\eqref{eq:formula-2-5} and Appendix~\ref{app:reachability-scope}. \\
$P_k(T)$ & Generation probability of a qualified modification (Definition~\ref{def:reachability}). \\
$I_k$ & Expected-reward increment of an update attempt (Definition~\ref{def:evolution-step}). \\
\bottomrule
\end{tabular}
\end{center}

\section{Reliable Adoption and Retained-Task Performance}
\label{app:safety-supplement}

\subsection{Failure weighting and proofs for reliable adoption}
\label{app:failure-weighting}\label{app:proofs-decomposition}

\begin{definition}[Failure detector]
\label{def:failure-detector}
The failure detector is a measurable function $\phi : \mathcal{T}_{\text{user}} \times \mathcal{O} \to \{0, 1\}$, where $\phi(t, o) = 1$ indicates that the outcome $o$ is judged a failure on the user task $t$.
\end{definition}

The failure detector and reward are specified separately; both are measurable.

For an agent $A$ and task $t \in \mathcal{T}_{\text{user}}$, the failure probability is
\begin{equation}
\psi_A(t) := \mathbb{E}_{o \sim A(\cdot|t)}[\phi(t, o)] \in [0,1].
\label{eq:formula-3-1}
\end{equation}
Its average over user tasks is the failure rate,
\begin{equation}
Z(A) := \mathbb{E}_{t \sim \mathcal{D}_{\text{user}}}[\psi_A(t)], \quad Z_k := Z(A_k).
\label{eq:formula-3-2}
\end{equation}

Under Assumptions~\ref{ass:bounded-reward} and \ref{ass:agent-independent-tasks}, any two agents $A,\tilde A$ are compared on the same task distribution. Linearity of expectation in Eq.~\eqref{eq:formula-2-8} therefore gives
\begin{equation}
J_{\text{user}}(\tilde A) - J_{\text{user}}(A) = \mathbb{E}_{t \sim \mathcal{D}_{\text{user}}}\bigl[V(\tilde A, t) - V(A, t)\bigr].
\label{eq:formula-3-3}
\end{equation}

Write the per-task difference as
\begin{equation}
\mathrm{Adv}_A(\tilde A, t) := V(\tilde A, t) - V(A, t) \in [-1, 1].
\label{eq:formula-3-4}
\end{equation}
It satisfies $\mathrm{Adv}_A(A, t) \equiv 0$ and $\mathrm{Adv}_A(\tilde A, t) = -\mathrm{Adv}_{\tilde A}(A, t)$.

\begin{lemma}
\label{lem:failure-decomposition}
For any integrable function $g : \mathcal{T}_{\text{user}} \to \mathbb{R}$:
\begin{equation}
\mathbb{E}_{t \sim \mathcal{D}_{\text{user}}}[g(t)] = \mathbb{E}_{t \sim \mathcal{D}_{\text{user}}}[\psi_{A_k}(t)\, g(t)] + \mathbb{E}_{t \sim \mathcal{D}_{\text{user}}}[(1 - \psi_{A_k}(t))\, g(t)].
\label{eq:formula-3-5}
\end{equation}
\end{lemma}

The two weighted task distributions in Eq.~\eqref{eq:formula-3-7} may assign positive weight to the same task.

Substituting $g(t)=\mathrm{Adv}_{A_k}(\tilde A,t)$ in Lemma~\ref{lem:failure-decomposition} and using Eq.~\eqref{eq:formula-3-3} gives
\begin{equation}
J_{\text{user}}(\tilde A) - J_{\text{user}}(A_k) = \underbrace{\mathbb{E}_{t \sim \mathcal{D}_{\text{user}}}[\psi_{A_k}(t)\,\mathrm{Adv}_{A_k}(\tilde A, t)]}_{=: L_{A_k}(\tilde A)} + \underbrace{\mathbb{E}_{t \sim \mathcal{D}_{\text{user}}}[(1 - \psi_{A_k}(t))\,\mathrm{Adv}_{A_k}(\tilde A, t)]}_{=: E_{A_k}(\tilde A)}.
\label{eq:formula-3-6}
\end{equation}
\begingroup
\makeatletter
\protected@edef\@currentlabel{\theequation}
\phantomsection\label{eq:formula-3-9}
\makeatother
\endgroup
\begin{proof}[Proof of Lemma~\ref{lem:failure-decomposition}.] For every $t$, $\psi_{A_k}(t) + (1 - \psi_{A_k}(t)) = 1$. Since $g$ is integrable and $0\le\psi_{A_k}\le1$, both $\psi_{A_k}g$ and $(1-\psi_{A_k})g$ are integrable. Taking expectations under $\mathcal D_{\mathrm{user}}$ and using linearity gives Eq.~\eqref{eq:formula-3-5}. \end{proof}

\begin{definition}
\label{def:retained-reward-change}
For a reference agent with $Z_k<1$ and a candidate agent $\widetilde A$, define $D_R(\widetilde A;A_k)$ as the average under $\mathcal D_{R,k}$ of the absolute change in per-task expected reward relative to the current agent:
\[
D_R(\widetilde A;A_k)
:=
\mathbb E_{t\sim\mathcal D_{R,k}}
\bigl|V(\widetilde A,t)-V(A_k,t)\bigr|.
\]
This is the quantity in Section~\ref{sec:qualified-modifications}.
\end{definition}

\paragraph{Proofs for reliable adoption.}

\begin{proof}[Proof of Theorem~\ref{thm:expected-reward-improvement}.] By Jensen's inequality, Eq.~\eqref{eq:formula-3-4}, and Definition~\ref{def:retained-reward-change},

\begin{equation}
\begin{aligned}
|E_{A_k}(\widetilde A)|
&=(1-Z_k)\left|\mathbb E_{t\sim\mathcal D_{R,k}}
[\mathrm{Adv}_{A_k}(\widetilde A,t)]\right|\\
&\le(1-Z_k)\mathbb E_{t\sim\mathcal D_{R,k}}
[|\mathrm{Adv}_{A_k}(\widetilde A,t)|]\\
&=(1-Z_k)D_R(\widetilde A;A_k)
\le(1-Z_k)\delta.
\end{aligned}
\label{eq:formula-F-1}
\end{equation}

Substitution into Eq.~\eqref{eq:formula-3-9} gives

\[
J_{\mathrm{user}}(\widetilde A)-J_{\mathrm{user}}(A_k)
=L_{A_k}(\widetilde A)+E_{A_k}(\widetilde A)
\ge L_{A_k}(\widetilde A)-(1-Z_k)\delta.
\]
\end{proof}

\begin{proof}[Proof of Corollary~\ref{cor:qualified-safe-improvement}.] For every $c\in\mathcal P_k(T)$, Definition~\ref{def:qualified-modification} gives $L_k(c)\ge\lambda$ and $D_k(c)\le\delta$. Theorem~\ref{thm:expected-reward-improvement} therefore gives $J_{\mathrm{user}}(A_k\oplus c)-J_{\mathrm{user}}(A_k)\ge\lambda-(1-Z_k)\delta>0$. Substituting the choice of $\lambda_k(\gamma,\delta)$ in Corollary~\ref{cor:qualified-safe-improvement} gives the prescribed-gain conclusion. \end{proof}

\begin{proposition}
\label{prop:retained-change-properties}
Fix a reference agent $A_k$ with $Z_k<1$.
\begin{resultparts}
\item\label{part:retained-change-boundedness} $D_R(\tilde A; A_k) \in [0,1]$.

\item\label{part:retained-change-reflexivity} $D_R(A_k; A_k) = 0$.

\item\label{part:retained-change-asymmetry} When also $Z(\tilde A)<1$, the two defined quantities $D_R(\tilde A;A_k)$ and $D_R(A_k;\tilde A)$ can differ.

\end{resultparts}
\end{proposition}

\begin{proof}[Proof of Proposition~\ref{prop:retained-change-properties}.] (\ref{part:retained-change-boundedness})--(\ref{part:retained-change-reflexivity}) follow from Definition~\ref{def:retained-reward-change} and Assumption~\ref{ass:bounded-reward}. For (\ref{part:retained-change-asymmetry}), take three equally likely tasks, binary outputs, $r(t,o)=o$, and $\phi(t,o)=1-o$. Let the deterministic current agent's outputs be $(1,0,0)$ and the candidate outputs $(1,1,0)$. Their failure rates are $2/3$ and $1/3$. The current agent's retained distribution is concentrated on the first task, whereas the candidate's is uniform on the first two. Thus $D_R(\tilde A;A_k)=0$ and $D_R(A_k;\tilde A)=1/2$. \end{proof}

\subsection{Tail constraints and subgroup protection}
\label{app:tail-retention}\label{app:subgroup-improvement}

The mean constraint $\mathbb{E}_{\mathcal{D}_{R,k}}[|\mathrm{Adv}|]\le\delta$ implies, for a measurable subgroup $S$ with probability $w>0$ under $\mathcal D_{R,k}$,
\begin{equation}
\mathbb{E}_{\mathcal{D}_{R,k}}\bigl[|\mathrm{Adv}|\;\big|\;S\bigr]\;\le\;\frac{\delta}{w},
\label{eq:formula-G-19}
\end{equation}
by integrating the nonnegative variable over $S$. The bound becomes less informative as $w$ decreases and exceeds the universal bound $1$ when $w<\delta$.

A tail-level constraint gives uniform control for subgroups above a specified probability. For $\alpha\in(0,1)$ write $\mathrm{CVaR}_\alpha$ for the conditional value at risk at tail level $\alpha$, in the Rockafellar--Uryasev form \citep{rockafellar2002cvar} $\mathrm{CVaR}_\alpha(X)=\min_t\{t+\alpha^{-1}\mathbb{E}[(X-t)_+]\}$, defined for arbitrary distributions.

\begin{definition}
\label{def:tail-deviation}
For $\alpha\in(0,1)$,
\begin{equation}
D_R^{\alpha}(\tilde A;A_k) \;:=\; \mathrm{CVaR}_\alpha^{\mathcal{D}_{R,k}}\bigl(|V(\tilde A,\cdot)-V(A_k,\cdot)|\bigr).
\label{eq:formula-G-20}
\end{equation}
\end{definition}

\begin{proposition}
\label{prop:tail-subgroup-protection}
If $D_R^{\alpha}(\tilde A;A_k)\le\delta$, then for every measurable $S$ with $\mathcal{D}_{R,k}(S)\ge\alpha$,
\begin{equation}
\mathbb{E}_{\mathcal{D}_{R,k}}\bigl[|\mathrm{Adv}_{A_k}(\tilde A,\cdot)|\;\big|\;S\bigr]\;\le\;\delta.
\label{eq:formula-G-21}
\end{equation}
Moreover, $D_R\le D_R^{\alpha}$.
\end{proposition}

The tail-level constraint therefore implies $D_R\le\delta$, so Theorem~\ref{thm:expected-reward-improvement} gives the same expected-reward guarantee as under the mean constraint of Definition~\ref{def:retained-reward-change}.

\begin{proof}[Proof of Proposition~\ref{prop:tail-subgroup-protection}.] The risk-envelope representation of CVaR (\citealp{shapiro2009lectures}, Example 6.19) states
\[
\mathrm{CVaR}_\alpha(X)\;=\;\max\Bigl\{\mathbb{E}[XZ]\;:\;0\le Z\le\alpha^{-1},\ \mathbb{E}[Z]=1\Bigr\},
\]
valid for arbitrary distributions. Taking $Z=\mathbf{1}_S/\mathcal{D}_{R,k}(S)$, which satisfies $0\le Z\le\alpha^{-1}$ precisely because $\mathcal{D}_{R,k}(S)\ge\alpha$, gives $\mathbb{E}[|\mathrm{Adv}|\mid S]=\mathbb{E}[|\mathrm{Adv}|Z]\le\mathrm{CVaR}_\alpha\le\delta$. The second claim is $Z\equiv1$ in the same representation. \end{proof}

\begin{remark}
\label{rem:tail-event-comparison}
The risk-envelope representation implies $\mathbb E[X\mid S]\le\mathrm{CVaR}_\alpha(X)$ whenever $\mathbb P(S)\ge\alpha$, without an atomlessness assumption. For distributions with atoms, \citet[Corollaries 3.20--3.21 and Example 5.4]{acerbi2002coherence} show that equality with the supremum over events need not hold.

\end{remark}

\begin{proposition}
\label{prop:empirical-cvar}
In the retained-task application, let $X=|V(\tilde A,t)-V(A_k,t)|$ for $t\sim\mathcal D_{R,k}$, and let $X_1,\ldots,X_n$ be i.i.d.\ observations of these true quantities. Define the empirical upper-tail statistic by
\[
\widehat{\mathrm{CVaR}}_\alpha
:=\min_{z\in[0,1]}\left\{z+\frac1{n\alpha}\sum_{i=1}^n(X_i-z)_+\right\}.
\]
This definition also applies when $n\alpha$ is not an integer, with fractional weight at the empirical tail boundary. Let $\alpha,\beta\in(0,1)$. If $n\ge25\log(2/\beta)$, then with probability at least $1-\beta$,
\[
\bigl|\widehat{\mathrm{CVaR}}_\alpha-\mathrm{CVaR}_\alpha\bigr|\;\le\;\sqrt{\frac{3\log(2/\beta)}{n\alpha}}+\frac{15\log(2/\beta)}{n\alpha}.
\]
For $0<\varepsilon\le1$, a radius $\varepsilon$ is achieved by $n=O\bigl(\log(2/\beta)/(\alpha\varepsilon^2)\bigr)$ observations. When $n\alpha\ge1$, any estimator has worst-case error $\Omega(1/\sqrt{\alpha n})$ with probability bounded away from zero. Thus the additional factor $1/\alpha$ in the sample size is necessary in the worst case at a fixed confidence level.
\end{proposition}

\begin{proof}[Proof of Proposition~\ref{prop:empirical-cvar}.] Applying \citet[Theorem C.6]{wang2023cvar} to $1-X$ gives the displayed upper-tail bound under the stated sample-size condition. Requiring each term to be at most $\varepsilon/2$ gives the sufficient sample size.

For the lower bound, let $X=\mathrm{Bernoulli}(p)$ with $p\le\alpha$, so $\mathrm{CVaR}_\alpha(X)=p/\alpha$. For $n\alpha\ge1$ and a fixed $c\in(0,1/8]$, set $p_0=\alpha/2$ and $p_1=\alpha/2+c\sqrt{\alpha/n}$. Then $p_1\le5\alpha/8\le\alpha$ and the Bernoulli KL inequality gives
\[
n\,\mathrm{KL}(\mathrm{Ber}(p_0)\Vert\mathrm{Ber}(p_1))
\le \frac{n(p_1-p_0)^2}{p_1(1-p_1)}\le\frac{16c^2}{3}.
\]
Pinsker's inequality bounds the total variation of the two product measures by $\sqrt{8/3}\,c<1/2$, while their CVaR values differ by $c/\sqrt{\alpha n}$. A two-point testing argument therefore gives worst-case estimation error $\Omega(1/\sqrt{\alpha n})$ with probability bounded away from zero. \end{proof}

The usual sufficient sample size for the mean is $O\bigl(\log(2/\beta)/\varepsilon^2\bigr)$, so estimating the tail criterion incurs an additional $1/\alpha$ cost in this worst-case i.i.d.\ comparison.

A rollout-based implementation additionally controls the errors in the two per-task expected rewards.

\paragraph{Subgroup improvement.}
For a fixed candidate, let $g(t)=\mathrm{Adv}_{A_k}(\widetilde A,t)$, $\bar\mu=\mathbb E_{\mathcal D_{\mathrm{user}}}g$, and $\sigma^2=\mathrm{Var}_{\mathcal D_{\mathrm{user}}}g$. Cauchy--Schwarz gives, for every subgroup $S$ with $\mathcal D_{\mathrm{user}}(S)\ge w_0\in(0,1)$,
\[
\mathbb E[g\mid S]\ge\bar\mu-\sigma\sqrt{\frac{1-w_0}{w_0}}.
\]
Indeed, writing $w=\mathcal D_{\mathrm{user}}(S)$ gives
\[
\mathbb E[g\mid S]-\bar\mu
=\frac{\operatorname{Cov}(g,\mathbf1_S)}{w}
\ge-\sigma\sqrt{\frac{\operatorname{Var}(\mathbf1_S)}{w^2}}
=-\sigma\sqrt{\frac{1-w}{w}}
\ge-\sigma\sqrt{\frac{1-w_0}{w_0}},
\]
where the first inequality is Cauchy--Schwarz and the last uses $w\ge w_0$. Thus $\sigma<\bar\mu\sqrt{w_0/(1-w_0)}$ suffices for positive improvement in every such subgroup.

\subsection{Alternative retained-task norms}
\label{app:section-G-12}

For $p\in[1,\infty]$, define
\[
D_R^{(p)}(\tilde A;A_k)
:=\bigl\|V(\tilde A,\cdot)-V(A_k,\cdot)\bigr\|_{L^p(\mathcal D_{R,k})}.
\]
In particular, $D_R^{(1)}=D_R$. At a fixed radius $\eta$, the feasible sets are nested as
\[
\mathcal C_\eta^{(\infty)}\subseteq\cdots\subseteq\mathcal C_\eta^{(1)},
\qquad
\mathcal C_\eta^{(p)}:=\{\tilde A:D_R^{(p)}(\tilde A;A_k)\le\eta\}.
\]

\begin{proposition}
\label{prop:lp-improvement-bound}
Under Assumptions~\ref{ass:bounded-reward} and \ref{ass:agent-independent-tasks}, fix $Z_k\in(0,1)$, $p\in[1,\infty]$, and $\delta_p\ge0$. If $D_R^{(p)}(\tilde A;A_k)\le\delta_p$, then $J_{\mathrm{user}}(\tilde A)-J_{\mathrm{user}}(A_k)\ge L_{A_k}(\tilde A)-(1-Z_k)\delta_p$.
\end{proposition}

\begin{proof}[Proof of Proposition~\ref{prop:lp-improvement-bound}.] On the probability space with measure $\mathcal D_{R,k}$, $\|\mathrm{Adv}\|_1\le\|\mathrm{Adv}\|_p$ for every $p\in[1,\infty]$. Thus $D_R\le\delta_p$, and Theorem~\ref{thm:expected-reward-improvement} gives the conclusion. \end{proof}

At a fixed radius, the $L^1$ constraint is the most permissive of these nested requirements. Finite-data validation for another $p$ requires an estimator and a coverage bound for $D_R^{(p)}$.

\subsection{Output-distribution comparisons}
\label{app:section-I-6-2}\label{app:section-G-13}

\begin{equation}
D_R^{\mathrm{KL}}(\tilde A; A_k)
:=\mathbb{E}_{t\sim\mathcal D_{R,k}}
\bigl[D_{\mathrm{KL}}(\tilde A(\cdot\mid t)\Vert A_k(\cdot\mid t))\bigr].
\label{eq:formula-3-12}
\end{equation}
The direction of KL is from the candidate's output distribution to that of the current agent.

\begin{proposition}
\label{prop:reward-preserving-changes}
Fix $A_k$ with $Z_k<1$. If $V(\tilde A,t)=V(A_k,t)$ for $\mathcal D_{\mathrm{user}}$-almost every $t$, then $D_R(\tilde A;A_k)=0$, regardless of whether the outcome distributions differ.
\end{proposition}

\begin{proof}[Proof of Proposition~\ref{prop:reward-preserving-changes}.] By Eq.~\eqref{eq:formula-3-7}, $\mathcal D_{R,k}$ is absolutely continuous with respect to $\mathcal D_{\mathrm{user}}$, with density $(1-\psi_{A_k})/(1-Z_k)$. Hence equality of expected rewards $\mathcal D_{\mathrm{user}}$-almost everywhere also holds $\mathcal D_{R,k}$-almost everywhere. Integrating the absolute difference in Definition~\ref{def:retained-reward-change} gives $D_R(\tilde A;A_k)=0$. \end{proof}

\begin{lemma}
\label{lem:reward-outcome-separation}
There exist $\mathcal{O}, \mathcal{T}_{\text{user}}, A_k, \tilde A, \mathcal{D}_{\text{user}}, r$ satisfying Assumption~\ref{ass:bounded-reward} and:
(a) $V(\tilde A, t) = V(A_k, t)$ for all $t \in \mathcal{T}_{\text{user}}$;

(b) $D_R(\tilde A; A_k) = 0$;

(c) $D_R^{\mathrm{KL}}(\tilde A; A_k) = \infty$.
\end{lemma}

\begin{proof} Fix $z\in(0,1)$. Take two tasks $t_0,t_1$ with probabilities $1-z,z$ and outcomes $\{a,b\}$. On $t_0$, let $A_k=\delta_a$, $\tilde A=\delta_b$, $r(t_0,a)=r(t_0,b)=1/2$, and $\phi(t_0,\cdot)=0$. On $t_1$, give both agents the same kernel and reward and set $\phi(t_1,\cdot)=1$. Then $Z_k=z$ and $\mathcal D_{R,k}=\delta_{t_0}$. The expected rewards agree on both tasks, so $D_R=0$; the candidate assigns mass one to an output of zero probability under the current agent on $t_0$, giving $D_R^{\mathrm{KL}}=\infty$. \end{proof}

\paragraph{Improvement bounds.}

Bounding the output-distribution KL divergence gives the following improvement bound.

\begin{proposition}
\label{prop:kl-improvement-bound}
Under Assumptions~\ref{ass:bounded-reward} and \ref{ass:agent-independent-tasks}, fix a current agent with $Z_k<1$. If $D_R^{\mathrm{KL}}(\tilde A; A_k) \leq \delta'$, then:
\begin{equation}
J_{\text{user}}(\tilde A) - J_{\text{user}}(A_k) \geq L_{A_k}(\tilde A) - (1 - Z_k)\,\sqrt{\delta'/2}.
\label{eq:formula-G-22}
\end{equation}
\end{proposition}

\begin{proposition}
\label{prop:kl-deviation-comparison}
Fix a current agent with $Z_k<1$ and a retained-task bound $\eta>0$.
\begin{resultparts}
\item\label{part:equal-retained-penalty} Both $D_R(\tilde A;A_k)\le\eta$ and $D_R^{\mathrm{KL}}(\tilde A;A_k)\le2\eta^2$ are sufficient for $\mathbb E_{t\sim\mathcal D_{R,k}}[|\mathrm{Adv}_{A_k}(\tilde A,t)|]\le\eta$.

\item\label{part:retained-constraint-inclusion} The feasible sets satisfy
\[
\{\tilde A:D_R^{\mathrm{KL}}(\tilde A;A_k)\le2\eta^2\}
\subseteq
\{\tilde A:D_R(\tilde A;A_k)\le\eta\}.
\]
There exist instances in the model class for which this inclusion is strict.
\end{resultparts}
\end{proposition}

The constraint $D_R\le\delta$ admits at least the candidate agents allowed by $D_R^{\mathrm{KL}}\le2\delta^2$, and can additionally admit reward-preserving changes in the output distribution that violate the KL constraint. Substituting $\delta'=2\delta^2$ into Eq.~\eqref{eq:formula-G-22} returns $(1-Z_k)\delta$, exactly the retained-task term of Theorem~\ref{thm:expected-reward-improvement}. Thus the two sufficient conditions give the same expected-reward bound, while their admissible sets can differ.

\begin{proof}[Proof of Proposition~\ref{prop:kl-improvement-bound}.] Use the retained-task expression for $|E_{A_k}(\tilde A)|$ in Eq.~\eqref{eq:formula-F-1}. Fix $t$, and write $P = \tilde A(\cdot|t)$, $Q = A_k(\cdot|t)$, $f = r(t, \cdot) \in [0,1]$. By the Hahn-Jordan decomposition $\nu := P - Q = \nu^+ - \nu^-$, with $\nu^+(\mathcal{O}) = \nu^-(\mathcal{O}) = D_{\mathrm{TV}}(P, Q)$. Since $f \in [0,1]$ and $\int f\, d\nu^\pm \in [0, D_{\mathrm{TV}}]$:
\begin{equation}
|\mathrm{Adv}_{A_k}(\tilde A, t)| = \left|\int f\, d\nu\right| \leq D_{\mathrm{TV}}(P, Q).
\label{eq:formula-F-2}
\end{equation}

By Pinsker's inequality $D_{\mathrm{TV}} \leq \sqrt{D_{\mathrm{KL}}/2}$ and the concavity of $\sqrt{\cdot}$ (Jensen):
\[
\mathbb{E}_{t \sim \mathcal{D}_{R,k}}[|\mathrm{Adv}|] \leq \mathbb{E}_{t \sim \mathcal{D}_{R,k}}\bigl[\sqrt{D_{\mathrm{KL}}/2}\bigr] \leq \sqrt{D_R^{\mathrm{KL}}/2} \leq \sqrt{\delta'/2}.
\]

Substituting this bound into Eq.~\eqref{eq:formula-3-9} yields Eq.~\eqref{eq:formula-G-22}. \end{proof}

\begin{proof}[Proof of Proposition~\ref{prop:kl-deviation-comparison}.] (\ref{part:equal-retained-penalty}) The $D_R$ constraint gives the penalty $\eta$ directly; the KL constraint gives $\sqrt{2\eta^2/2}=\eta$ by Eq.~\eqref{eq:formula-F-2}, Pinsker's inequality, and Jensen's inequality. (\ref{part:retained-constraint-inclusion}) Pinsker's inequality and Jensen give the explicit comparison $D_R\le\sqrt{D_R^{\mathrm{KL}}/2}$, so every candidate satisfying the KL constraint also satisfies the corresponding $D_R$ constraint. Lemma~\ref{lem:reward-outcome-separation} provides an instance with $D_R=0$ and $D_R^{\mathrm{KL}}=\infty$, so for every finite $\eta\ge0$ strict inclusion occurs in an instance allowed by the model. \end{proof}

\subsection{Reward fidelity and ideal utility}
\label{app:reward-fidelity}

Let the jointly measurable utility $r^*:\mathcal T\times\mathcal O\to[0,1]$ represent the designer's ideal preference. Define its difference from the paper's reward by $\Delta r^{\mathrm{tar}}(t,o):=r^{\mathrm{tar}}(t,o)-r^*(t,o)\in[-1,1]$.
\begin{assumption}[Uniform reward fidelity]
\label{ass:reward-fidelity}
There is $\epsilon_{\mathrm{tar}}\ge0$ such that $|\Delta r^{\mathrm{tar}}(t,o)|\le\epsilon_{\mathrm{tar}}$ for every task--outcome pair.
\end{assumption}
Write $V^*(A,t):=\int r^*(t,o)A(do\mid t)$ and $J^*_{\mathrm{user}}(A):=\mathbb E_{\mathcal D_{\mathrm{user}}}V^*(A,t)$. These quantities are measurable and lie in $[0,1]$. The pointwise fidelity condition covers every agent, including those obtained by generation and selection.

\begin{proposition}
\label{prop:reward-fidelity}
Under Assumption~\ref{ass:reward-fidelity},
\[
\bigl|J_{\mathrm{user}}(A)-J^*_{\mathrm{user}}(A)\bigr|\le\epsilon_{\mathrm{tar}}
\]
for every agent $A$.
\end{proposition}
\begin{proof} By Eqs.~\eqref{eq:formula-2-7}--\eqref{eq:formula-2-8} and the corresponding definitions using $r^*$,
\[
J_{\mathrm{user}}(A)-J^*_{\mathrm{user}}(A)
=\mathbb E_{t\sim\mathcal D_{\mathrm{user}}}
\mathbb E_{o\sim A(\cdot\mid t)}[r(t,o)-r^*(t,o)].
\]
Taking absolute values and using the triangle inequality gives
\[
\bigl|J_{\mathrm{user}}(A)-J^*_{\mathrm{user}}(A)\bigr|
\leq \mathbb E_{t\sim\mathcal D_{\mathrm{user}}}\mathbb E_{o\sim A(\cdot\mid t)}
\bigl[|r(t,o)-r^*(t,o)|\bigr]
\leq\epsilon_{\mathrm{tar}},
\]
where the last inequality is pointwise by Assumption~\ref{ass:reward-fidelity}. \end{proof}

\begin{corollary}
\label{cor:utility-improvement}
Under Assumption~\ref{ass:reward-fidelity}, if
\[
J_{\mathrm{user}}(\tilde A)-J_{\mathrm{user}}(A_k)\ge\gamma,
\]
then
\begin{equation}
J^*_{\mathrm{user}}(\tilde A)-J^*_{\mathrm{user}}(A_k)
\ge\gamma-2\epsilon_{\mathrm{tar}}.
\label{eq:formula-E-1}
\end{equation}
Consequently, a finite-sample lower bound on the expected-reward difference transfers on the same confidence event after subtracting $2\epsilon_{\mathrm{tar}}$.
\end{corollary}
\begin{proof} Write the ground-truth improvement as
\[
\begin{aligned}
J^*_{\mathrm{user}}(\tilde A)-J^*_{\mathrm{user}}(A_k)
={}&J_{\mathrm{user}}(\tilde A)-J_{\mathrm{user}}(A_k)\\
&-\bigl[J_{\mathrm{user}}(\tilde A)-J^*_{\mathrm{user}}(\tilde A)\bigr]
+\bigl[J_{\mathrm{user}}(A_k)-J^*_{\mathrm{user}}(A_k)\bigr].
\end{aligned}
\]
Applying Proposition~\ref{prop:reward-fidelity} to the two agents proves Eq.~\eqref{eq:formula-E-1}. \end{proof}

In particular, when Theorem~\ref{thm:one-step-guarantee} supplies the stated lower bound $\Delta_k$ on the expected-reward difference on its confidence event, Corollary~\ref{cor:utility-improvement} applies with $\gamma=\Delta_k$.

For fixed $r,r^*$ and $\mathcal D_{\mathrm{user}}$, applying Corollary~\ref{cor:utility-improvement} to $A_0,A_K$ gives an endpoint utility gain at least $J_{\mathrm{user}}(A_K)-J_{\mathrm{user}}(A_0)-2\epsilon_{\mathrm{tar}}$.

\section{Generation and Validation for Reliable Adoption}
\label{app:certification-supplement}

\subsection{Validation data and update procedure}
\label{app:proofs-evaluation}\label{app:section-I-5-3}\label{app:section-I-5-4}\label{app:section-G-16}\label{app:stored-sample-reuse}\label{app:deviation-bias}

This appendix supplies the evaluation conditions and proofs for Section~\ref{sec:certification}. The update procedure uses the assumptions below and Conditions~\ref{cond:evaluation-data-separation} and \ref{cond:returned-sample-coverage}.

The sequence $\{(T_n, O_n, \Phi_n)\}_{n \geq 1}$ records the user-facing task $T_n \in \mathcal{T}_{\text{user}}$, the outcome $O_n \in \mathcal{O}$, and the failure indicator $\Phi_n := \phi(T_n, O_n)$ for the $n$-th task handled by the current agent.

\begin{assumption}[Stationary $\beta$-mixing task process]
\label{ass:mixing-task-process}
$\{T_n\}_{n \geq 1}$ is a stationary ergodic process with marginal distribution $T_n \sim \mathcal{D}_{\text{user}}$. Write $\mathcal F_n:=\sigma(T_1,\ldots,T_n)$ and $\mathcal F_{n+k}^{\infty}:=\sigma(T_{n+k},T_{n+k+1},\ldots)$. The process satisfies $\beta$-mixing: there exist coefficients $\{\beta(k)\}_{k \geq 1}$ tending to zero, where
\begin{equation}
\beta(k) := \sup_{n \geq 1}\,\mathbb{E}\!\left[\sup_{B \in \mathcal F_{n+k}^{\infty}}\bigl|\mathbb{P}(B\mid\mathcal F_n)-\mathbb{P}(B)\bigr|\right].
\label{eq:formula-4-1}
\end{equation}
\end{assumption}

\begin{remark}
\label{rem:mixing-examples}
Stationary $\beta$-mixing processes \citep{yu1994rates,mohri2009rademacher} include i.i.d.\ tasks, finite-range dependent tasks, and ergodic finite-state Markov chains under the usual aperiodicity conditions. Independence beyond a finite lag holds for the finite-range case.
\end{remark}

Assumption~\ref{ass:mixing-task-process} and $\ell$ control the stored-sample distribution comparison. Failure-indexed evaluation is governed separately by Condition~\ref{cond:returned-sample-coverage}, with Proposition~\ref{prop:first-failure-sampling} providing an explicit i.i.d. task-stream construction.

\paragraph{Observed outcomes and fresh evaluation.}
\begin{assumption}
\label{ass:failure-detector-locality}
The failure detector is local to each task: $\Phi_n$ depends only on $(T_n,O_n)$, not on $\{T_m,O_m:m\neq n\}$.
\end{assumption}

\begin{assumption}
\label{ass:task-outcome-model}
Given $T_n$ and the current agent $A_k$, $O_n\sim A_k(\cdot\mid T_n)$ is conditionally independent of $\{T_m,O_m:m<n\}$.
\end{assumption}

\begin{assumption}
\label{ass:outcome-evaluation-independence}
Strengthening Assumption~\ref{ass:task-outcome-model}, for the current agent $A_k$, the sequence $\{O_n\}_{n \geq 1}$ satisfies
\begin{equation}
O_n \perp \{(T_m, O_m) : m \neq n\} \cup \{O^{\text{fresh}}_{m, l} : \forall m, l\} \,\bigg|\, T_n.
\label{eq:formula-4-2}
\end{equation}

where $\{O^{\text{fresh}}_{m, l}\}$ are the fresh evaluation rollouts used for validation and defined in Eq.~\eqref{eq:formula-4-14}. That is, the observed task outcome $O_n$ given $T_n$ is conditionally independent of the other observed task records and fresh rollouts for this fixed current agent.
\end{assumption}

\begin{assumption}
\label{ass:trial-consistency}
For a fixed current agent $A_k$ and task $t$, repeated evaluations of $\phi$ and $r$ are i.i.d. When the evaluator is an LLM judge, its prompt template, model version, sampling temperature, and other settings are held fixed.
\end{assumption}

\begin{definition}[Evolution step]
\label{def:evolution-step}
The $k$-th evolution step uses predeclared finite task-arrival horizons for generation material, evaluation on failure tasks, and any recollection of the stored sample. Its output is
\begin{equation}
(A_k,F_k^{\mathrm{gen}},\Delta C,\mathrm{decision}_k)
\mapsto A_{k+1}=
\begin{cases}
A_k\oplus\Delta C,&\mathrm{decision}_k=\mathrm{accept},\\
A_k,&\text{otherwise},
\end{cases}
\label{eq:formula-4-3}
\end{equation}
When generation material is available, the current agent draws
\begin{equation}
\Delta C\sim\Pi_k(\cdot\mid F_k^{\mathrm{gen}})
=A_k(\cdot\mid t_{F_k^{\mathrm{gen}}});
\label{eq:formula-4-4}
\end{equation}
otherwise $\Delta C=\bot$.
Here $A_k=(M,C_k)$ is the current agent and $F_k^{\mathrm{gen}}$ is the returned generation material. Write $\mathsf{Cand}_k:=\{\Delta C\in\mathcal M\}$ and $\mathrm{decision}_k\in\{\mathrm{accept},\mathrm{reject},\mathrm{abstain}\}$. The increment $I_k:=J_{\mathrm{user}}(A_{k+1})-J_{\mathrm{user}}(A_k)$ is defined on every branch. Observable margins have their pre-adoption definitions on $\mathsf{Cand}_k\cap\mathsf{Ready}_k$ and are zero elsewhere. Procedure~\ref{alg:finite-horizon-update} specifies collection and evaluation.
\end{definition}

\paragraph{Stored sample and returned failure batch.}
The stored sample consists of $n_R$ task records $(t_j^R,\phi_j^R)$ labeled by the current agent. Its distribution satisfies Condition~\ref{cond:evaluation-data-separation}; Remark~\ref{rem:common-offset} gives a fixed-window construction. Separately, Phase E returns $n_F$ failure-marked tasks within its declared finite horizon, or returns $\bot$.

\paragraph{Reward estimation from repeated evaluations.}
Failure-marked outcomes of the current agent are not reused as value estimates. Given the frozen agents and evaluation records, both $A_k$ and $\tilde A$ are rerun using independent groups of $m$ fresh rollouts on every evaluation task:
\begin{equation}
\hat V(A, t) := \frac{1}{m}\sum_{l=1}^m r(t, o_l), \quad o_l \overset{\text{i.i.d.}}{\sim} A(\cdot | t).
\label{eq:formula-4-14}
\end{equation}
The stored failure mark remains unchanged because it defines the weighting by the retained-task distribution. The stored failure-rate estimate is $\hat Z_k:=\frac{1}{n_R}\sum_{j=1}^{n_R}\phi^R_j$. The failure-task improvement estimate is
\begin{equation}
\hat L_k := \hat Z_k \cdot \frac{1}{n_F}\sum_{i=1}^{n_F}\bigl[\hat V(\tilde A, t^F_i) - \hat V(A_k, t^F_i)\bigr],
\label{eq:formula-4-16}
\end{equation}
and, on the nonempty retained set $S_R:=\{j:\phi_j^R=0\}$,
\begin{equation}
\hat D_R := \frac{1}{n_R(1-\hat Z_k)}\sum_{j=1}^{n_R}(1-\phi^R_j)\,\bigl|\hat V(\tilde A,t_j^R)-\hat V(A_k,t_j^R)\bigr|,
\label{eq:formula-4-17}
\end{equation}
equivalently,
\begin{equation}
\hat D_R = \frac{1}{|S_R|}\sum_{j\in S_R}\bigl|\hat V(\tilde A,t_j^R)-\hat V(A_k,t_j^R)\bigr|.
\label{eq:formula-4-18}
\end{equation}
All statistics and radii under the retained-task distribution are evaluated only when $|S_R|\ge1$; otherwise the validation rule abstains. The expected-reward estimate from the full stored sample is
\[
\hat J_k:=\frac{1}{n_R}\sum_{j=1}^{n_R}\hat V(A_k,t_j^R).
\]

\paragraph{Evaluation-data separation and readiness.}
\begin{condition}[Evaluation-data separation and stored-sample comparison]
\label{cond:evaluation-data-separation}
The generation material $F_k^{\mathrm{gen}}$ is disjoint from the evaluation failure batch $F_k^{\mathrm{gate}}$ and the stored sample $R_k$. Write $n_F=|F_k^{\mathrm{gate}}|$. Let $\mathcal G_k$ contain the current agent, frozen candidate, generation information and randomness, and any stored-sample offset. The joint law $\mathbb P$ of this information and the stored records has a reference law $\mathbb P^\circ$ with the same $\mathcal G_k$ marginal such that, under $\mathbb P^\circ$, conditionally on $\mathcal G_k$,
\[
T^\circ_{1:n_R}\stackrel{\mathrm{iid}}{\sim}\mathcal D_{\mathrm{user}},
\qquad
\mathcal L(\Phi^\circ_{1:n_R}\mid\mathcal G_k,T^\circ_{1:n_R})
=\bigotimes_{j=1}^{n_R}\mathrm{Bernoulli}(\psi_{A_k}(T^\circ_j)),
\]
and
\begin{equation}
\|\mathbb P-\mathbb P^\circ\|_{\mathrm{TV}}\le2n_R\beta(\ell).
\label{eq:formula-G-10}
\end{equation}
A coupling preserving the generation information and disagreeing on the records with probability at most this bound is sufficient. Fresh-rollout coverage is imposed separately, given the frozen agents and observed evaluation records.
\end{condition}

\begin{condition}[Returned-sample coverage]
\label{cond:returned-sample-coverage}
Let $\mathsf{Pre}_k$ contain all information frozen before Phase E: the current agent and candidate agent, the finite family $\mathcal Q_k^{\mathrm{eval}}$ and deterministic range bounds, radii $r_{F,k}(f)$, deterministic risk $\beta_{F,k}$, the finite arrival horizon, and the return rule. Phase E either returns $n_F$ failure tasks $T^F_{k,1:n_F}$ or returns $\bot$; write $\mathsf{Ret}_{F,k}$ for the return event. For
\[
\bar f_{F,k}:=\frac1{n_F}\sum_{i=1}^{n_F}f(T^F_{k,i}),
\qquad
\mu_{F,k}(f):=\mathbb E_{T\sim\mathcal D_{F,k}}[f(T)],
\]
the returned failure sample satisfies
\[
\mathbb P\!\left(\mathsf{Ret}_{F,k}\cap
\left\{\exists f\in\mathcal Q_k^{\mathrm{eval}}:
|\bar f_{F,k}-\mu_{F,k}(f)|>r_{F,k}(f)\right\}\right)
\le\beta_{F,k}.
\]
The candidate agent, family, radii, risk, and return rule may depend on $\mathsf{Pre}_k$ but not on Phase-E outcomes. Let $\mathsf{Ready}_{R,k}$ denote availability of a valid stored sample under Condition~\ref{cond:evaluation-data-separation}, set $\mathsf{Ready}_k:=\mathsf{Ret}_{F,k}\cap\mathsf{Ready}_{R,k}$, and require every acceptance event to lie in $\mathsf{Cand}_k\cap\mathsf{Ready}_k$.
\end{condition}

\begin{remark}
\label{rem:common-offset}
Fix a deterministic task index $b$. Let the current agent, candidate, and generation information be determined by tasks through $b$ and auxiliary randomness independent of the task stream. Draw an independent $U\sim\mathrm{Unif}\{0,\ldots,\ell-1\}$ and collect records at indices $b+j\ell+U$, $j=1,\ldots,n_R$. Given the generation information and these tasks, generate the labels independently with probabilities $\psi_{A_k}(t_j^R)$, using randomness separated from generation.

Conditioning on $U$, the first task is at least $\ell$ indices after $b$, and successive tasks are $\ell$ apart. The absolute-regularity comparison of the history and these tasks with the product reference law \citep{yu1994rates,doukhan1994mixing}, applied successively over these $n_R$ gaps, is bounded by $2n_R\beta(\ell)$. Applying the same conditional product marking kernel to both laws does not increase total variation. Averaging over $U$ proves Eq.~\eqref{eq:formula-G-10}. This compares the joint experiments, preserving the same generation-information marginal.
\end{remark}

\begin{remark}
\label{rem:recollection-wait}
In this construction the last stored record arrives by $b+(n_R+1)\ell-1$. The declared recollection horizon includes the initial wait and all subsequent gaps. Collection is scheduled from the fixed cutoff $b$ even if generation material becomes available earlier. Evaluation readiness additionally requires at least one retained mark, as specified in Procedure~\ref{alg:finite-horizon-update}.
\end{remark}

\paragraph{Concentration radii.}
Lemma~\ref{lem:value-estimation} gives $\varepsilon_V=\sqrt{\log(2/\beta_V^{(1)})/(2m)}$. Lemma~\ref{lem:failure-rate-estimation} gives the mixing-aware $\varepsilon_Z$ in Eq.~\eqref{eq:formula-4-25}. For $f=\operatorname{Adv}_{A_k}(\tilde A,\cdot)$, write $\varepsilon_F:=r_{F,k}(f)$, $\varepsilon_\mu:=2\varepsilon_V+\varepsilon_F$, and $\varepsilon_L:=\varepsilon_Z+(\hat Z_k+\varepsilon_Z)\varepsilon_\mu$.
Lemma~\ref{lem:improvement-estimation} controls $|\hat L_k-L_{A_k}(\tilde A)|$ by $\varepsilon_L$ on the stated coverage event. Lemma~\ref{lem:retained-change-estimation} gives
\[
\varepsilon_D:=\frac{2\varepsilon_V+\varepsilon_R+\varepsilon_Z}{1-\hat Z_k},
\qquad
\varepsilon_R:=\sqrt{\frac{\log(2/\widetilde\beta_R)}{2n_R}},
\]
on the nonempty-retained branch. The exact events and probability allocations are recorded with the corresponding lemmas below and in Appendix~\ref{app:section-G-5}.

\begin{proposition}[First-failure sampling on an i.i.d. task stream]
\label{prop:first-failure-sampling}
Let $\mathsf{Run}_{F,k}$ be the $\mathsf{Pre}_k$-measurable event that Phase E is entered. Suppose that, conditionally on $\mathsf{Pre}_k$, $0<Z_k<1$ and $1\le n_F\le H_{F,k}<\infty$, and, on $\mathsf{Run}_{F,k}$, the next $H_{F,k}$ task arrivals form i.i.d. marked pairs $(T_s,\Phi_s)$ with $T_s\sim\mathcal D_{\mathrm{user}}$ and $\Phi_s\mid T_s\sim\mathrm{Bernoulli}(\psi_{A_k}(T_s))$. The current agent and failure detector remain fixed throughout the window. The evaluator returns the first $n_F$ tasks with $\Phi_s=1$ if at least $n_F$ such marks occur, and otherwise returns $\bot$.
\begin{resultparts}
\item\label{part:first-failure-distribution} Conditional on $\mathsf{Pre}_k$ and $\mathsf{Ret}_{F,k}$, the returned tasks are i.i.d. from $\mathcal D_{F,k}$. Hence, for a frozen family of $1\le M<\infty$ bounded statistics with deterministic $f\in[a_f,b_f]$,
\[
r_{F,k}(f)=(b_f-a_f)\sqrt{\frac{\log(2M/\beta_{F,k})}{2n_F}}
\]
instantiates Condition~\ref{cond:returned-sample-coverage}. \item\label{part:first-failure-return} With $B_{H,n}(z):=\mathbb P\{\mathrm{Binomial}(H,z)\ge n\}$,
\[
\mathbb P(\mathsf{Ret}_{F,k}\mid\mathsf{Pre}_k)
=\mathbf1_{\mathsf{Run}_{F,k}}B_{H_{F,k},n_F}(Z_k).
\]
\item\label{part:first-failure-planning} If a clipped $\mathsf{Pre}_k$-measurable lower bound $\underline Z_k$ satisfies $\mathbb P(\mathsf{Run}_{F,k}\cap\{\underline Z_k>Z_k\})\le\beta_{Z,\mathrm{plan}}$ and, whenever Phase E is entered, $\underline Z_k>0$ and $B_{H_{F,k},n_F}(\underline Z_k)\ge1-\alpha_{F,\mathrm{ret}}$, then
\[
\mathbb P(\mathsf{Run}_{F,k}\cap\mathsf{Ret}_{F,k}^{c})
\le\beta_{Z,\mathrm{plan}}+\alpha_{F,\mathrm{ret}}.
\]
In particular, if $\mathsf{Run}_{F,k}$ holds almost surely, then $\mathbb P(\mathsf{Ret}_{F,k})\ge1-\beta_{Z,\mathrm{plan}}-\alpha_{F,\mathrm{ret}}$.
\end{resultparts}
\end{proposition}

More generally, a sampler-specific analysis may prove, conditionally on each realized $\mathsf{Pre}_k$,
\[
\mathbb P\!\left(\mathsf{Ret}_{F,k}\cap
\{|\bar f_{F,k}-\mu_{F,k}(f)|>x\}\mid\mathsf{Pre}_k\right)
\le2\exp\!\left(-\frac{2n_Fx^2}{q_{F,k}(b_f-a_f)^2}\right).
\]
A union bound then gives $r_{F,k}(f)=(b_f-a_f)\sqrt{q_{F,k}\log(2M/\beta_{F,k})/(2n_F)}$. The value $q_{F,k}$ must be justified by the sampling model; for the advantage statistic, the i.i.d. construction gives $\varepsilon_F=\sqrt{2\log(2/\beta_{F,k})/n_F}$.

\begin{remark}
\label{rem:stored-label-recollection}
The stored labels have probabilities $\psi_{A_k}(t_j^R)$ for the current agent at collection. Adoption both changes the reference agent and uses these records in selecting the successor. Recollection therefore supplies current-agent labels and data satisfying Condition~\ref{cond:evaluation-data-separation}. A sample carried after rejection must satisfy the same condition for its next use.
\end{remark}

\begin{procedure}[Generation, validation, and update]
\label{alg:finite-horizon-update}
Use $A_k$ throughout the step.
\begin{enumerate}
\item Observe at most $H_{\mathrm{gen},k}$ task arrivals and return the first $n_F$ failures as generation material. If the quota is not met, abstain. Otherwise draw $\Delta C\sim A_k(\cdot\mid t_{F_k^{\mathrm{gen}}})$; if $\Delta C=\bot$, abstain.
\item For $\Delta C\in\mathcal M$, form $\tilde A=A_k\oplus\Delta C$. Before Phase E, freeze the candidate agent, finite statistic family, ranges and radii, risk allocation, failure detector, Phase-E return rule, and $H_{F,k}$.
\item Use one predeclared carried sample or finite recollection window $H_{R,k}$ satisfying Condition~\ref{cond:evaluation-data-separation}. If it does not yield a valid nonempty retained subset, abstain.
\item In Phase E, observe the next $H_{F,k}$ task arrivals with the current agent and detector fixed. Return $F_k^{\mathrm{gate}}$, consisting of the first $n_F$ failures, or abstain if the quota is not met.
\item On $\mathsf{Cand}_k\cap\mathsf{Ready}_k$, perform the independent fresh rollouts of both agents and compute the validation statistics and $\Delta_k$ in Eq.~\eqref{eq:formula-4-28}. Set $\Delta_k=0$ off this branch.
\item On $\mathsf{Cand}_k\cap\mathsf{Ready}_k$, validate the modification if and only if Eqs.~\eqref{eq:formula-4-23} and~\eqref{eq:formula-4-28} hold; set $\mathrm{decision}_k=\mathrm{accept}$ exactly when it is validated and $\mathrm{reject}$ otherwise. Apply the update of Definition~\ref{def:evolution-step}; adoption marks the stored sample for recollection.
\end{enumerate}
\end{procedure}

The runtime check Eq.~\eqref{eq:formula-4-28} uses only $\hat Z_k$, $\varepsilon_Z$, $\varepsilon_D$, and $\varepsilon_L$, all observable at decision time. After the samples return, evaluating one candidate and the current agent uses $2m(n_F+|S_R|)$ complete rollouts for the two validation statistics; task-arrival horizons are accounted for separately.

The error bounds for $F_k^{\mathrm{gate}}$ and $R_k$ are combined by a union bound. They may share evaluation records while satisfying their respective sampling conditions.

When both generation and evaluation batches return, their quotas total $2n_F$ failures. Under i.i.d.\ arrivals, $2n_F/Z_k$ is the corresponding infinite-horizon planning scale; finite-window return is governed by Proposition~\ref{prop:first-failure-sampling}.

\begin{proposition}
\label{prop:rollout-jensen-bounds}
For each $t$, let $\delta(t) := V(\tilde A, t) - V(A_k, t)$ and $\hat\delta(t) := \hat V(\tilde A, t) - \hat V(A_k, t)$ (based on the independent rollouts Eq.~\eqref{eq:formula-4-14}). Then:

\begin{resultparts}
\item\label{part:jensen-direction}For every $t$, $\mathbb E[|\hat\delta(t)|]\ge|\delta(t)|$. Equality holds if and only if $\hat\delta(t)$ is almost surely nonnegative or almost surely nonpositive; equivalently, the inequality is strict exactly when $\mathbb P(\hat\delta(t)>0)>0$ and $\mathbb P(\hat\delta(t)<0)>0$.

\item\label{part:jensen-magnitude}let $\sigma_t^2 := \mathrm{Var}_{o \sim A_k(\cdot|t)}[r(t, o)] + \mathrm{Var}_{o \sim \tilde A(\cdot|t)}[r(t, o)]$. Then $\sigma_t^2 \leq 1/2$, and:
\begin{equation}
0 \leq \mathbb{E}[|\hat\delta(t)|] - |\delta(t)| \leq \sqrt{\mathrm{Var}[\hat\delta(t)]} = \sigma_t/\sqrt{m} \leq \frac{1}{\sqrt{2m}}.
\label{eq:formula-4-20}
\end{equation}

\item\label{part:jensen-conditional-comparison}let $\mathcal G:=\sigma(\{(t_j^R,\phi_j^R)\}_{j\le n_R})\vee\sigma(\tilde A)$. On the nonempty-retained branch,
\begin{equation}
\mathbb{E}\bigl[\hat D_R\mid\mathcal{G}\bigr]\;\ge\;\hat D_R^{\mathrm{exact}}:=\frac{1}{|S_R|}\sum_{j\in S_R}\bigl|\delta(t^R_j)\bigr|\qquad\text{a.s. on }\{|S_R|\ge1\}.
\label{eq:formula-F-3}
\end{equation}
\end{resultparts}
\end{proposition}

For fixed current and candidate agents and independently collected i.i.d.\ marked task records, averaging Eq.~\eqref{eq:formula-F-3} conditional on $|S_R|\ge1$ gives $\mathbb E[\hat D_R\mid |S_R|\ge1]\ge D_R$. For a random candidate, the same comparison is conditional on the generation information under the product reference experiment of Condition~\ref{cond:evaluation-data-separation}.

\begin{proof}[Proof of Proposition~\ref{prop:rollout-jensen-bounds}.]

(\ref{part:jensen-direction}) By Jensen's inequality, $\mathbb E|X|\ge|\mathbb EX|$. Applying it to $X=\hat\delta(t)$ and using $\mathbb EX=\delta(t)$ gives the stated direction. Equality in Jensen holds exactly when $X$ lies almost surely in one affine branch of $|\cdot|$, i.e. when $X\ge0$ a.s. or $X\le0$ a.s. Hence the inequality is strict exactly when $X$ takes both positive and negative values with positive probability.

(\ref{part:jensen-magnitude}) The variance of $\hat\delta(t) - \delta(t)$ equals the sum of the variances of the two independent groups of rollouts: $\mathrm{Var}[\hat\delta(t)] = \sigma_t^2/m$. Since $r \in [0,1]$, $\mathrm{Var}(r) \leq 1/4$ per agent, hence $\sigma_t^2 \leq 1/2$.

By $|x| - |y| \leq |x - y|$ together with Jensen applied to $\sqrt{\cdot}$:
\[
\mathbb{E}|\hat\delta(t)| - |\delta(t)| \leq \mathbb{E}|\hat\delta(t) - \delta(t)| \leq \sqrt{\mathrm{Var}[\hat\delta(t)]} = \sigma_t/\sqrt{m}.
\]

(\ref{part:jensen-conditional-comparison}) Work on $\{|S_R|\ge1\}$. Given $\mathcal G$, both $S_R$ and its cardinality are fixed, so
\[
\mathbb{E}\bigl[\hat D_R\mid\mathcal{G}\bigr]=\frac{1}{|S_R|}\sum_{j\in S_R}\mathbb{E}\bigl[|\hat\delta_j|\mid\mathcal{G}\bigr]=\frac{1}{|S_R|}\sum_{j\in S_R}h_{\tilde A}(t^R_j),\qquad h_{\tilde A}(t):=\mathbb{E}\bigl[|\hat\delta(t)|\bigr],
\]
where the fresh rollouts obey Eq.~\eqref{eq:formula-4-14} and Assumptions~\ref{ass:outcome-evaluation-independence} and \ref{ass:trial-consistency} for the fixed candidate agent. Applying (\ref{part:jensen-direction}) to each summand proves Eq.~\eqref{eq:formula-F-3}. \end{proof}

\begin{lemma}
\label{lem:value-estimation}
Under Eq.~\eqref{eq:formula-4-14}, for any $A, t$:
\[
\mathbb{P}\bigl[|\hat V(A, t) - V(A, t)| \geq \varepsilon_V\bigr] \leq 2\exp(-2 m \varepsilon_V^2) =: \beta_V^{(1)}.
\]
That is, $\varepsilon_V = \sqrt{\log(2/\beta_V^{(1)})/(2m)}$.
\end{lemma}

\begin{proof}[Proof of Lemma~\ref{lem:value-estimation}.] Under Eq.~\eqref{eq:formula-4-14} and Assumption~\ref{ass:bounded-reward}, the $m$ rollout rewards are independent and lie in $[0,1]$. Hoeffding's inequality for their sample mean therefore gives
\[
\mathbb P\bigl[|\hat V(A,t)-V(A,t)|\ge\varepsilon_V\bigr]
\le 2\exp(-2m\varepsilon_V^2).
\]
Setting the right-hand side to $\beta_V^{(1)}$ and solving for $\varepsilon_V$ yields the stated radius. \end{proof}

\begin{lemma}
\label{lem:failure-rate-estimation}
{Under Assumptions~\ref{ass:mixing-task-process}, \ref{ass:failure-detector-locality}, \ref{ass:task-outcome-model}, and \ref{ass:outcome-evaluation-independence} and the stored-sample comparison of Condition~\ref{cond:evaluation-data-separation}, with block size $\ell$, the common-offset block sample defined above, and $2n_R\beta(\ell)<\beta_Z$:}
\begin{equation}
\mathbb{P}[|\hat Z_k - Z_k| \geq \varepsilon_Z] \leq 2\exp\bigl(-2 n_R \varepsilon_Z^2\bigr) + 2 n_R\,\beta(\ell) =: \beta_Z.
\label{eq:formula-4-25}
\end{equation}
\end{lemma}

The corresponding $\varepsilon_Z = \sqrt{\frac{\log(2/\tilde\beta_Z)}{2 n_R}}$, where $\tilde\beta_Z = \beta_Z - 2 n_R \beta(\ell)$.

\begin{proof}[Proof of Lemma~\ref{lem:failure-rate-estimation}.] Under the reference law of Condition~\ref{cond:evaluation-data-separation}, conditionally on $\mathcal G_k$, the labels are i.i.d.\ Bernoulli$(Z_k)$ after integrating the independent tasks. Hoeffding bounds the reference probability of $|\hat Z_k-Z_k|\ge\varepsilon_Z$ by $2e^{-2n_R\varepsilon_Z^2}$. Transferring this joint event by Eq.~\eqref{eq:formula-G-10} adds at most $2n_R\beta(\ell)$, proving Eq.~\eqref{eq:formula-4-25}. \end{proof}

Write
\[
\mathcal E_Z:=\{|\hat Z_k-Z_k|\le\varepsilon_Z\},\qquad
\mathcal E_V^{(B)}:=
\bigcap_{t\in B}\ \bigcap_{A\in\{A_k,\tilde A\}}
\{|\hat V(A,t)-V(A,t)|\le\varepsilon_V\},
\]
where $B=F_k^{\mathrm{gate}}$ or $R_k$ denotes the corresponding list of task records; abbreviate these events as $\mathcal E_V^{(F)}$ and $\mathcal E_V^{(R)}$. The fresh-rollout union bounds on available records are $2n_F\beta_V^{(1)}$ and $2n_R\beta_V^{(1)}$, respectively.

\begin{lemma}
\label{lem:improvement-estimation}
Under Assumption~\ref{ass:bounded-reward}, assume the $\hat Z_k$ coverage of Lemma~\ref{lem:failure-rate-estimation}, the per-task fresh-rollout coverage of Lemma~\ref{lem:value-estimation}, and Condition~\ref{cond:returned-sample-coverage} for $f=\operatorname{Adv}_{A_k}(\tilde A,\cdot)$. Let
\[
\mu_F:=\mathbb E_{\mathcal D_{F,k}}[\operatorname{Adv}_{A_k}(\tilde A,T)],
\qquad L_{A_k}(\tilde A)=Z_k\mu_F,
\]
and set $\varepsilon_F:=r_{F,k}(f)$ and $\varepsilon_\mu:=2\varepsilon_V+\varepsilon_F$. On the returned branch write
\[
\mathcal E_{F,k}(f):=\{ |\bar f_{F,k}-\mu_{F,k}(f)|\le\varepsilon_F\}.
\]
On $\mathsf{Ready}_k\cap\mathcal E_Z\cap\mathcal E_V^{(F)}\cap\mathcal E_{F,k}(f)$,
\begin{equation}
|\hat L_k-L_{A_k}(\tilde A)|
\le\varepsilon_Z+Z_k\varepsilon_\mu
\le\varepsilon_L:=\varepsilon_Z+(\hat Z_k+\varepsilon_Z)\varepsilon_\mu.
\label{eq:formula-4-26}
\end{equation}
Moreover,
\[
\mathbb P\!\left(\mathsf{Ready}_k\cap
\{|\hat L_k-L_{A_k}(\tilde A)|>\varepsilon_L\}\right)
\le\beta_Z+2n_F\beta_V^{(1)}+\beta_{F,k}.
\]
\end{lemma}

\begin{proof}[Proof of Lemma~\ref{lem:improvement-estimation}.]
Work on $\mathsf{Ready}_k$, which is a subset of $\mathsf{Ret}_{F,k}$, and define
\[
\hat\mu_F:=\frac1{n_F}\sum_{i=1}^{n_F}
\bigl(\hat V(\tilde A,T^F_{k,i})-\hat V(A_k,T^F_{k,i})\bigr),
\quad
\bar\mu_F:=\frac1{n_F}\sum_{i=1}^{n_F}
\operatorname{Adv}_{A_k}(\tilde A,T^F_{k,i}).
\]
The product identity gives
\[
\hat L_k-L_{A_k}(\tilde A)=(\hat Z_k-Z_k)\hat\mu_F+Z_k(\hat\mu_F-\mu_F).
\]
On $\mathcal E_Z$, the first term is at most $\varepsilon_Z$, since $|\hat\mu_F|\le1$.  On $\mathcal E_V^{(F)}$, fresh-rollout coverage gives $|\hat\mu_F-\bar\mu_F|\le2\varepsilon_V$.  Finally, Condition~\ref{cond:returned-sample-coverage} gives
\[
\mathbb P\!\left(\mathsf{Ret}_{F,k}\cap
\{|\bar\mu_F-\mu_F|>\varepsilon_F\}\right)\le\beta_{F,k}.
\]
Consequently the triangle inequality yields
$|\hat L_k-L_{A_k}(\tilde A)|\le\varepsilon_Z+Z_k(2\varepsilon_V+\varepsilon_F)$ on the stated coverage event.  The inequality $Z_k\le\hat Z_k+\varepsilon_Z$ on $\mathcal E_Z$ gives the observable radius in Eq.~\eqref{eq:formula-4-26}.  A union bound over the estimator failures on the ready branch proves the displayed joint error bound. \end{proof}

\begin{proof}[Proof of Proposition~\ref{prop:first-failure-sampling}.]
Condition on $\mathsf{Pre}_k$.  On $\mathsf{Run}_{F,k}$ the current agent, detector, horizon, and return rule are fixed.  For any complete mark vector $\varphi\in\{0,1\}^{H_{F,k}}$ with at least $n_F$ successes, the i.i.d. marked-pair distribution factorizes across positions.  At every success position,
\[
\mathbb P(T_s\in dt\mid\Phi_s=1,\mathsf{Pre}_k)
=\frac{\psi_{A_k}(t)\,\mathcal D_{\mathrm{user}}(dt)}{Z_k}
=\mathcal D_{F,k}(dt).
\]
The success-position tasks are conditionally independent, while integration over the nonfailure positions contributes a factor independent of their values.  This distribution is the same for every such $\varphi$.  Mixing over all return-producing mark vectors therefore leaves the first $n_F$ failure tasks i.i.d. from $\mathcal D_{F,k}$ after conditioning on return.

For each frozen $f\in\mathcal Q_k^{\mathrm{eval}}$, Hoeffding's inequality under this returned-task distribution, followed by a union bound over the $M$ declared statistics, gives part (\ref{part:first-failure-distribution}) and the event-level Condition~\ref{cond:returned-sample-coverage} bound.  On $\mathsf{Run}_{F,k}$ the failure marks are i.i.d. Bernoulli$(Z_k)$, so the return probability is $B_{H_{F,k},n_F}(Z_k)$; off that event the rule returns $\bot$.  This proves part (\ref{part:first-failure-return}).

For part (\ref{part:first-failure-planning}), $B_{H,n}(z)$ is nondecreasing in $z$.  On
$\mathsf{Run}_{F,k}\cap\{\underline Z_k\le Z_k\}$ the declared tail condition therefore makes conditional nonreturn probability at most $\alpha_{F,\mathrm{ret}}$.  Adding the probability of
$\mathsf{Run}_{F,k}\cap\{\underline Z_k>Z_k\}$ proves the stated bound. \end{proof}

\subsection{Retained-task estimation and the reliable-adoption proof}
\label{app:section-G-5}\label{app:radius-accounting}

\begin{lemma}
\label{lem:retained-change-estimation}
Under Assumptions~\ref{ass:bounded-reward}, \ref{ass:mixing-task-process}, \ref{ass:failure-detector-locality}, \ref{ass:task-outcome-model}, \ref{ass:outcome-evaluation-independence}, and \ref{ass:trial-consistency}, Condition~\ref{cond:evaluation-data-separation}, and
$2n_R\beta(\ell)<\min\{\beta_Z,\beta_R\}$, put
\[
g_k(t):=|V(\tilde A,t)-V(A_k,t)|,\quad
\bar W_g:=\frac1{n_R}\sum_j(1-\phi_j^R)g_k(t_j^R),\quad
\widetilde\beta_R:=\beta_R-2n_R\beta(\ell)>0.
\]
Define
\[
\varepsilon_R:=\sqrt{\frac{\log(2/\widetilde\beta_R)}{2n_R}},
\quad
\mathcal E_R:=\{|\bar W_g-(1-Z_k)D_R|\le\varepsilon_R\},
\quad
\mathcal E_D:=\mathcal E_Z\cap\mathcal E_V^{(R)}\cap\mathcal E_R.
\]
On $\{|S_R|\ge1\}\cap\mathcal E_D$,
\begin{equation}
|\hat D_R-D_R|\le\varepsilon_D:=
\frac{2\varepsilon_V+\varepsilon_R+\varepsilon_Z}{1-\hat Z_k}.
\label{eq:formula-4-27}
\end{equation}
The corresponding joint error probability on the ready branch is at most
$\beta_Z+\beta_R+2n_R\beta_V^{(1)}$.
\end{lemma}

\begin{proof}[Proof of Lemma~\ref{lem:retained-change-estimation}.]
First separate fresh-rollout error. Let
$\hat g_j:=|\hat V(\tilde A,t_j^R)-\hat V(A_k,t_j^R)|$
and $W_j:=(1-\phi_j^R)\hat g_j$. On $\mathcal E_V^{(R)}$, the triangle inequality gives
\begin{equation}
|\hat g_j-g_k(t_j^R)|\le2\varepsilon_V.
\label{eq:formula-G-7}
\end{equation}
Consequently, with
\begin{equation}
\bar W_g=\frac1{n_R}\sum_j(1-\phi_j^R)g_k(t_j^R),
\label{eq:formula-G-8}
\end{equation}
we have
\begin{equation}
\left|\frac1{n_R}\sum_jW_j-\bar W_g\right|\le2\varepsilon_V.
\label{eq:formula-G-9}
\end{equation}

Under the reference law in Condition~\ref{cond:evaluation-data-separation}, condition on $\mathcal G_k$. The function $g_k$ is fixed and the reference weighted observations are independent in $[0,1]$, with mean
\[
\int(1-\psi_{A_k}(t))g_k(t)\,\mathcal D_{\mathrm{user}}(dt)=(1-Z_k)D_R.
\]
Conditional Hoeffding and then averaging over the unchanged $\mathcal G_k$ marginal give reference failure probability $2e^{-2n_R\varepsilon_R^2}$. The joint-law comparison Eq.~\eqref{eq:formula-G-10} therefore yields
\begin{equation}
\mathbb P\{|\bar W_g-(1-Z_k)D_R|>\varepsilon_R\}
\le2e^{-2n_R\varepsilon_R^2}+2n_R\beta(\ell)=\beta_R.
\label{eq:formula-G-11}
\end{equation}
Fresh rollouts retain their prescribed conditional distributions given the observed records and frozen agents. Their error event and $\mathcal E_R$ are combined by a union bound.

Finally, on $\mathcal E_V^{(R)}\cap\mathcal E_R$,
\begin{equation}
\left|\frac1{n_R}\sum_jW_j-(1-Z_k)D_R\right|
\le2\varepsilon_V+\varepsilon_R.
\label{eq:formula-G-12}
\end{equation}
Since the sample numerator is $(1-\hat Z_k)\hat D_R$,
\begin{equation}
|(1-\hat Z_k)\hat D_R-(1-Z_k)D_R|
\le2\varepsilon_V+\varepsilon_R.
\label{eq:formula-G-13}
\end{equation}
On the nonempty retained branch, $1-\hat Z_k=|S_R|/n_R>0$ and
\[
\hat D_R-D_R=
\frac{[(1-\hat Z_k)\hat D_R-(1-Z_k)D_R]+D_R(\hat Z_k-Z_k)}
{1-\hat Z_k}.
\]
Using $D_R\in[0,1]$ and $\mathcal E_Z$ gives
\begin{equation}
|\hat D_R-D_R|\le
\frac{2\varepsilon_V+\varepsilon_R+\varepsilon_Z}{1-\hat Z_k}.
\label{eq:formula-G-14}
\end{equation}
The union of the three coverage failures has probability at most
$\beta_Z+\beta_R+2n_R\beta_V^{(1)}$ on available records. \end{proof}

\begin{remark}
\label{rem:stored-sample-radii}
The reference sample has size $n_R$. Dependence contributes the additive error $2n_R\beta(\ell)$, leaving the Hoeffding radius
$\sqrt{\log(2/\widetilde\beta)/(2n_R)}$ with
$\widetilde\beta=\beta-2n_R\beta(\ell)>0$.
At a fixed gap with $\beta(\ell)>0$, the chosen $n_R$ must satisfy $2n_R\beta(\ell)<\min\{\beta_Z,\beta_R\}$ for both stored-sample estimates; increasing $n_R$ without enlarging the gap need not improve these guarantees. For an i.i.d.\ comparison, $\beta(\ell)=0$ and the ordinary independent-sample radius applies.
Failure-task concentration uses the separate returned-sample radius of
Condition~\ref{cond:returned-sample-coverage}.
\end{remark}

\paragraph{Proof of the reliable-adoption guarantee.}
\begin{proof}[Proof of Theorem~\ref{thm:one-step-guarantee}.] Let $\mathcal E_k$ be the intersection of the coverage events on the ready branch in Lemmas~\ref{lem:improvement-estimation} and \ref{lem:retained-change-estimation}, and set
\[
\mathsf{Good}_k:=\mathsf{Ready}_k^c\cup(\mathsf{Ready}_k\cap\mathcal E_k).
\]
The two lemmas and a union bound give
\begin{equation}
\mathbb P(\mathsf{Good}_k^c)\le
2\beta_Z+\beta_{F,k}+\beta_R+2(n_F+n_R)\beta_V^{(1)}
=\beta_{\mathrm{step},k}.
\label{eq:step-risk}
\end{equation}
Because $\mathsf{Acc}_k\subseteq\mathsf{Ready}_k$, every accepted path in $\mathsf{Good}_k$ lies in $\mathcal E_k$; in particular the retained denominator is valid and both estimator bounds apply.  On such a path, Theorem~\ref{thm:expected-reward-improvement} and the two bounds give
\[
I_k\ge L_{A_k}(\tilde A)-(1-Z_k)D_R(\tilde A;A_k)
\ge\hat L_k-\varepsilon_L
-(1-\hat Z_k+\varepsilon_Z)(\hat D_R+\varepsilon_D).
\]
The Two-Gate inequalities $\hat L_k\ge\tau$ and $\hat D_R\le\delta$ therefore imply $I_k\ge\Delta_k$.  Hence
$\mathsf{Acc}_k\cap\{I_k<\Delta_k\}\subseteq\mathsf{Good}_k^c$, which proves the quantitative bound.  The runtime condition $\Delta_k>0$ gives the stated non-degradation consequence. \end{proof}

\subsection{Generation procedures and pool selection}
\label{app:reachability-scope}\label{app:proofs-generation}

Membership in $\mathcal P_k(T)$ is determined by true modification effects under $\mathcal D_{\mathrm{user}}$ and the declared target. At fixed generation law, reachability is non-increasing in $\lambda$ and non-decreasing in $\delta$.

A generation procedure fixes how failure material, history $\mathcal H_k$, shared context $\Xi_k=\xi$, and its predeclared internal steps produce one terminal output. It induces $\pi_{k,\xi}$ on $\mathcal M_\bot$, including failures. Eq.~\eqref{eq:formula-2-5} supplies the current agent's kernel on a given encoded failure batch; collecting that batch is included when specified by the procedure.

\begin{definition}[Selection from a candidate pool]
\label{def:pool-selector}
Conditional on $(\mathcal H_k,\Xi_k=\xi)$, draw $C_{1:N}\stackrel{\mathrm{iid}}{\sim}\pi_{k,\xi}$ on $\mathcal M_\bot$. A selection rule may use arbitrary evaluation material to choose a valid drawn modification, or retain the current agent, but may not repair, compose, or generate a new modification. It succeeds for $T$ only when the selected modification $C_i$ from the candidate pool lies in $\mathcal P_k(T)$; retaining the current agent is not counted. The pool size $N$ includes failed terminal outputs.
\end{definition}

\begin{remark}
\label{rem:pool-selection-scope}
The pool-selection bound permits arbitrary evaluation of the realized pool. Fixed modification-wise repair induces a new distribution (Proposition~\ref{prop:regeneration-probability}); finite predeclared assembly is covered by Theorem~\ref{thm:finite-assembly-bound} (both in Appendix~\ref{app:section-G-6}).
\end{remark}

\paragraph{Pool-selection upper bound.}

Fix a target $T$ and $(\mathcal H_k,\Xi_k=\xi)$, and draw $C_{1:N}\stackrel{\mathrm{iid}}{\sim}\pi_{k,\xi}$. For a selection rule of Definition~\ref{def:pool-selector}, let $R_{k,N}(T\mid\mathcal H_k,\xi)$ be the conditional probability of selecting a modification in $\mathcal P_k(T)$. Write
\[
\mathsf C_N(T\mid\mathcal H_k,\xi)
:=1-\bigl(1-P_k(T\mid\mathcal H_k,\xi)\bigr)^N.
\]
This is the probability that the candidate pool contains at least one qualified modification. When the conditioning is fixed, we write $\mathsf C_N(T)$.

\begin{theorem}[Candidate-pool selection bound]
\label{thm:candidate-pool-selection}
Fix a target $T$ and $(\mathcal H_k,\Xi_k=\xi)$, and let $C_{1:N}\stackrel{\mathrm{iid}}{\sim}\pi_{k,\xi}$. Any selector of Definition~\ref{def:pool-selector} satisfies

\begin{equation}
R_{k,N}(T\mid\mathcal H_k,\xi)
\le
1-\bigl(1-P_k(T\mid\mathcal H_k,\xi)\bigr)^N
\le
N P_k(T\mid\mathcal H_k,\xi).
\label{eq:formula-3-26}
\end{equation}
\end{theorem}

An oracle that identifies a qualified modification whenever one is present attains $\mathsf C_N$. No downstream rule restricted to choosing from the candidate pool can exceed this probability.

\begin{proof}[Proof of Theorem~\ref{thm:candidate-pool-selection}.] Conditional on $(\mathcal H_k,\Xi_k=\xi)$, the selector can choose a modification from the candidate pool that belongs to $\mathcal P_k(T)$ only if at least one of the $N$ i.i.d.\ terminal outputs lies in that set.  The probability of this necessary event is $1-(1-P_k(T\mid\mathcal H_k,\xi))^N$.  Evaluation changes only the choice within the candidate pool.  Bernoulli's inequality gives the second bound. \end{proof}

\begin{proposition}
\label{prop:self-task-separation}
In the kernel setting of Appendix~\ref{app:section-A-1}, allow the composition map $(M,C)\mapsto A$ to be any measurable map consistent with Eqs.~\eqref{eq:formula-2-1}--\eqref{eq:formula-2-5}. For every fixed integer $N\ge1$, there are a common target $T$ and two configurations with the same frozen LLM, identical behavior on user tasks, and the same qualified set $\mathcal P_k(T)$, but whose generation procedures assign different probabilities to this set and hence yield different values of $\mathsf C_N(T)$.
\end{proposition}

\begin{proof}[Proof of Proposition~\ref{prop:self-task-separation}.] Let two harness states differ only in their self-task fragment, and let their allowed modifications preserve that fragment. Choose the composition maps so that both current agents' user-task kernels, and both modified kernels for each common modification, agree. Their qualified sets consequently agree.

To realize different probabilities under the composition maps allowed in Proposition~\ref{prop:self-task-separation}, take two user tasks $t_0,t_1$ of equal probability, binary user-task outputs, reward $r(t,o)=o$, and detector $\phi(t,o)=1-o$. Both current agents output $0$ on $t_0$ and $1$ on $t_1$. Let $\mathcal M=\{c_+,c_0\}$, where $c_+$ changes the output on $t_0$ to $1$ and leaves $t_1$ unchanged, while $c_0$ preserves both outputs. These modifications leave the self-task fragment unchanged. Then $Z=1/2$, $(L(c_+),D(c_+))=(1/2,0)$, and $(L(c_0),D(c_0))=(0,0)$. For the common target $T=(1/4,0)$, both qualified sets are $\{c_+\}$. Choose the two self-task kernels to assign probabilities $p$ and $p'$ to $c_+$, where $0<p'<p<1$. These choices define measurable kernels with the same frozen $M$ and identical user-task behavior. Since $p\mapsto1-(1-p)^N$ is strictly increasing for $N\ge1$, the two values of $\mathsf C_N(T)$ differ. \end{proof}

\subsection{Validation and adoption from a candidate pool}
\label{app:proofs-pool}\label{app:section-G-1}

For each candidate-pool output $C_i$, write $L_i:=L_k(C_i)$ and $D_i:=D_k(C_i)$, with estimates $\widehat L_i,\widehat D_i$. Nonnegative error bounds $e^\ell_{L,i},e^u_{L,i},e^\ell_{D,i},e^u_{D,i}$ specify the lower and upper endpoints of the intervals in Eq.~\eqref{eq:formula-4-5}. The candidate-pool validation rule \(G\) in Section~\ref{sec:task-stream-updates} is
\begin{equation}
G_i(T)=1
\quad\Longleftrightarrow\quad
\widehat L_i-e^\ell_{L,i}\ge\lambda,
\qquad
\widehat D_i+e^u_{D,i}\le\delta.
\label{eq:formula-4-6}
\end{equation}

The following restates Theorem~\ref{thm:certified-selection} with its complete conditions. For a failed output $C_i=\bot$, use $L_i=D_i=\widehat L_i=\widehat D_i=0$ and zero error radii; only valid outputs require candidate-agent evaluation.

The intervals cover the effects defined by the fixed reward and target. Evaluation through another signal requires a calibration allowance sufficient for this coverage.

\begin{theoremrestatement}[Probability of achieving harness-level reliable self-evolution]{thm:certified-selection}
Fix $T$ and $(\mathcal H_k,\Xi_k=\xi)$. For modification $C_i$, write $L_i:=L_k(C_i)$ and $D_i:=D_k(C_i)$. Let $\mathcal E_{N,\beta}$ be the event that, for all $i=1,\ldots,N$,
\begin{equation}
L_i\in[\widehat L_i-e^\ell_{L,i},\widehat L_i+e^u_{L,i}],\qquad
D_i\in[\widehat D_i-e^\ell_{D,i},\widehat D_i+e^u_{D,i}],
\label{eq:formula-4-5}
\end{equation}
with, on every sample path,
\begin{equation}
e^\ell_{L,i}+e^u_{L,i}\le\bar w_L,\qquad
e^\ell_{D,i}+e^u_{D,i}\le\bar w_D.
\label{eq:formula-4-5a}
\end{equation}
Suppose: (i) $C_{1:N}\stackrel{\mathrm{iid}}{\sim}\pi_{k,\xi}$; (ii) the evaluation material used by the validation rule is separated from the generation material; (iii) the modification-independent deterministic upper bounds $\bar w_L,\bar w_D$ are fixed before the candidate pool and dominate every realized interval width as in Eq.~\eqref{eq:formula-4-5a}; (iv) the event $\mathcal E_{N,\beta}$ satisfies
\begin{equation}
\mathbb P(\mathcal E_{N,\beta}\mid\mathcal H_k,\xi,C_{1:N})\ge1-\beta_{\mathrm{stat}},
\label{eq:formula-4-8}
\end{equation}
and (v) $\sigma$ is a selector as specified in Theorem~\ref{thm:certified-selection}. Then, on $\mathcal E_{N,\beta}$, every validated modification belongs to $\mathcal P_k(T)$ and every candidate-pool member in $\mathcal P^+_{k,N,\mathbf n,\beta}(T)$ is validated. Consequently the lower and upper bounds in Eq.~\eqref{eq:formula-4-9} hold. For $T=(\lambda_k(\gamma,\delta),\delta)$ with $\gamma>0$, the lower bound also applies to the probability of an update with overall expected-reward gain at least $\gamma$ by Corollary~\ref{cor:qualified-safe-improvement}.
\end{theoremrestatement}

\begin{proof}[Proof of Theorem~\ref{thm:certified-selection}.] Write $\mathcal E=\mathcal E_{N,\beta}$. A failed output has deterministic coverage and zero interval widths; since $\lambda>0$, it is neither qualified nor in $\mathcal P^+_{k,N,\mathbf n,\beta}(T)$, and Eq.~\eqref{eq:formula-4-6} does not validate it. For valid outputs, on $\mathcal E$, if $G_i(T)=1$, then Eq.~\eqref{eq:formula-4-5} and Eq.~\eqref{eq:formula-4-6} give
\[
L_i\ge\widehat L_i-e^\ell_{L,i}\ge\lambda,
\qquad
D_i\le\widehat D_i+e^u_{D,i}\le\delta,
\]
so every validated modification satisfies the target.  Conversely, if $C_i\in\mathcal P^+_{k,N,\mathbf n,\beta}(T)$, then
\[
\widehat L_i-e^\ell_{L,i}
\ge L_i-e^u_{L,i}-e^\ell_{L,i}
\ge L_i-\bar w_L
\ge\lambda,
\]
and
\[
\widehat D_i+e^u_{D,i}
\le D_i+e^\ell_{D,i}+e^u_{D,i}
\le D_i+\bar w_D
\le\delta,
\]
so every candidate-pool member in $\mathcal P^+_{k,N,\mathbf n,\beta}(T)$ is validated.

The bounds $\bar w_L,\bar w_D$ and hence $\mathcal P^+_{k,N,\mathbf n,\beta}(T)$ are fixed before the pool draw. Let $I$ be the event that the candidate pool contains at least one modification in this set.  Conditional i.i.d.\ sampling gives
\[
\mathbb P(I\mid\mathcal H_k,\xi)=1-\bigl(1-P_k^+(T\mid\mathcal H_k,\xi)\bigr)^N.
\]
On $I\cap\mathcal E$ the validated set is nonempty, every validated modification satisfies the target, and the selector therefore chooses a modification satisfying the target.  Because Eq.~\eqref{eq:formula-4-8} holds conditional on the candidate pool,
\[
\mathbb P(I\cap\mathcal E\mid\mathcal H_k,\xi)
=\mathbb E\!\left[\mathbf1_I\mathbb P(\mathcal E\mid\mathcal H_k,\xi,C_{1:N})\,\middle|\,\mathcal H_k,\xi\right]
\ge(1-\beta_{\mathrm{stat}})\mathbb P(I\mid\mathcal H_k,\xi).
\]
This proves the lower bound.  The upper bound is Theorem~\ref{thm:candidate-pool-selection}: the event counted by $R^{G,\sigma}$ requires a modification satisfying the target to be present in the candidate pool. \end{proof}

Conditionally independent evaluation bundles instantiate Eq.~\eqref{eq:formula-4-8} with a single pool-wise allocation and deterministic upper bounds on interval widths of order $\sqrt{\log(N/\beta_{\mathrm{stat}})/n_{\mathrm{eff}}}$. This includes records from an independent post-reset i.i.d.\ task generator, or from an $m$-dependent stream sampled beyond its dependence range and separated from the past. A single reset before evaluation removes the initial history dependence but does not by itself make later records independent.

\begin{corollary}
\label{cor:mixing-certified-selection}
Suppose the reference experiment satisfies Theorem~\ref{thm:certified-selection}. Let $\mathsf S_T$ be the event that the selector chooses a candidate-pool modification satisfying $T$. If, after fixing $(\mathcal H_k,\xi)$, the original and reference experiments obey the conditional bound
\[
\operatorname*{ess\,sup}_{\mathcal H_k,\xi}
\|\mathbb P(\cdot\mid\mathcal H_k,\xi)-\widetilde{\mathbb P}(\cdot\mid\mathcal H_k,\xi)\|_{\mathrm{TV}}
\le\beta_{\mathrm{mix}},
\]
then the conditional conclusion is
\begin{equation}
R^{G,\sigma}_{k,N}(T\mid\mathcal H_k,\xi)
\ge
\left[(1-\beta_{\mathrm{stat}})\{1-(1-P^+_k(T\mid\mathcal H_k,\xi))^N\}-\beta_{\mathrm{mix}}\right]_+.
\label{eq:formula-4-10}
\end{equation}
If instead the joint laws of history, context, candidate pool, all evaluation records and marks, and selector randomness have the same history/context marginal and are within total variation $\beta_{\mathrm{mix}}$, then the marginal bound is
\begin{equation}
\mathbb P(\mathsf S_T)\ge
\left[(1-\beta_{\mathrm{stat}})
\mathbb E_{\mathcal H_k,\Xi_k}
\bigl[\mathsf C^+_{k,N}(T\mid\mathcal H_k,\Xi_k)\bigr]
-\beta_{\mathrm{mix}}\right]_+.
\label{eq:formula-4-10a}
\end{equation}
\end{corollary}

For an ordered stationary $\beta$-mixing evaluation record, let $g_0$ be the gap from the pre-evaluation information and $g_1,\ldots,g_{q-1}$ the successive gaps. Under the information and marking conditions of Remark~\ref{rem:common-offset}, sequential comparison with the product reference law gives
\begin{equation}
\beta_{\mathrm{mix}}
\le \beta(g_0)+\sum_{j=1}^{q-1}\beta(g_j).
\label{eq:formula-4-10b}
\end{equation}
Thus $q$ records separated from the past and one another by at least $\ell$ give $\beta_{\mathrm{mix}}\le q\beta(\ell)$. A reset independent of the pre-evaluation information removes $\beta(g_0)$ only; the remaining record-to-record terms still apply unless the post-reset generator supplies independent records. The count $q$ includes all evaluation records and may depend on the candidate count.

\begin{proof}[Proof of Corollary~\ref{cor:mixing-certified-selection}.] Apply Theorem~\ref{thm:certified-selection} under the reference law. For fixed $(\mathcal H_k,\xi)$, conditional total-variation control transfers the selection event with loss at most $\beta_{\mathrm{mix}}$, so the reference lower bound loses this term and gives Eq.~\eqref{eq:formula-4-10}. Under joint total-variation control, first average the reference conditional lower bound over the common history/context marginal; transferring the resulting joint event again costs at most $\beta_{\mathrm{mix}}$, giving Eq.~\eqref{eq:formula-4-10a}. For ordinary $\beta$-mixing, compare the joint record law with the product reference law successively at the $q$ gaps. At each step, the common marking and evaluation kernels are Markov kernels and therefore do not increase total variation. The triangle inequality then gives the sum in Eq.~\eqref{eq:formula-4-10b}. \end{proof}

\begin{corollary}
\label{cor:pilot-certified-selection}
Let a failure-rate estimate completed before the candidate pool and separated from generation produce $W=(\underline Z_k,\overline Z_k)$. Set $U=(\mathcal H_k,\Xi_k,W)$. Conditional on $U$, the candidate pool has law $\pi_{k,\Xi_k}^{\otimes N}$. Evaluation-width bounds may depend on $W$, are fixed before the pool, and dominate every realized width. Let $\mathcal E_Z:=\{\underline Z_k\le Z_k\le\overline Z_k\}$ satisfy the marginal bound $\mathbb P(\mathcal E_Z)\ge1-\beta_Z$, and let $\bar w_Z$ be a deterministic pre-pool bound on $\overline Z_k-\underline Z_k$ on every path. For required gain $\gamma$ and retained-task bound $\delta$, use
\begin{equation}
\lambda_k^{\mathrm{obs}}(\gamma,\delta):=\gamma+(1-\underline Z_k)\delta.
\label{eq:formula-4-11}
\end{equation}
On $\mathcal E_Z\cap\mathcal E_{N,\beta}$, validation at $(\lambda_k^{\mathrm{obs}},\delta)$ implies an expected-reward gain of at least $\gamma$. Relative to the unknown $\lambda_k(\gamma,\delta)$, a sufficient set is
\[
\{c:L_k(c)\ge\lambda_k(\gamma,\delta)+\bar w_Z\delta+\bar w_L,\ D_k(c)\le\delta-\bar w_D\}.
\]
Let $\mathsf C_{N,\mathrm{svc}}^+(U)$ denote the probability that the candidate pool contains a modification in this displayed sufficient set, conditional on $U$. For dependent evaluation, let $\mathbb P$ be the actual law and $\widetilde{\mathbb P}$ a reference law of the complete joint object consisting of $U$, the candidate pool, all evaluation records and marks, and selector randomness. Require the two laws to have the same $U$ marginal and satisfy $\|\mathbb P-\widetilde{\mathbb P}\|_{\mathrm{TV}}\le\beta_{\mathrm{mix}}$. Under $\widetilde{\mathbb P}$, conditional on $U$ the pool has law $\pi_{k,\Xi_k}^{\otimes N}$, and the event $\mathcal E_{N,\beta}$ of Theorem~\ref{thm:certified-selection}, on which all intervals contain their corresponding true values, has conditional probability at least $1-\beta_{\mathrm{stat}}$ given $U$ and the pool. The exact conditional implementation takes $\widetilde{\mathbb P}=\mathbb P$ and $\beta_{\mathrm{mix}}=0$. Let $\mathsf S_\gamma$ be the event that the selected candidate-pool modification has expected-reward gain at least $\gamma$. Then the marginal guarantee is
\begin{equation}
\mathbb P(\mathsf S_\gamma)\ge\left[(1-\beta_{\mathrm{stat}})
\mathbb E\bigl[\mathsf C_{N,\mathrm{svc}}^+(U)\bigr]
-\beta_Z-\beta_{\mathrm{mix}}\right]_+,
\label{eq:formula-4-11a}
\end{equation}
where the expectation is over the common marginal of $U$.
\end{corollary}

For an ordinary $\beta$-mixing construction of this joint comparison, $W$ belongs to the pre-evaluation information, the initial gap in Eq.~\eqref{eq:formula-4-10b} starts after that information, and the candidate pool retains its stated conditional law given $U$. The joint comparison gives the marginal guarantee in Eq.~\eqref{eq:formula-4-11a}; a conditional guarantee under the actual law requires conditional comparison. Pilot uncertainty contributes both the margin $\bar w_Z\delta$ and the marginal failure allowance $\beta_Z$. If domination by $\bar w_Z$ is itself only a $1-\beta_{\mathrm{env}}$ event, intersect that event and subtract $\beta_{\mathrm{env}}$ once.

\begin{proof}[Proof of Corollary~\ref{cor:pilot-certified-selection}.] On $\mathcal E_Z\cap\mathcal E_{N,\beta}$, any modification validated at $\lambda_k^{\mathrm{obs}}$ satisfies
\[
J_{\mathrm{user}}(A_k\oplus c)-J_{\mathrm{user}}(A_k)
\ge L_k(c)-(1-Z_k)D_k(c)
\ge\gamma+(Z_k-\underline Z_k)\delta\ge\gamma.
\]
Moreover $\lambda_k^{\mathrm{obs}}-\lambda_k(\gamma,\delta)=(Z_k-\underline Z_k)\delta\le\bar w_Z\delta$. Adding the pre-pool interval-width bounds gives the displayed sufficient set. Under $\widetilde{\mathbb P}$, conditional on each $U$ in $\mathcal E_Z$, the sufficient set is fixed before the pool draw, so Theorem~\ref{thm:certified-selection} and the conditional probability bound on $\mathcal E_{N,\beta}$ give
\[
\widetilde{\mathbb P}(\mathsf S_\gamma\cap\mathcal E_Z)
\ge(1-\beta_{\mathrm{stat}})
\widetilde{\mathbb E}[\mathsf C_{N,\mathrm{svc}}^+(U)\mathbf1_{\mathcal E_Z}]
\ge(1-\beta_{\mathrm{stat}})
\bigl(\mathbb E_{\mathbb P}\mathsf C_{N,\mathrm{svc}}^+(U)-\beta_Z\bigr),
\]
because $0\le\mathsf C_{N,\mathrm{svc}}^+\le1$, $\mathbb P(\mathcal E_Z^c)\le\beta_Z$, and the two laws have the same $U$ marginal. Transfer the single joint event $\mathsf S_\gamma\cap\mathcal E_Z$ once:
\[
\mathbb P(\mathsf S_\gamma)\ge\mathbb P(\mathsf S_\gamma\cap\mathcal E_Z)
\ge\widetilde{\mathbb P}(\mathsf S_\gamma\cap\mathcal E_Z)-\beta_{\mathrm{mix}}.
\]
This gives Eq.~\eqref{eq:formula-4-11a}, whose displayed subtraction of $\beta_Z$ is looser than $(1-\beta_{\mathrm{stat}})\beta_Z$. A separate event for an uncertain width envelope is included in the same reference event and allocated its marginal failure probability once. \end{proof}

\begin{remark}
\label{rem:candidate-count}
For a fixed set, increasing $N$ raises the probability that the pool contains a member. Maintaining the confidence guarantee across more candidates also widens the intervals through $\log N$, potentially shrinking $\mathcal P^+$. Hence the adoption-probability lower bound need not be monotone in $N$.
\end{remark}

For context averaging, write
\begin{equation}
\mathsf C_N(T\mid\mathcal H_k,\xi):=1-\bigl(1-P_k(T\mid\mathcal H_k,\xi)\bigr)^N,
\qquad
\mathsf C^+_N(T\mid\mathcal H_k,\xi):=1-\bigl(1-P^+_k(T\mid\mathcal H_k,\xi)\bigr)^N.
\label{eq:formula-4-12}
\end{equation}
Theorem~\ref{thm:certified-selection} gives $(1-\beta_{\mathrm{stat}})\mathsf C_N^+\le R^{G,\sigma}_{k,N}\le\mathsf C_N$ under its conditional-coverage assumptions. Because all generation runs in a step share the context, define
\[
\overline{\mathsf C}_{k,N}(T):=
\mathbb E_{\Xi_k}\!\left[\mathsf C_N(T\mid\mathcal H_k,\Xi_k)\mid\mathcal H_k\right],
\]
and $\bar P_k(T):=\mathbb E_{\Xi_k}[P_k(T\mid\mathcal H_k,\Xi_k)\mid\mathcal H_k]$.

\begin{corollary}
\label{cor:context-averaging}
For every $N\ge1$, conditional on $\mathcal H_k$,
\[
\overline{\mathsf C}_{k,N}(T)\le1-\bigl(1-\bar P_k(T)\bigr)^N.
\]
\end{corollary}

Substituting context-averaged reachability into $1-(1-p)^N$ gives an upper bound on the probability that the pool contains a qualified modification; equality need not hold.

\begin{proof}[Proof of Corollary~\ref{cor:context-averaging}.] The candidate-pool event in Theorem~\ref{thm:candidate-pool-selection}, averaged over the shared context $\Xi_k$, has probability $\overline{\mathsf C}_{k,N}(T)$. Since $f(p)=1-(1-p)^N$ is concave for $N\ge1$, Jensen gives
\[
\mathbb E[f(P_k(T\mid\Xi_k))]\le f(\mathbb E[P_k(T\mid\Xi_k)]),
\]
which is the stated upper bound. \end{proof}

If the width bounds tend to zero and $\pi_{k,\xi}\{c:L_k(c)=\lambda\ \text{or}\ D_k(c)=\delta\}=0$, then $P_k^+(T)\to P_k(T)$ by bounded convergence of the target indicators.

\begin{proposition}
\label{prop:expected-selected-improvement}
Under Theorem~\ref{thm:certified-selection}, suppose \(g_T:=\lambda-(1-Z_k)\delta>0\). Then
\begin{equation}
\mathbb E[J_{\mathrm{user}}(A_{k+1})-J_{\mathrm{user}}(A_k)\mid\mathcal H_k,\xi]
\ge
(1-\beta_{\mathrm{stat}})\mathsf C_N^+(T\mid\mathcal H_k,\xi)\,g_T
-\beta_{\mathrm{stat}}.
\label{eq:formula-4-13}
\end{equation}
\end{proposition}

Eq.~\eqref{eq:formula-4-13} bounds expected improvement by combining the gain guaranteed when the pool contains a member of $\mathcal P^+$ and the evaluation intervals contain the true values, with the possible loss when these intervals fail to do so.

\begin{proof}[Proof of Proposition~\ref{prop:expected-selected-improvement}.] Let \(I\) be the event that the candidate pool contains a modification in $\mathcal P^+_{k,N,\mathbf n,\beta}(T)$ and write $\mathcal E=\mathcal E_{N,\beta}$.  On \(I\cap\mathcal E\), Theorem~\ref{thm:certified-selection} selects a modification satisfying the exact target, whose expected-reward gain is at least \(g_T\) by the lower bound Eq.~\eqref{eq:formula-3-18}.  On \(\mathcal E\setminus I\), the gain is nonnegative by the validation guarantee and Definition~\ref{def:evolution-step}.  On \(\mathcal E^c\), bounded rewards give gain at least \(-1\).  Finally \(\mathbb P(I\cap\mathcal E)\ge(1-\beta_{\mathrm{stat}})\mathsf C_N^+\) and \(\mathbb P(\mathcal E^c)\le\beta_{\mathrm{stat}}\), proving Eq.~\eqref{eq:formula-4-13}. \end{proof}

\paragraph{Pool-wise confidence allocation.}

Using one event $\mathcal E_{N,\beta}$ for all intervals in the realized candidate pool places the effect of statistical error in the interval widths and hence in $P_k^+$. The finite-data bounds for adopting qualified modifications in Theorem~\ref{thm:certified-selection} therefore count familywise error once while preserving the target.

\begin{corollary}
\label{cor:pool-confidence-allocation}
In the setting of Theorem~\ref{thm:certified-selection}, fix \(N\), \(\beta_{\mathrm{stat}}\), \(c>0\), and \(n_{\mathrm{eff}}>0\) before drawing the candidate pool. Suppose that, conditional on the pre-pool variables and the pool, each of the four endpoint failures for every valid output admits the common bound \(2\exp(-n_{\mathrm{eff}}e^2/c)\). Allocating \(\beta_{\mathrm{stat}}/(4N)\) to each endpoint yields Eq.~\eqref{eq:formula-4-8} by choosing the common endpoint radius
\begin{equation}
r_N
\;:=\;
\sqrt{\frac{c}{n_{\mathrm{eff}}}
      \log\frac{8N}{\beta_{\mathrm{stat}}}}
\;=\;
O\!\left(\sqrt{\frac{\log(N/\beta_{\mathrm{stat}})}
                         {n_{\mathrm{eff}}}}\right).
\label{eq:formula-G-1}
\end{equation}
For valid outputs, the corresponding symmetric intervals have total width \(2r_N\); failed outputs retain zero radii. Thus one may take the deterministic width bounds \(\bar w_L=\bar w_D=2r_N\).
\end{corollary}

\begin{proof}[Proof of Corollary~\ref{cor:pool-confidence-allocation}.] At radius $r_N$, the endpoint tail bound gives failure probability at most
\[
2\exp(-n_{\mathrm{eff}}r_N^2/c)=\frac{\beta_{\mathrm{stat}}}{4N}
\]
for each endpoint. A union bound over at most $4N$ endpoint events therefore gives total failure probability at most $\beta_{\mathrm{stat}}$. Substituting Eq.~\eqref{eq:formula-G-1} into this union bound yields Eq.~\eqref{eq:formula-4-8}. \end{proof}

For a fixed interval width, Eq.~\eqref{eq:formula-G-1} requires only logarithmic growth in effective sample size \emph{per candidate} as $N$ grows; total evaluation cost also scales with the number of candidates. At a fixed total budget, fewer samples per candidate and the familywise $\log N$ term can shrink $P_k^+$, so the adoption-probability lower bound need not increase with $N$. Corollary~\ref{cor:mixing-certified-selection} gives the dependent-evaluation transfer for this pool event.

\subsection{Paired evaluation}
\label{app:section-D}

This appendix analyzes paired evaluation and fixed-budget empirical-Bernstein bounds as options for evaluation efficiency. The guarantees in the main text use the radii specified with their protocols.

\paragraph{Paired sampling.}

\begin{definition}[Paired rollout]
\label{def:paired-rollout}
For a task $t$, a paired rollout is a pair $(o^{A_k},o^{\tilde A})$ drawn from a coupling $\Gamma_t\in\Delta(\mathcal O\times\mathcal O)$ with the following properties.

\textbf{(i) Marginal correctness}: Its marginals are $A_k(\cdot\mid t)$ and $\tilde A(\cdot\mid t)$.

\textbf{(ii) Coupling}: Within a pair, the external randomness source is shared (task instance state, external API responses, environment randomness); LLM decoding is independent.

An evaluation with $m$ paired rollouts consists of independent repetitions $(o_l^{A_k},o_l^{\tilde A})\overset{\mathrm{i.i.d.}}{\sim}\Gamma_t$, $l=1,\ldots,m$.
\end{definition}

The paired estimator of the advantage is
\begin{equation}
\hat\delta^{\mathrm{paired}}(t) := \frac{1}{m}\sum_{l=1}^m \bigl[r(t, o_l^{\tilde A}) - r(t, o_l^{A_k})\bigr].
\label{eq:formula-D-1}
\end{equation}

Positive within-pair reward covariance reduces variance relative to independent rollouts. Shared external state must give the marginal kernels in Definition~\ref{def:paired-rollout}.

\paragraph{Variance.}

\begin{proposition}
\label{prop:paired-variance}
Let $R_A=r(t,o^{A_k})$ and $R_B=r(t,o^{\tilde A})$ under the coupling $\Gamma_t$, with variances $\sigma_{A_k,t}^2$ and $\sigma_{\tilde A,t}^2$. Then
\begin{equation}
\operatorname{Var}[\hat\delta^{\mathrm{paired}}(t)]
=\frac{\sigma_{A_k,t}^2+\sigma_{\tilde A,t}^2-2\operatorname{Cov}_{\Gamma_t}(R_A,R_B)}{m}.
\label{eq:formula-D-2}
\end{equation}
If both variances are positive, write $\rho_t=\operatorname{Corr}_{\Gamma_t}(R_A,R_B)$, so the covariance is $\rho_t\sigma_{A_k,t}\sigma_{\tilde A,t}$.
\end{proposition}

When both variances are positive, the variance ratio relative to independent sampling is $1 - 2\rho_t \sigma_{A_k, t}\sigma_{\tilde A, t}/(\sigma_{A_k, t}^2 + \sigma_{\tilde A, t}^2)$.

\begin{proof}[Proof of Proposition~\ref{prop:paired-variance}.] The independent pairs give
\[
\operatorname{Var}[\hat\delta^{\mathrm{paired}}(t)]
=\frac1m\operatorname{Var}(R_B-R_A)
=\frac{\sigma_{A_k,t}^2+\sigma_{\tilde A,t}^2
-2\operatorname{Cov}_{\Gamma_t}(R_A,R_B)}{m},
\]
which is Eq.~\eqref{eq:formula-D-2}. Dividing by the variance under independent rollouts gives the stated ratio. \end{proof}

\paragraph{Empirical-Bernstein bounds.}

Let $m\ge2$ and let $X_1,\ldots,X_m$ be independent random variables in $[a,a+R]$. Write $\bar X=m^{-1}\sum_lX_l$ and $\hat\sigma^2=(m-1)^{-1}\sum_l(X_l-\bar X)^2$. Applying the independent-variable empirical-Bernstein bound of \citet{maurer2009empirical} to both directions, with error probability $\beta/2$ each, gives
\begin{equation}
\mathbb{P}\biggl[|\bar X - \mathbb{E}\bar X| > \sqrt{\frac{2\hat\sigma^2 \log(4/\beta)}{m}} + \frac{7R\log(4/\beta)}{3(m-1)}\biggr] \leq \beta.
\label{eq:formula-D-3}
\end{equation}

Applied to the paired advantage estimator $\hat\delta^{\mathrm{paired}}(t)$ (each term $\in [-1, 1]$, $R = 2$):

\begin{lemma}
\label{lem:paired-bernstein}
For $m\ge2$, let $\delta(t)=\operatorname{Adv}_{A_k}(\tilde A,t)$ and let $\hat\sigma_{\delta,t}^2$ be the sample variance with denominator $m-1$ of $\{r(t, o_l^{\tilde A}) - r(t, o_l^{A_k})\}_{l=1}^m$. Then with probability $\geq 1 - \beta_\delta^{(1)}$:
\begin{equation}
|\hat\delta^{\mathrm{paired}}(t) - \delta(t)| \leq \varepsilon_\delta(t) := \sqrt{\frac{2\hat\sigma_{\delta, t}^2 \log(4/\beta_\delta^{(1)})}{m}} + \frac{14\log(4/\beta_\delta^{(1)})}{3(m - 1)}.
\label{eq:formula-D-4}
\end{equation}
\end{lemma}

The empirical variance in Eq.~\eqref{eq:formula-D-4} gives task-specific radii: at small sample variance the additive $O(\log(1/\beta)/m)$ term dominates, while larger variance contributes a square-root term.

\subsection{Repair and finite assembly}
\label{app:section-G-6}

\begin{proposition}
\label{prop:regeneration-probability}
Fix $(\mathcal H_k,\Xi_k=\xi)$ and target $T$. Let $g$ be a Markov kernel from $\mathcal M_\bot$ to $\mathcal M_\bot$, fixed before the candidate pool is drawn. Draw $C_{1:N}\stackrel{\mathrm{iid}}{\sim}\pi_{k,\xi}$ and, conditionally independently given the pool, draw $Y_i\sim g(\cdot\mid C_i)$. Then $Y_{1:N}$ are i.i.d.\ from the induced distribution $\pi_{k,\xi}g$. Writing
\begin{equation}
P_k^g(T\mid\mathcal H_k,\xi):=(\pi_{k,\xi}g)(\mathcal P_k(T)),
\qquad
R^g_{k,N}(T\mid\mathcal H_k,\xi)
\le 1-\bigl(1-P_k^g(T\mid\mathcal H_k,\xi)\bigr)^N,
\label{eq:formula-3-27}
\end{equation}
the bound holds for every rule that selects among $Y_{1:N}$ or retains the current agent, with equality for an oracle restricted to selecting from the repaired candidate pool.  The kernel may contain finite predeclared tests, repair steps, and a stopping rule, but each invocation must return one terminal output and different invocations must be conditionally independent.
\end{proposition}

Each invocation of $g$ receives only its own drawn candidate and uses independent local randomness. The comparison does not impose an ordering between $P_k^g$ and $P_k$.

\begin{proof}[Proof of Proposition~\ref{prop:regeneration-probability}.] For every measurable $B\subseteq\mathcal M_\bot$,
\[
\mathbb P(Y_i\in B\mid\mathcal H_k,\Xi_k=\xi)
=\int_{\mathcal M_\bot}g(B\mid c)\,\pi_{k,\xi}(dc)
=(\pi_{k,\xi}g)(B).
\]
The conditional independence of the kernel calls and the i.i.d.\ base draws make $Y_{1:N}$ i.i.d.\ from $\pi_{k,\xi}g$. A selector succeeds only if the repaired pool contains some $Y_i\in\mathcal P_k(T)$; this event has probability $1-(1-P_k^g)^N$. An oracle restricted to that pool attains it, proving Eq.~\eqref{eq:formula-3-27}. \end{proof}

\paragraph{Assembly.}

\begin{definition}[Assembly family]
\label{def:assembly-family}
Fix $(\mathcal H_k,\Xi_k=\xi)$ and target $T$, and abbreviate $\pi=\pi_{k,\xi}$ and $\mathcal P=\mathcal P_k(T)$. Let $A<\infty$ and, for each $a\in\{1,\ldots,A\}$, let $\mathcal G_a$ be a finite, possibly empty family of Markov kernels from $\mathcal M_\bot^a$ to $\mathcal M_\bot$, fixed before the candidate pool. The active arities are $\mathcal A_{\mathcal G}:=\{a\in\{1,\ldots,A\}:\mathcal G_a\ne\varnothing\}$.
Require $1\in\mathcal A_{\mathcal G}$ and $\mathrm{id}\in\mathcal G_1$.

Draw $C_{1:N}\stackrel{\mathrm{iid}}{\sim}\pi$. For each $a\in\mathcal A_{\mathcal G}$, each ordered tuple $I=(i_1,\ldots,i_a)$ of distinct indices, and each $g\in\mathcal G_a$, instantiate one terminal output $Y_{g,I}\sim g(\cdot\mid C_{i_1},\ldots,C_{i_a})$, using invocation-local randomness that is independent across $(g,I)$ conditionally on the base pool.

An assembly-and-selection rule may select a valid modification among these outputs and the identity outputs, or retain the current agent, but may not redraw from $\pi$ or alter a kernel after observing the candidate pool. Each invocation receives only its designated tuple from the base pool. A kernel may contain finite predeclared tuple-local tests, feedback, repair steps, and a stopping rule, provided it returns one terminal output in $\mathcal M_\bot$ for every input tuple. Internal drafts and feedback remain within that invocation; $\bot$ is not required to be absorbing under repair.
\end{definition}

For each $a\in\mathcal A_{\mathcal G}$, the fixed kernel family and base-pool law define
\[
Q_a:=\max_{g\in\mathcal G_a}(\pi^{\otimes a}g)(\mathcal P).
\]
For deterministic $g$ this is $\pi^{\otimes a}(g^{-1}(\mathcal P))$; when $\mathcal G_1=\{\mathrm{id}\}$, $Q_1=P_k(T\mid\mathcal H_k,\xi)$.

\begin{theorem}
\label{thm:finite-assembly-bound}
Let $R_N^{\mathcal G}(\mathcal P)$ denote the conditional probability that an assembly-and-selection rule of Definition~\ref{def:assembly-family} selects an assembled output derived from the candidate pool that belongs to $\mathcal P$.  Then
\begin{equation}
\begin{aligned}
R_N^{\mathcal G}(\mathcal P)
&\le \min\!\left\{1,
|\mathcal G_1|\bigl[1-(1-Q_1)^N\bigr]
+\!\!\sum_{\substack{a\in\mathcal A_{\mathcal G}\\2\le a\le N}}
(N)_a|\mathcal G_a|Q_a\right\},\\
(N)_a&:=\frac{N!}{(N-a)!}.
\end{aligned}
\label{eq:formula-G-15}
\end{equation}
When $\mathcal G_1=\{\mathrm{id}\}$ and $A=1$, this is Theorem~\ref{thm:candidate-pool-selection}.
\end{theorem}

\begin{remark}
\label{rem:assembly-bound-scope}
The bound can reach one when many tuples or kernels are instantiated. Estimating terminal-output probabilities $Q_a$ and exploiting dependence among assembled outputs are directions discussed in Appendix~\ref{app:scope-open-problems}.
\end{remark}

\begin{proof}[Proof of Theorem~\ref{thm:finite-assembly-bound}.] Fix $g\in\mathcal G_a$ and an ordered tuple $I$ of distinct indices.  By Definition~\ref{def:assembly-family},
\[
\mathbb P(Y_{g,I}\in\mathcal P\mid\mathcal H_k,\Xi_k=\xi)
=(\pi^{\otimes a}g)(\mathcal P)\le Q_a.
\]
For a fixed $g\in\mathcal G_1$, distinct base inputs are i.i.d.\ and the invocation-local randomness is independent, so the $N$ outputs $Y_{g,(i)}$ are i.i.d.\ from $\pi g$. The probability that at least one of these outputs belongs to $\mathcal P$ is at most $1-(1-Q_1)^N$. A union bound over the finite family $\mathcal G_1$ gives the first term of Eq.~\eqref{eq:formula-G-15}. For each active $a\ge2$ with $a\le N$, there are $(N)_a$ ordered tuples of distinct indices and $|\mathcal G_a|$ kernels. A union bound over these output events gives $(N)_a|\mathcal G_a|Q_a$ without requiring independence across overlapping tuples. Selecting an assembled output from the candidate pool that belongs to $\mathcal P$ requires at least one of the counted events. Summing over active arities and truncating at one proves Eq.~\eqref{eq:formula-G-15}. \end{proof}

\subsection{Endpoint tails at a fixed retained-task bound}
\label{app:reachability-tail}

\begin{proposition}
\label{prop:tail-description}
Fix the history, context, and retained-task bound $\delta$, and draw $C\sim\pi_{k,\xi}$. Write $L=L_k(C)$, $D=D_k(C)$,
$\bar G_{k,\delta}(x):=\mathbb P(L\ge x,D\le\delta)$, and let $s_{k,\delta}$ be the essential upper endpoint of $L$ on $\{D\le\delta\}$. Suppose $f(u):=\bar G_{k,\delta}(s_{k,\delta}-u)$ is regularly varying at zero with index $\vartheta_{k,\delta}>0$. Then
$f(u)=u^{\vartheta_{k,\delta}}m_{k,\delta}(u)=u^{\vartheta_{k,\delta}+o(1)}$
for a function $m_{k,\delta}$ slowly varying at zero. For positive targets, if $s_{k,\delta}>0$, then as $0<\lambda\uparrow s_{k,\delta}$,
\begin{equation}
P_k((\lambda,\delta))
=(s_{k,\delta}-\lambda)^{\vartheta_{k,\delta}}
m_{k,\delta}(s_{k,\delta}-\lambda)
=(s_{k,\delta}-\lambda)^{\vartheta_{k,\delta}+o(1)}.
\label{eq:formula-3-31}
\end{equation}
\end{proposition}
\begin{proof}[Proof of Proposition~\ref{prop:tail-description}.]
By regular variation, $m_{k,\delta}(u)=f(u)/u^{\vartheta_{k,\delta}}$ is slowly varying. For slowly varying functions, \citet{dehaan2006evt} give $\log m_{k,\delta}(u)/\log u\to0$, yielding the second representation. For $0<\lambda<s_{k,\delta}$, failed outputs have $L_k(\bot)=0<\lambda$, so $P_k((\lambda,\delta))=\bar G_{k,\delta}(\lambda)=f(s_{k,\delta}-\lambda)$. This proves Eq.~\eqref{eq:formula-3-31}. \end{proof}

The endpoint obeys $s_{k,\delta}\le Z_k\bar a_k\le1-J_{\mathrm{user}}(A_k)$ by Eq.~\eqref{eq:formula-3-30}. It can be below the pointwise bound when generation assigns no mass near the best feasible modifications. The endpoint and tail index belong to the specified state, history, context, and retained-task bound.

\section{Performance Ceiling and Verification}\label{app:limits}
\subsection{Remaining improvement and positive reachability}\label{app:positive-reachability}
\begin{lemma}
\label{lem:failure-improvement-bound}
Write
\begin{equation}
\bar a_k \;:=\; 1-\mathbb{E}_{t\sim\mathcal{D}_{F,k}}\bigl[V(A_k,t)\bigr]\;\in\;[0,1].
\label{eq:formula-3-28}
\end{equation}
Under Assumptions~\ref{ass:bounded-reward} and \ref{ass:agent-independent-tasks}, (i) the following decomposition and bound hold:
\begin{equation}
1-J_{\text{user}}(A_k) \;=\; Z_k\,\bar a_k \;+\; (1-Z_k)\bigl(1-\mathbb{E}_{\mathcal{D}_{R,k}}[V(A_k,\cdot)]\bigr) \;\ge\; Z_k\,\bar a_k .
\label{eq:formula-3-29}
\end{equation}

(ii) For every candidate agent $\tilde A$, writing $\mu_F=\mathbb E_{\mathcal D_{F,k}}[\mathrm{Adv}_{A_k}(\tilde A,t)]$,
\begin{equation}
L_{A_k}(\tilde A)\;=\;Z_k\mu_F\;\le\;Z_k\,\bar a_k\;\le\;1-J_{\text{user}}(A_k).
\label{eq:formula-3-30}
\end{equation}
\end{lemma}

\begin{proof}[Proof of Lemma~\ref{lem:failure-improvement-bound}.] Apply Lemma~\ref{lem:failure-decomposition} to $g=1-V(A_k,\cdot)$ and use Eq.~\eqref{eq:formula-2-8} and Eq.~\eqref{eq:formula-3-28} to obtain Eq.~\eqref{eq:formula-3-29}. Assumption~\ref{ass:bounded-reward} gives $\mathrm{Adv}_{A_k}(\tilde A,t)\le1-V(A_k,t)$ pointwise. Taking the expectation under $\mathcal D_{F,k}$ yields $\mu_F\le\bar a_k$, and hence Eq.~\eqref{eq:formula-3-30}. \end{proof}

\begin{corollary}[Necessary condition for positive reachability]
\label{cor:positive-reachability}
Under Assumptions~\ref{ass:bounded-reward} and~\ref{ass:agent-independent-tasks}, if $P_k((\lambda,\delta)\mid\mathcal H_k,\xi)>0$, then
\[
Z_k\bar a_k\ge\lambda,
\qquad
J_{\mathrm{user}}(A_k)\le1-\lambda.
\]
\end{corollary}

For the calibrated target in Corollary~\ref{cor:qualified-safe-improvement}, the second inequality implies $J_{\mathrm{user}}(A_k)\le1-\gamma-(1-Z_k)\delta$.

\begin{proof}[Proof of Corollary~\ref{cor:positive-reachability}.] Positive reachability implies that $\mathcal P_k((\lambda,\delta))$ is nonempty, so some modification has $L_k(c)\ge\lambda$. Lemma~\ref{lem:failure-improvement-bound}, specifically Eq.~\eqref{eq:formula-3-30}, gives $\lambda\le L_k(c)\le Z_k\bar a_k\le1-J_{\mathrm{user}}(A_k)$. Substituting $\lambda=\gamma+(1-Z_k)\delta$ gives the statement for required overall gain $\gamma$. \end{proof}

\subsection{One-sided changes and Measured-margin}\label{app:one-sided-deviations}\label{app:one-sided-estimation}
\begin{definition}[One-sided deviations]
\label{def:one-sided-deviations}
For a candidate agent $\widetilde A$, define the average loss and gain in per-task expected reward under $\mathcal D_{R,k}$ by
\[
\begin{aligned}
D_R^-(\widetilde A;A_k)
&:=\mathbb E_{t\sim\mathcal D_{R,k}}[(V(A_k,t)-V(\widetilde A,t))_+],\\
D_R^+(\widetilde A;A_k)
&:=\mathbb E_{t\sim\mathcal D_{R,k}}[(V(\widetilde A,t)-V(A_k,t))_+].
\end{aligned}
\]

\end{definition}
By $|x|=(x)_++(-x)_+$, $D_R=D_R^-+D_R^+$ and $\mathbb E_{\mathcal D_{R,k}}[\mathrm{Adv}_{A_k}]=D_R^+-D_R^-$.

\begin{proposition}
\label{prop:one-sided-improvement}
 Under Assumptions~\ref{ass:bounded-reward} and \ref{ass:agent-independent-tasks}, if $D_R^-(\widetilde A;A_k)\le\delta^-$, then

\begin{equation}
J_{\mathrm{user}}(\widetilde A)-J_{\mathrm{user}}(A_k)
\ge
L_{A_k}(\widetilde A)-(1-Z_k)\delta^-.
\label{eq:formula-3-17}
\end{equation}

\end{proposition}

\begin{proof}[Proof of Proposition~\ref{prop:one-sided-improvement}.] From Eq.~\eqref{eq:formula-3-9} and $D_R^+ \geq 0$:
\[
J_{\text{user}}(\tilde A) - J_{\text{user}}(A_k) = L_{A_k} + (1 - Z_k)(D_R^+ - D_R^-) \geq L_{A_k} - (1 - Z_k) D_R^-. \quad
\]
Since the proposition assumes $D_R^-\le\delta^-$, the displayed inequality gives
\[
J_{\text{user}}(\tilde A)-J_{\text{user}}(A_k)
\ge L_{A_k}-(1-Z_k)\delta^-.
\]
\end{proof}

Since $D_R=D_R^-+D_R^+$, Eq.~\eqref{eq:formula-3-9} and the definition of $M_{A_k}$ in Section~\ref{sec:limits} give

\begin{equation}
J_{\mathrm{user}}(\widetilde A)-J_{\mathrm{user}}(A_k)
=
M_{A_k}(\widetilde A)
+(1-Z_k)D_R^+(\widetilde A;A_k)
\ge M_{A_k}(\widetilde A).
\label{eq:formula-4-22}
\end{equation}

To estimate the mean retained-task loss, define the estimator and its error bound by
\[
\hat D^-_{R,k}:=\frac{1}{|S_R|}\sum_{j\in S_R}
\bigl(\hat V(A_k,t_j^R)-\hat V(\tilde A,t_j^R)\bigr)_+,
\qquad
\varepsilon_D^-:=\frac{2\varepsilon_V+\varepsilon_R+\varepsilon_Z}{1-\hat Z_k}.
\]
\paragraph{One-sided estimation.}
Use the sampling conditions, readiness event, and fresh-rollout coverage of Appendix~\ref{app:proofs-evaluation}. For the exact loss $q^-(t)=(V(A_k,t)-V(\widetilde A,t))_+$, define the weighted stored-record statistic $(1-\phi_j^R)q^-(t_j^R)\in[0,1]$. Its reference expectation is $(1-Z_k)D_R^-$. Apply the stored-sample concentration argument of Lemma~\ref{lem:retained-change-estimation} to this statistic, obtaining an event $\mathcal E_R^-$ with the same allocated failure probability $\beta_R$. The $1$-Lipschitz property of the positive part bounds the replacement of exact rewards by fresh estimates by $2\varepsilon_V$. On $\mathcal E_R^-\cap\mathcal E_Z$ and the fresh-rollout event,
\[
|(1-\widehat Z_k)\widehat D^-_{R,k}-(1-Z_k)D_R^-|
\le2\varepsilon_V+\varepsilon_R.
\]
Since $D_R^-\le1$, division by $1-\widehat Z_k>0$ on $\mathsf{Ready}_k$ gives
\[
|\widehat D^-_{R,k}-D_R^-|
\le(2\varepsilon_V+\varepsilon_R+\varepsilon_Z)/(1-\widehat Z_k).
\]
The reference-sample comparison and its joint-TV risk transfer are the same as in Lemma~\ref{lem:retained-change-estimation}; only the bounded statistic changes. Measured-margin and the one-sided rule in Eq.~\eqref{eq:one-sided-adoption-rule} use $\mathcal E_R^-$ in place of the symmetric retained event, with the original total risk allocation.

\begin{proposition}
\label{prop:measured-margin-guarantee}
Under the reward, task, and evaluation assumptions stated in Theorem~\ref{thm:one-step-guarantee}, use Rule~\ref{alg:measured-margin} in place of Rule~\ref{alg:two-gate} in the one-candidate update, with allowed retained-task decrease $\delta^-$. Let $\mathsf{Acc}^M_k\subseteq\mathsf{Cand}_k\cap\mathsf{Ready}_k$ be the resulting adoption event and use $\widehat M_k$ from Eq.~\eqref{eq:formula-4-24}. \begin{resultparts}
\item\label{part:measured-margin-safety}
\[
\mathbb P\!\left(\mathsf{Acc}^M_k\cap\{I_k<\widehat M_k\}\right)
\le\beta_{\mathrm{step},k}.
\]
Since acceptance requires $\widehat M_k>0$, the joint event of adoption and nonpositive actual improvement has probability at most $\beta_{\mathrm{step},k}$. \item\label{part:measured-margin-inclusion} Using the same evaluation data and per-task reward estimates, at matched retained-task bounds $\delta^-=\delta$ and matched radii, every ready sample path accepted by Eq.~\eqref{eq:formula-4-23} under Eq.~\eqref{eq:formula-4-28} is accepted by Rule~\ref{alg:measured-margin}; the reverse inclusion does not hold in general.
\end{resultparts}
\end{proposition}

\begin{proof}[Proof of Proposition~\ref{prop:measured-margin-guarantee}.] Let $\mathcal E_k^-$ be the intersection, on $\mathsf{Ready}_k$, of the one-sided retained-task event above with the failure-rate, returned failure-mean, and fresh-rollout events used in Theorem~\ref{thm:one-step-guarantee}. Define $\mathcal G_k^-:=\mathsf{Ready}_k^c\cup(\mathsf{Ready}_k\cap\mathcal E_k^-)$. The same union bound gives $\mathbb P((\mathcal G_k^-)^c)\le\beta_{\mathrm{step},k}$. On $\mathcal G_k^-\cap\mathsf{Acc}^M_k$, the one-sided estimate in Appendix~\ref{app:one-sided-estimation} applies because readiness includes a valid retained denominator, and the margin identity Eq.~\eqref{eq:formula-4-22} gives
\[
I_k\ge L_{A_k}-(1-Z_k)D_R^-
\ge\hat L_k-\varepsilon_L
-(1-\hat Z_k+\varepsilon_Z)(\hat D^-_{R,k}+\varepsilon_D^-)
=\widehat M_k.
\]
Thus $\mathsf{Acc}^M_k\cap\{I_k<\widehat M_k\}\subseteq(\mathcal G_k^-)^c$, proving part (\ref{part:measured-margin-safety}). For part (\ref{part:measured-margin-inclusion}), no coverage event is needed: on the same ready evaluation path accepted by Eq.~\eqref{eq:formula-4-23},
$\hat D^-_{R,k}\le\hat D_R\le\delta=\delta^-$, and matched radii together with Eq.~\eqref{eq:formula-4-28} give $\widehat M_k>0$.  For strictness, use the improving instance of Proposition~\ref{prop:positive-margin-instances} with $u\le(1-Z)\delta$, and the sufficiently precise evaluation design of Proposition~\ref{prop:certification-beyond-threshold}. On their common ready coverage event, Measured-margin accepts. If Two-Gate with its positive-gain check accepted, then
\[
u/2=L\ge\widehat L-\varepsilon_L\ge\tau-\varepsilon_L
>(1-\widehat Z+\varepsilon_Z)(\delta+\varepsilon_D)
\ge(1-Z)\delta\ge u,
\]
a contradiction. The i.i.d. arrival horizons can be chosen so that readiness and coverage hold together with positive probability. This provides a matching evaluation configuration for which the inclusion is strict. \end{proof}
\begin{remark}
\label{rem:declared-gain}
For a declared $\gamma>0$, replace $\widehat M_k>0$ in Rule~\ref{alg:measured-margin} by $\widehat M_k\ge\gamma$, retaining its readiness and retained-task requirements. The proof of Proposition~\ref{prop:measured-margin-guarantee} gives the same joint error bound for accepting a modification whose expected-reward improvement is less than $\gamma$. The finite-run union and telescoping argument of Theorem~\ref{thm:finite-run-guarantee}, under its fixed-distribution and summed-risk conditions, then gives a lower bound $\gamma|\mathcal A^M|$ on that variant's adopted-step set and common good event.

On an accepted path in the common coverage event with finite positive radii, $\widehat M_k<L_k\le1-J_{\mathrm{user}}(A_k)$: the strict inequality comes from the positive retained-side error term in Eq.~\eqref{eq:formula-4-24}. Thus accepting under the positive-threshold variant requires $J_{\mathrm{user}}(A_k)<1-\gamma$ on that event.
\end{remark}
\subsection{Constraints on further improvement}
\label{app:proofs-limits}

\begin{corollary}[Performance bound for the one-sided gain condition]
\label{cor:two-gate-threshold}
Under Assumptions~\ref{ass:bounded-reward} and \ref{ass:agent-independent-tasks}, suppose some candidate agent $\tilde A$ meets the one-sided sufficient condition of Proposition~\ref{prop:one-sided-improvement} at allowed retained-task decrease $\delta^-$ --- that is, $D_R^-(\tilde A;A_k)\le\delta^-$ and $L_{A_k}(\tilde A)>(1-Z_k)\delta^-$. Then
\begin{equation}
J_{\text{user}}(A_k) < 1-(1-Z_k)\delta^-.
\label{eq:formula-5-1}
\end{equation}
\end{corollary}

\begin{proof} By Eq.~\eqref{eq:formula-3-30}, $1-J_{\text{user}}(A_k)\ge L_{A_k}(\tilde A)$. The hypothesis gives $L_{A_k}(\tilde A)>(1-Z_k)\delta^-$. Hence $1-J_{\text{user}}(A_k)>(1-Z_k)\delta^-$, and rearranging yields Eq.~\eqref{eq:formula-5-1}. \end{proof}

Write $B_k:=1-(1-Z_k)\delta^-$ for fixed $\delta^->0$. Corollary~\ref{cor:two-gate-threshold} implies that a current agent with $J_{\mathrm{user}}(A_k)\ge B_k$ has no candidate satisfying $L_{A_k}(\widetilde A)>(1-Z_k)\delta^-$. The following result shows that positive actual margins can nevertheless occur, and bounds them by the remaining reward shortfall.

\begin{proposition}[Improvement beyond the ceiling of the sufficient condition]
\label{prop:positive-margin-instances}
Under Assumptions~\ref{ass:bounded-reward} and~\ref{ass:agent-independent-tasks}, fix \(\delta^->0\) and write \(B_k:=1-(1-Z_k)\delta^-\).
\begin{resultparts}
\item\label{part:positive-margin-instance}
In the kernel setting of Appendix~\ref{app:section-A-1}, allow the composition map \((M,C)\mapsto A\) to be any measurable map consistent with Eqs.~\eqref{eq:formula-2-1}--\eqref{eq:formula-2-5}. For every \(Z\in(0,1)\) and
\[
0<u\le\min\{Z/2,(1-Z)\delta^-\},
\]
there exists a two-task instance with \(Z_k=Z\), \(J_{\mathrm{user}}(A_k)=1-u\ge B_k\), and a modification whose candidate agent satisfies
\[
D_R^-(\widetilde A;A_k)=0,
\qquad
M_{A_k}(\widetilde A)=L_{A_k}(\widetilde A)=u/2>0.
\]
\item\label{part:universal-margin-bound}
For every current agent with \(0<Z_k<1\) and every candidate agent,
\[
M_{A_k}(\widetilde A)
\le L_{A_k}(\widetilde A)
\le Z_k\bar a_k
\le 1-J_{\mathrm{user}}(A_k).
\]
\end{resultparts}
\end{proposition}

\begin{proof}[Proof of Proposition~\ref{prop:positive-margin-instances}.] For part (\ref{part:positive-margin-instance}), construct the following instance. Let $\mathcal T_{\mathrm{user}}=\{t_F,t_R\}$,
$\mathcal D_{\mathrm{user}}=Z\delta_{t_F}+(1-Z)\delta_{t_R}$, and $\mathcal O=\{0,1\}$ with reward $r(t,o)=o$. Set $\phi(t_F,o)=1$ and $\phi(t_R,o)=0$. Then $\mathcal D_{F,k}=\delta_{t_F}$ and $\mathcal D_{R,k}=\delta_{t_R}$. On $t_F$, take
\[
A_k(\cdot\mid t_F)=\operatorname{Bernoulli}(1-u/Z),
\qquad
\widetilde A(\cdot\mid t_F)=\operatorname{Bernoulli}(1-u/(2Z)).
\]
On $t_R$, both agents return $1$ deterministically. Choose a measurable composition map in Appendix~\ref{app:section-A-1} that realizes these two kernels at $C_k$ and $C_k\oplus c$, with the LLM fixed. All kernels are measurable on these finite spaces. Direct calculation gives
\[
J_{\mathrm{user}}(A_k)=1-u,\qquad
Z_k\bar a_k=u,\qquad
D_R^-=D_R^+=0,\qquad
L_{A_k}=M_{A_k}=u/2.
\]
The stipulated range of $u$ gives $J_{\mathrm{user}}(A_k)\ge B_k$.

Part (\ref{part:universal-margin-bound}) follows from $D_R^-\ge0$ and Eq.~\eqref{eq:formula-3-30}. \end{proof}

\begin{lemma}
\label{lem:observed-improvement-bound}
On evaluation paths where the required samples are available, with $\hat L_k$ as defined in Eq.~\eqref{eq:formula-4-16},
\begin{equation}
\hat L_k \;\in\; [-\hat Z_k,\;\hat Z_k] \qquad\text{with probability }1.
\label{eq:formula-5-3}
\end{equation}
Consequently, on any evaluated path with $\tau>\hat Z_k$, no modification can satisfy $\widehat L_k\ge\tau$ in Eq.~\eqref{eq:formula-4-23} or be adopted under Rule~\ref{alg:two-gate}.
\end{lemma}

\begin{proof}[Proof of Lemma~\ref{lem:observed-improvement-bound}.] By Eq.~\eqref{eq:formula-4-16}, $\hat L_k=\hat Z_k\cdot\frac{1}{n_F}\sum_{i=1}^{n_F}\bigl[\hat V(\tilde A,t^F_i)-\hat V(A_k,t^F_i)\bigr]$. Each $\hat V$ is an average of rewards, so $\hat V\in[0,1]$ by Assumption~\ref{ass:bounded-reward} and each bracket lies in $[-1,1]$; hence so does their average, and multiplying by $\hat Z_k\ge0$ gives Eq.~\eqref{eq:formula-5-3}. \end{proof}

Let $\mathcal E_L$ denote the intersection of $\mathsf{Ready}_k$ with the events on which the failure-rate, rollout, and returned failure-mean error bounds hold in Lemma~\ref{lem:improvement-estimation}. On $\mathcal E_L$,
\[
|\hat L_k-L_{A_k}(\tilde A)|\le\varepsilon_L,
\qquad
\mathcal E_L\subseteq\mathsf{Ready}_k\cap\mathcal E_Z,
\]
so $\hat L_k\le L_{A_k}(\tilde A)+\varepsilon_L\le Z_k\bar a_k+\varepsilon_L$.
Together with Lemma~\ref{lem:observed-improvement-bound}, this gives $\hat L_k\le\Theta_k:=\min\{\hat Z_k,Z_k\bar a_k+\varepsilon_L\}$.

\begin{definition}[Admissible region]
\label{def:admissible-region}
On \(\mathcal E_L\), define the admissible region of Two-Gate parameters by
\begin{equation}
\mathcal R:=\left\{(\tau,\delta)\in\mathbb R_{>0}^2:
(1-\widehat Z_k+\varepsilon_Z)(\delta+\varepsilon_D)+\varepsilon_L
<\tau\le\Theta_k\right\}.
\label{eq:formula-5-7}
\end{equation}
\end{definition}

\begin{theorem}
\label{thm:admissibility}
The admissible region \(\mathcal R\) is nonempty if and only if \(\delta_{\max}>0\), where
\begin{equation}
\delta_{\max}:=
\frac{\Theta_k-\varepsilon_L-\varepsilon_Z\varepsilon_D-\varepsilon_\Sigma}
{1-\widehat Z_k+\varepsilon_Z}.
\label{eq:formula-5-8}
\end{equation}
When \(\delta_{\max}>0\), there exists a \(\tau\) with \((\tau,\delta)\in\mathcal R\) exactly when \(0<\delta<\delta_{\max}\). The corresponding values of \(\tau\) form the interval
\[
\left(
(1-\widehat Z_k+\varepsilon_Z)(\delta+\varepsilon_D)+\varepsilon_L,
\Theta_k
\right].
\]
\end{theorem}

\begin{proof}[Proof of Theorem~\ref{thm:admissibility}.] For a fixed $\delta>0$ the $\tau$-interval is non-empty iff $(1-\hat Z_k+\varepsilon_Z)(\delta+\varepsilon_D)+\varepsilon_L<\Theta_k$. Substituting $\varepsilon_D=\varepsilon_\Sigma/(1-\hat Z_k)$ and expanding,
\[
(1-\hat Z_k)\delta+\varepsilon_\Sigma+\varepsilon_Z\delta+\varepsilon_Z\varepsilon_D+\varepsilon_L \;<\; \Theta_k,
\]
i.e. $\delta\,(1-\hat Z_k+\varepsilon_Z)<\Theta_k-\varepsilon_L-\varepsilon_Z\varepsilon_D-\varepsilon_\Sigma$, which is Eq.~\eqref{eq:formula-5-8}. The coefficient $(1-\hat Z_k+\varepsilon_Z)$ is strictly positive, so the left side is strictly increasing in $\delta$; hence the admissible $\delta$ form the interval $(0,\delta_{\max})$, non-empty iff $\delta_{\max}>0$, and $\mathcal{R}$ is the union of the corresponding non-empty $\tau$-intervals. \end{proof}

\begin{corollary}
\label{cor:necessary-failure-rate}
A necessary condition for the admissible region of Definition~\ref{def:admissible-region} to be nonempty is $\Theta_k>\varepsilon_\Sigma+\varepsilon_L$. Its two branches imply the following bounds.

\begin{resultparts}
\item\label{part:necessary-observed-failure-rate}On every sample path where the region is evaluated,
\begin{equation}
\varepsilon_\mu<1 \quad\text{and}\quad \hat Z_k > \frac{\varepsilon_\Sigma+\varepsilon_Z+\varepsilon_Z\varepsilon_\mu}{1-\varepsilon_\mu}.
\label{eq:formula-5-9}
\end{equation}

\item\label{part:zero-failure-radius}Within the feasible domain $2\varepsilon_V<1$, consider a sampling design along which the returned radius $\varepsilon_F$ can be driven to zero while the other radii are fixed. Then $\varepsilon_\mu\to2\varepsilon_V$ and
\begin{equation}
\hat Z_k > \hat Z_\infty:=
\frac{2\varepsilon_V+\varepsilon_R+2\varepsilon_Z+2\varepsilon_V\varepsilon_Z}{1-2\varepsilon_V}.
\label{eq:formula-5-9-prime}
\end{equation}
\item\label{part:necessary-expected-reward}On the concentration event $\mathcal E_L$ used by Definition~\ref{def:admissible-region},
\begin{equation}
Z_k\bar a_k>\varepsilon_\Sigma.
\label{eq:formula-5-10}
\end{equation}
The corresponding expected-reward bound is Corollary~\ref{cor:finite-evaluation-threshold}.

\end{resultparts}
\end{corollary}

Branch (\ref{part:necessary-observed-failure-rate}) requires the empirical statistic to exceed its finite-sample error; branch (\ref{part:necessary-expected-reward}) requires $Z_k\bar a_k$ to exceed the retained-task measurement error. The zero-failure-radius limit holds the other radii fixed. At a fixed finite arrival horizon, increasing $n_F$ eventually makes return impossible. Appendix~\ref{app:conditional-numerics} gives the i.i.d. reference calculation when the rollout risk allocation also varies with $n_F$. The allowed retained-task change enters separately in the positive-$\delta$ condition below.

\begin{proof}[Proof of Corollary~\ref{cor:necessary-failure-rate}.] Since $\delta>0$ and $1-\widehat Z_k+\varepsilon_Z>0$, dropping the positive left side and the nonnegative term $\varepsilon_Z\varepsilon_D$ from Eq.~\eqref{eq:formula-5-8} gives the necessary inequality $\Theta_k>\varepsilon_\Sigma+\varepsilon_L$. For (\ref{part:necessary-observed-failure-rate}), substitute $\Theta_k\le\hat Z_k$ and expand $\varepsilon_L=\varepsilon_Z+(\hat Z_k+\varepsilon_Z)\varepsilon_\mu$ from Eq.~\eqref{eq:formula-4-26}:
\[
\hat Z_k>\varepsilon_\Sigma+\varepsilon_Z+\hat Z_k\varepsilon_\mu+\varepsilon_Z\varepsilon_\mu
\quad\Longleftrightarrow\quad
\hat Z_k(1-\varepsilon_\mu)>\varepsilon_\Sigma+\varepsilon_Z+\varepsilon_Z\varepsilon_\mu .
\]
If $\varepsilon_\mu\ge1$ the left side is non-positive while the right side is strictly positive, so no $\hat Z_k$ satisfies it; if $\varepsilon_\mu<1$, dividing by $1-\varepsilon_\mu>0$ gives Eq.~\eqref{eq:formula-5-9}. For (\ref{part:zero-failure-radius}), hold the other radii fixed and consider any returned-radius design along which $\varepsilon_F\downarrow0$. The right side of Eq.~\eqref{eq:formula-5-9} tends to the expression in Eq.~\eqref{eq:formula-5-9-prime}, using $\varepsilon_\mu\downarrow2\varepsilon_V$ and $\varepsilon_\Sigma+\varepsilon_Z=2\varepsilon_V+\varepsilon_R+2\varepsilon_Z$. Monotonicity in $\varepsilon_\mu$ makes it a lower boundary for every positive returned radius; every finite positive-radius design still obeys the strict condition Eq.~\eqref{eq:formula-5-9}. For (\ref{part:necessary-expected-reward}), substitute $\Theta_k\le Z_k\bar a_k+\varepsilon_L$, cancel $\varepsilon_L$, and apply Eq.~\eqref{eq:formula-3-29}. \end{proof}

\begin{corollaryrestatement}{cor:fixed-tolerance-constraint}
Under Assumptions~\ref{ass:bounded-reward} and~\ref{ass:agent-independent-tasks}, fix the threshold \(\delta>0\) on estimated retained-task change in Two-Gate. If some \(\tau>0\) satisfies \((\tau,\delta)\in\mathcal R\), then, on \(\mathcal E_L\),
\[
J_{\mathrm{user}}(A_k)
<1-(1-Z_k)\delta-\varepsilon_\Sigma.
\]
\end{corollaryrestatement}

Under these conditions, the stronger intermediate bound is
\begin{equation}
Z_k\bar a_k \;>\; (1-\hat Z_k+\varepsilon_Z)\,\delta \;+\; \varepsilon_\Sigma \;+\; \varepsilon_Z\varepsilon_D .
\label{eq:formula-5-13}
\end{equation}
Since $\mathcal E_L\subseteq\mathcal E_Z$ of Lemma~\ref{lem:failure-rate-estimation}, it follows that
\begin{equation}
Z_k\,\bar a_k \;>\; (1-Z_k)\,\delta+\varepsilon_\Sigma, \qquad\text{equivalently}\qquad Z_k \;>\; \frac{\delta+\varepsilon_\Sigma}{\bar a_k+\delta} \;\ge\; \frac{\delta}{\bar a_k+\delta},
\label{eq:formula-5-14}
\end{equation}
For the one-sided rule Eq.~\eqref{eq:one-sided-adoption-rule}, the same derivation uses $\delta^-$ together with its one-sided radius $\varepsilon_D^-$ and the event $\mathcal E_R^-$; the retained-risk allocation is replaced, not split. Corollary~\ref{cor:finite-evaluation-threshold} gives the corresponding necessary condition without a fixed positive retained-change requirement.

\begin{proof}[Proof of Corollary~\ref{cor:fixed-tolerance-constraint}.] By Theorem~\ref{thm:admissibility} a radius $\delta$ admits an accompanying $\tau$ exactly when $\delta<\delta_{\max}$, i.e. --- clearing the strictly positive denominator of Eq.~\eqref{eq:formula-5-8} --- when $\delta\,(1-\hat Z_k+\varepsilon_Z)<\Theta_k-\varepsilon_L-\varepsilon_Z\varepsilon_D-\varepsilon_\Sigma$. A minimum is at most either of its arguments, so $\Theta_k\le Z_k\bar a_k+\varepsilon_L$; substituting and cancelling $\varepsilon_L$ gives Eq.~\eqref{eq:formula-5-13}.

For Eq.~\eqref{eq:formula-5-14}, discard the non-negative $\varepsilon_Z\varepsilon_D$ and note that on $\mathcal{E}_Z$ one has $\hat Z_k-Z_k\le\varepsilon_Z$, hence $1-\hat Z_k+\varepsilon_Z\ge1-Z_k$. The rearrangement for $Z_k$ divides by $\bar a_k+\delta>0$, and the last inequality drops $\varepsilon_\Sigma\ge0$. The term $\delta/(\bar a_k+\delta)$ in this finite-sample condition contains no sampling radius. For Eq.~\eqref{eq:formula-5-15}, identity Eq.~\eqref{eq:formula-3-29} gives $1-J_{\text{user}}(A_k)\ge Z_k\bar a_k$, the discarded term being non-negative by Assumption~\ref{ass:bounded-reward}. \end{proof}

\begin{corollary}
\label{cor:budget-requirements}
Work in the domain $2\varepsilon_V<1$ of Eq.~\eqref{eq:formula-5-9-prime}. Write $L_V:=\log(2/\beta_V^{(1)})$, $L_R:=\log(2/\widetilde\beta_R)$, $L_Z:=\log(2/\widetilde\beta_Z)$, and $L_F:=\log(2/\beta_{F,k})$.
\begin{resultparts}[alph]
\item\label{part:measurement-budget}For the stated estimator and radius formulas, the zero-failure-radius boundary Eq.~\eqref{eq:formula-5-9-prime} satisfies $\hat Z_\infty\ge2\varepsilon_V+\varepsilon_R+2\varepsilon_Z$. Hence $\hat Z_k>\hat Z_\infty$ requires
\begin{equation}
m>\frac{2L_V}{\hat Z_k^2},\qquad
n_R>\frac{L_R}{2\hat Z_k^2},\qquad
n_R>\frac{2L_Z}{\hat Z_k^2}.
\label{eq:formula-5-17}
\end{equation}

\item\label{part:failure-quota}Whenever the returned-sample construction supplies $\varepsilon_F=\sqrt{2q_{F,k}L_F/n_F}$, solving Eq.~\eqref{eq:formula-5-9} for the quota gives
\begin{equation}
n_F>
\frac{2q_{F,k}L_F(\hat Z_k+\varepsilon_Z)^2}
{(1-2\varepsilon_V)^2(\hat Z_k-\hat Z_\infty)^2},
\label{eq:formula-5-18}
\end{equation}
for $\hat Z_k>\hat Z_\infty$. The requirement diverges as $\hat Z_k\downarrow\hat Z_\infty$ and is unsatisfiable below that boundary. It is a returned-sample rate, not an arrival-time identity. If risk allocation makes $\varepsilon_V$ depend on $n_F$, Eq.~\eqref{eq:formula-5-18} is an implicit inequality. Under the first-failure i.i.d. construction $q_{F,k}=1$; for a finite horizon, return is instead governed by $B_{H_{F,k},n_F}(Z_k)$.
\end{resultparts}
\end{corollary}

\begin{proof}[Proof of Corollary~\ref{cor:budget-requirements}.] \textbf{(\ref{part:measurement-budget})} By Eq.~\eqref{eq:formula-5-9-prime}, $\hat Z_k>\hat Z_\infty=(2\varepsilon_V+\varepsilon_R+2\varepsilon_Z+2\varepsilon_V\varepsilon_Z)/(1-2\varepsilon_V)\ge2\varepsilon_V+\varepsilon_R+2\varepsilon_Z$, the last step because the numerator dominates $2\varepsilon_V+\varepsilon_R+2\varepsilon_Z$ and the denominator lies in $(0,1]$. All three summands being non-negative, each is separately $<\hat Z_k$. With Lemmas~\ref{lem:value-estimation}, \ref{lem:failure-rate-estimation}, and \ref{lem:retained-change-estimation}: $2\varepsilon_V=2\sqrt{L_V/(2m)}<\hat Z_k$ rearranges to the first display of Eq.~\eqref{eq:formula-5-17}, $\varepsilon_R=\sqrt{L_R/(2n_R)}<\hat Z_k$ to the second, $2\varepsilon_Z=2\sqrt{L_Z/(2n_R)}<\hat Z_k$ to the third. 

\textbf{(\ref{part:failure-quota})} The proof of Corollary~\ref{cor:necessary-failure-rate}(\ref{part:necessary-observed-failure-rate}) reaches $\hat Z_k(1-\varepsilon_\mu)>\varepsilon_\Sigma+\varepsilon_Z+\varepsilon_Z\varepsilon_\mu$. Substituting $\varepsilon_\mu=2\varepsilon_V+\varepsilon_F$ and collecting the $\varepsilon_F$ terms,
\[
\hat Z_k(1-2\varepsilon_V)-\bigl(\varepsilon_\Sigma+\varepsilon_Z+2\varepsilon_V\varepsilon_Z\bigr) \;>\; \varepsilon_F\,(\hat Z_k+\varepsilon_Z),
\]
and the bracket equals $(1-2\varepsilon_V)\hat Z_\infty$ by Eq.~\eqref{eq:formula-5-9-prime}, so the left side is $(1-2\varepsilon_V)(\hat Z_k-\hat Z_\infty)$. Under the sampler-specific specialization $\varepsilon_F=\sqrt{2q_{F,k}L_F/n_F}$, division by the positive $\hat Z_k+\varepsilon_Z$ gives Eq.~\eqref{eq:formula-5-18}. The symbolic bound itself requires only a valid returned radius; the displayed relation for the quota is conditional on this rate specialization. If $\hat Z_k\le\hat Z_\infty$ and $\varepsilon_F>0$, no quota satisfies the inequality. \end{proof}

\begin{corollary}
\label{cor:finite-evaluation-threshold}
Under Assumptions~\ref{ass:bounded-reward} and \ref{ass:agent-independent-tasks}, on the concentration event $\mathcal E_L$ used by Definition~\ref{def:admissible-region}, if its admissible region $\mathcal R$ is nonempty at step $k$, then
\begin{equation}
J_{\text{user}}(A_k) \;<\; 1-\varepsilon_\Sigma, \qquad \varepsilon_\Sigma = 2\varepsilon_V+\varepsilon_R+\varepsilon_Z.
\label{eq:formula-5-19}
\end{equation}

\end{corollary}

\begin{proof}[Proof of Corollary~\ref{cor:finite-evaluation-threshold}.] Corollary~\ref{cor:necessary-failure-rate} gives $\Theta_k>\varepsilon_\Sigma+\varepsilon_L$ as necessary for $\mathcal{R}\neq\varnothing$, with $\Theta_k=\min\{\hat Z_k,\,Z_k\bar a_k+\varepsilon_L\}$. A minimum is at most either argument, so $Z_k\bar a_k+\varepsilon_L\ge\Theta_k>\varepsilon_\Sigma+\varepsilon_L$, whence $Z_k\bar a_k>\varepsilon_\Sigma$ irrespective of which branch binds. Identity Eq.~\eqref{eq:formula-3-29} gives $1-J_{\text{user}}(A_k)\ge Z_k\bar a_k$, since the discarded term $(1-Z_k)(1-\mathbb{E}_{\mathcal{D}_{R,k}}[V])$ is non-negative by Assumption~\ref{ass:bounded-reward}. \end{proof}
\begin{proposition}[Finite-data validation with Measured-margin]\label{prop:certification-beyond-threshold}
For the improving instance of Proposition~\ref{prop:positive-margin-instances}, fix $Z_k\in(0,1)$ and the total risk allocation.
One can choose $m,n_F,n_R$ each of order $\widetilde O(u_k^{-2})$ such that Rule~\ref{alg:measured-margin} validates the modification whenever the prescribed samples are ready and the joint error bounds hold. Here $m$ is the rollout count per agent per task, $n_F$ the returned failure-task count, and $n_R$ the stored-record count.
\end{proposition}
\begin{proof}[Proof of Proposition~\ref{prop:certification-beyond-threshold}.] Take the improving member of Proposition~\ref{prop:positive-margin-instances}, where $M=L=u_k/2$ and both agents return reward $1$ on $t_R$. Choose the i.i.d. returned-sample design with $\varepsilon_Z\le(1-Z_k)/2$. On $\mathsf{Ready}_k$, when all the stated error bounds hold,
$1-\widehat Z_k\ge(1-Z_k)/2>0$ and $\widehat D^-_{R,k}=0$. By Eq.~\eqref{eq:formula-4-24} and $\widehat L_k\ge L-\varepsilon_L$,
\[
\widehat M_k
\ge
\frac{u_k}{2}
-
\underbrace{2\varepsilon_L+
(1-\widehat Z_k+\varepsilon_Z)\varepsilon_D^-}_{\text{total error bound}}.
\]
At fixed $Z_k$ and confidence allocation, Eq.~\eqref{eq:formula-4-26} gives
$\varepsilon_L=\varepsilon_Z+(\widehat Z_k+\varepsilon_Z)(2\varepsilon_V+\varepsilon_F)
\le\varepsilon_Z+2(2\varepsilon_V+\varepsilon_F)$.
The bound $1-\widehat Z_k\ge(1-Z_k)/2$ gives
$\varepsilon_D^-\le2(2\varepsilon_V+\varepsilon_R+\varepsilon_Z)/(1-Z_k)$,
while $1-\widehat Z_k+\varepsilon_Z\le2$. It suffices to choose the component radii so that $\varepsilon_L<u_k/8$ and $\varepsilon_D^-<u_k/8$; then the underbraced error is less than $u_k/2$. In this i.i.d. design each component radius has inverse-square-root dependence on its respective $m,n_F,n_R$, with the existing risk-allocation logarithms. Thus sufficiently large constant multiples of $\widetilde O(u_k^{-2})$ for each count meet these inequalities, giving $\widehat M_k>0$ and $\widehat D^-_{R,k}=0\le\delta^-$. These counts specify returned evaluation data and rollouts per task; they do not guarantee a data-return time.

\end{proof}

\subsection{Sample-complexity lower bound for validation}
\label{app:section-G-21}

\paragraph{Setup.} An \emph{instance} is a configuration satisfying the assumptions of \S{}\ref{sec:safe-self-evolution} and the task-stream evaluation setup in Appendix~\ref{app:proofs-evaluation} --- a task space with distribution $\mathcal{D}_{\text{user}}$, a failure detector, a current agent $A_k$, and a candidate agent $\tilde A$ --- together with the available evaluation data: the observed task sequence, the stored reference sample with its stored labels, and fresh rollouts of either agent on any tasks, adaptively chosen, each rollout returning a reward in $[0,1]$. A \textbf{validation procedure at level $\beta$} is a measurable map from the evaluation data to $\{\textsf{validate},\textsf{abstain}\}$ such that in \emph{every} instance, $\mathbb{P}\bigl[\textsf{validate}\ \wedge\ J_{\text{user}}(\tilde A)\le J_{\text{user}}(A_k)\bigr]\le\beta$. Its task choices, statistics, and stopping rule are measurable functions of these observations and its own instance-independent randomization.

\begin{theoremrestatement}[Cost of validating further improvement]{thm:certification-sample-lower-bound}
Fix \(Z\in(0,1)\), \(\beta\le1/8\), and \(u\in(0,Z/2]\). There exist two instances with \(Z_k=Z\) and \(1-J_{\mathrm{user}}(A_k)=u\), identical in the distribution of every observable quantity except the candidate agent's reward distribution under the failure-task distribution. In one,
\[
J_{\mathrm{user}}(\widetilde A)-J_{\mathrm{user}}(A_k)=u/2,
\qquad D_R(\widetilde A;A_k)=0,
\]
and in the other the overall reward change is \(-u/2\). Any level-\(\beta\) validation procedure that validates the improving modification with probability at least \(1/2\), under any task-selection policy and any data-dependent stopping rule, must satisfy
\[
\mathbb E_-[N_{\mathrm{eval}}]\ge\frac{Z}{10u},
\]
where \(N_{\mathrm{eval}}\) is the total number of candidate-agent rollouts and the expectation is under the non-improving instance. Consequently, for fixed \(Z\), maintaining the stated error control and validation probability at least \(1/2\) on the improving instance in each pair requires a worst-case expected rollout count growing at least as \(u^{-1}\).
\end{theoremrestatement}

\begin{proof}\emph{Construction.} Two tasks, $\mathcal{T}_{\text{user}}=\{t_F,t_R\}$, with $\mathcal{D}_{\text{user}}=Z\,\delta_{t_F}+(1-Z)\,\delta_{t_R}$ and i.i.d.\ task records, so Assumption~\ref{ass:mixing-task-process} holds with $\beta(\cdot)\equiv0$. The detector is task-deterministic: $\phi(t_F,\cdot)\equiv1$, $\phi(t_R,\cdot)\equiv0$, so $\psi_A(t_F)=1$ and $\psi_A(t_R)=0$ for every agent, $Z_k=Z\in(0,1)$, and the conditionals are $\mathcal{D}_{F,k}=\delta_{t_F}$, $\mathcal{D}_{R,k}=\delta_{t_R}$. Set $\bar a:=u/Z\le1/2$. The current agent's reward distribution is $\delta_1$ on $t_R$ and $\mathrm{Bernoulli}(1-\bar a)$ on $t_F$, so $J_{\text{user}}(A_k)=Z(1-\bar a)+(1-Z)=1-u$. The two instances differ only in the candidate agent: on $t_R$ it induces the current agent's reward distribution (hence $D_R^-=D_R^+=D_R=0$ in both), and on $t_F$ its reward distribution is $\nu_\pm:=\mathrm{Bernoulli}\bigl(1-\bar a\pm\tfrac{\bar a}{2}\bigr)$. Then $\Delta J=\pm Z\bar a/2=\pm u/2$, and in the $+$ instance the candidate agent's margin is $M=L_{A_k}=u/2>0$ with $D_R^-=0\le\delta^-$ for every declared retained-task decrease bound. These finite kernels are realized by measurable agent and composition maps. Draw task records independently, keep the stored reference sample separate from candidate evaluation, and generate fresh rollouts independently from the appropriate fixed kernel. With the task-deterministic detector, both instances satisfy Assumptions~\ref{ass:bounded-reward}, \ref{ass:agent-independent-tasks}, \ref{ass:mixing-task-process}, \ref{ass:failure-detector-locality}, \ref{ass:task-outcome-model}, \ref{ass:outcome-evaluation-independence}, and \ref{ass:trial-consistency} and Condition~\ref{cond:evaluation-data-separation} by construction.

\emph{Indistinguishability.} Conditional on the same observed history and chosen evaluation action, the two instances give the same observation kernel except for candidate rollouts on $t_F$. The task-selection and stopping kernels of a fixed validation procedure are the same under this conditioning. For the informative rollouts, with $p_-=1-\tfrac{3\bar a}{2}$ and $p_+=1-\tfrac{\bar a}{2}$, using $\mathrm{KL}(\mathrm{Ber}(p)\Vert\mathrm{Ber}(q))\le(p-q)^2/\bigl(q(1-q)\bigr)$ and $q(1-q)=(1-\tfrac{\bar a}{2})\tfrac{\bar a}{2}\ge\tfrac{3\bar a}{8}$ for $\bar a\le1/2$:
\[
\mathrm{KL}(\nu_-\Vert\nu_+)\;\le\;\frac{\bar a^2}{3\bar a/8}\;=\;\frac{8\bar a}{3}.
\]
For a transcript truncated after $n$ observations, the conditional KL chain rule therefore gives
\[
\mathrm{KL}(\mathbb P_-^{(n)}\Vert\mathbb P_+^{(n)})
\le\mathbb E_-[N_F^{(n)}]\,\mathrm{KL}(\nu_-\Vert\nu_+)
\le\frac{8u}{3Z}\mathbb E_-[N_{\mathrm{eval}}],
\]
where $N_F^{(n)}$ counts candidate rollouts on $t_F$ before stopping or truncation. Observations after stopping may be padded by a fixed symbol. Relative entropy on increasing transcript sigma-fields converges to the entropy of the full transcript; thus the same bound holds for the stopped experiment, including unbounded stopping. If the expected count is infinite the claimed cost bound already holds. Pinsker's inequality and data processing for the validation decision give
\[
\mathrm{TV}(\mathbb P_+,\mathbb P_-)
\le\sqrt{4u\,\mathbb E_-[N_{\mathrm{eval}}]/(3Z)}.
\]
\emph{Two-point argument.} In the $-$ instance, $\Delta J<0$, so soundness forces $\mathbb{P}_-[\textsf{validate}]\le\beta$. If the validation procedure validates the improving modification with probability at least $1/2$, then
\[
\tfrac12\;\le\;\mathbb{P}_+[\textsf{validate}]\;\le\;\mathbb{P}_-[\textsf{validate}]+\mathrm{TV}(\mathbb{P}_+,\mathbb{P}_-)\;\le\;\beta+\sqrt{4\,\mathbb{E}_-[N_{\mathrm{eval}}]\,\bar a/3},
\]
so $\sqrt{4\,\mathbb{E}_-[N_{\mathrm{eval}}]\,\bar a/3}\ge\tfrac12-\beta\ge\tfrac38$, hence $\mathbb{E}_-[N_{\mathrm{eval}}]\ge\tfrac{27}{256\,\bar a}\ge\tfrac{Z}{10\,u}$. \end{proof}

\begin{remark}\label{rem:lower-bound-interpretation}
For candidate reward on $t_F$ in the two constructed instances, the Bernoulli means differ by $\bar a=u/Z$ and their variances are at most $3\bar a/2$. With fixed $Z$, Bernstein concentration for this coordinate gives error at most
$O(\sqrt{(u/Z)\log(1/\beta)/n}+\log(1/\beta)/n)$.
Thus $n=O((Z/u)\log(1/\beta))$ observations suffice to resolve a fixed fraction of that gap on this coordinate.
\end{remark}

\subsection{A one-sided rule for cumulative reward changes}\label{app:one-sided-cumulative-rule}
Use the candidate and selector restrictions, readiness conditions, and simultaneous evaluation events of Theorem~\ref{thm:one-step-guarantee}, replacing the symmetric retained-task event by $\mathcal E_R^-$ from Appendix~\ref{app:one-sided-estimation}. Define a rule that adopts a selected candidate only when
\begin{equation}
\widehat L_k\ge\tau,\qquad
\widehat D^-_{R,k}\le\delta^-,
\qquad
\Delta^-_k:=
\tau-\varepsilon_L^{(k)}
-(1-\widehat Z_k+\varepsilon_Z^{(k)})
(\delta^-+\varepsilon_D^{-,(k)})>0.
\label{eq:one-sided-adoption-rule}
\end{equation}
The third check supplies a positive lower bound on improvement. Let $\mathsf{Acc}^-_k$ be this rule's adoption event. On its joint coverage event, Eq.~\eqref{eq:formula-3-17} and the one-sided estimate imply
\[
I_k\ge\widehat L_k-\varepsilon_L^{(k)}
-(1-\widehat Z_k+\varepsilon_Z^{(k)})
(\widehat D^-_{R,k}+\varepsilon_D^{-,(k)})
\ge\Delta^-_k>0.
\]
Consequently,
$\mathbb P(\mathsf{Acc}^-_k\cap\{I_k<\Delta^-_k\})
\le\beta_{\mathrm{step},k}$ with the original joint allocation: $\mathcal E_R^-$ replaces, rather than augments, the symmetric retained event. The required positive-gain inequality uses the same failure and rollout radii as Two-Gate; the retained radius and event are one-sided. Appendix~\ref{app:cumulative-change} bounds absolute changes on this rule's own adopted-step set.

\subsection{Calculation for Figure~\ref{fig:boundary-position}}\label{app:conditional-numerics}\label{app:section-I-1}
\begin{figure}[H]
\centering
\includegraphics[width=\textwidth]{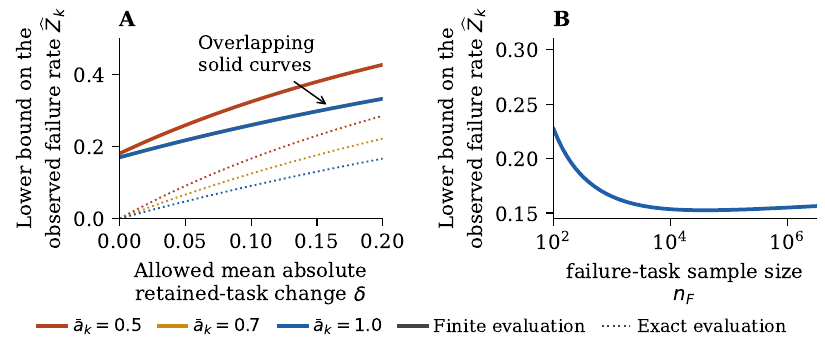}
\caption{\textbf{Necessary lower bounds on the failure rate for further validation.} (A) The allowed mean absolute retained-task change restricts validation even with exact evaluation (dotted); finite evaluation error further restricts it (solid). (B) The failure-task sample size varies while the per-task rollout count, retained-record count, and total risk allocations are fixed. The risk share per reward estimate consequently changes with the number of evaluated tasks. Curves use the i.i.d. design in Appendix~\ref{app:conditional-numerics}.}
\label{fig:boundary-position}
\end{figure}
Write $h=\widehat Z_k$ and let $\delta$ denote the allowed mean absolute retained-task change in Two-Gate. Setting $\delta^-=\delta$ gives the same numerical parameter for comparison with the one-sided rule.
On the relevant coverage event, the two necessary branches of the admissible-region inequality imply
\[
h>\varepsilon_L(h)+\varepsilon_Z\varepsilon_D(h)+\varepsilon_\Sigma
+\delta(1-h+\varepsilon_Z),
\]
\[
(h+\varepsilon_Z)\bar a_k>
\varepsilon_Z\varepsilon_D(h)+\varepsilon_\Sigma
+\delta(1-h+\varepsilon_Z).
\]
Panel A uses the larger lower root of these inequalities, at
$m=2\times10^4$, $n_R=2000$, $n_F=1000$, with i.i.d. records and
$(\beta_Z,\beta_R,\beta_F,\beta_V)=(0.014,0.024,0.02,0.05)$.
The rollout allocation is $\beta_V^{(1)}=\beta_V/[2(n_F+n_R)]$.
The solid curves for $\bar a_k=0.7$ and $1$ coincide because the branch independent of $\bar a_k$ determines their root; the exact-evaluation references $\delta/(\bar a_k+\delta)$ differ.
Panel B varies $n_F$ in the relaxed necessary bound Eq.~\eqref{eq:formula-5-9} at the same $m,n_R$ and total risks. The subsequent rise results from allocating the fixed total rollout risk over more tasks while the failure-task mean radius decreases.

\section{Multiple Updates and Task-Distribution Changes}\label{app:multistep}
\subsection{Cumulative improvement}\label{app:proofs-multistep}\label{app:section-I-4-1}
Use the adopted-step set $\mathcal A$, actual increments $I_k$, and pre-adoption bounds of Section~\ref{sec:multiple-updates}.
\begin{theoremrestatement}[Progressive improvement through evolution]{thm:finite-run-guarantee}
Fix $\mathcal D_{\mathrm{user}}$ and run Procedure~\ref{alg:finite-horizon-update} for $K$ steps under the hypotheses of Theorem~\ref{thm:one-step-guarantee} at every step. If its deterministic risks satisfy $\sum_{k<K}\beta_{\mathrm{step},k}\le\beta_{\mathrm{run}}$, then with probability at least $1-\beta_{\mathrm{run}}$,
\[
J_{\mathrm{user}}(A_K)-J_{\mathrm{user}}(A_0)
\ge\sum_{k\in\mathcal A}\Delta_k.
\]
Here $\Delta_k$ has the definition and conventions of Eq.~\eqref{eq:formula-4-28}.
\end{theoremrestatement}

\begin{proof}[Proof of Theorem~\ref{thm:finite-run-guarantee}.]
For each step define the quantitative bad-accept event
$\mathcal B_k:=\mathsf{Acc}_k\cap\{I_k<\Delta_k\}$.  Theorem~\ref{thm:one-step-guarantee} gives $\mathbb P(\mathcal B_k)\le\beta_{\mathrm{step},k}$ under the actual step protocol, so
\[
\mathbb P\!\left(\bigcup_{k<K}\mathcal B_k\right)
\le\sum_{k<K}\beta_{\mathrm{step},k}\le\beta_{\mathrm{run}}.
\]
Outside this union, every accepted step satisfies $I_k\ge\Delta_k$. Definition~\ref{def:evolution-step} gives $\sum_{k<K}I_k=\sum_{k\in\mathcal A}I_k$. Summing the increments and the pre-acceptance margins only over accepted steps gives
\[
J_{\text{user}}(A_K)-J_{\text{user}}(A_0)
=\sum_{k<K}I_k\ge\sum_{k\in\mathcal A}\Delta_k.
\]
The measured-margin statement follows identically from Proposition~\ref{prop:measured-margin-guarantee} using
$\mathsf{Acc}^M_k\cap\{I_k<\widehat M_k\}$. \end{proof}
For Measured-margin, on its own adopted-step set $\mathcal A^M$, the same argument gives
\[
J_{\mathrm{user}}(A_K)-J_{\mathrm{user}}(A_0)
\ge\sum_{k\in\mathcal A^M}\widehat M_k
\]
with probability at least $1-\beta_{\mathrm{run}}$, whenever Proposition~\ref{prop:measured-margin-guarantee} holds at every step and the deterministic step risks sum to at most $\beta_{\mathrm{run}}$.
Under equal risk allocation each step uses $\beta_{\mathrm{run}}/K$. In the i.i.d. marked-stream design, substitution in Proposition~\ref{prop:first-failure-sampling} adds the corresponding $\log K$ term to its sample requirement.
\begin{definition}[Distribution drift]
\label{def:distribution-drift}
For the task distributions $\mathcal D_{\mathrm{user}}^{(k)}$ and $\mathcal D_{\mathrm{user}}^{(k+1)}$ at adjacent self-evolution steps, define their drift by $\eta_k:=\mathrm{TV}(\mathcal D_{\mathrm{user}}^{(k+1)},\mathcal D_{\mathrm{user}}^{(k)})$.
Write $J_{\mathrm{user}}^{(k)}(A):=\mathbb E_{t\sim\mathcal D_{\mathrm{user}}^{(k)}}[V(A,t)]$ for expected reward under the step-$k$ distribution.
\end{definition}

\begin{lemma}
\label{lem:distribution-drift}
For any $A$:
\begin{equation}
|J_{\text{user}}^{(k+1)}(A) - J_{\text{user}}^{(k)}(A)| \leq \eta_k.
\label{eq:formula-6-18}
\end{equation}
\end{lemma}

\begin{proof} The signed measure $\mathcal D_{\mathrm{user}}^{(k+1)}-\mathcal D_{\mathrm{user}}^{(k)}$ has total mass zero. Hence
\[
\left|\int V(A,t)\,d(\mathcal D_{\mathrm{user}}^{(k+1)}-\mathcal D_{\mathrm{user}}^{(k)})(t)\right|
\le\operatorname{osc}(V(A,\cdot))\,\eta_k\le\eta_k,
\]
by bounded rewards. \end{proof}

\begin{theorem}[Cumulative gain under changing task distributions]
\label{thm:drifting-task-guarantee}
Under Assumptions~\ref{ass:bounded-reward}, \ref{ass:failure-detector-locality}, \ref{ass:task-outcome-model}, \ref{ass:outcome-evaluation-independence}, and \ref{ass:trial-consistency} and a drift sequence $\{\eta_k\}_{k=0}^{K-1}$, suppose that the actual finite-horizon protocol at step $k$ supplies the returned failure-task coverage and stored-sample readiness required by Theorem~\ref{thm:one-step-guarantee} for $\mathcal D_{\mathrm{user}}^{(k)}$. Let $I_k=J_{\mathrm{user}}^{(k)}(A_{k+1})-J_{\mathrm{user}}^{(k)}(A_k)$ and use the pre-acceptance $\Delta_k$ from Eq.~\eqref{eq:formula-4-28}, with zero value outside $\mathsf{Cand}_k\cap\mathsf{Ready}_k$. Write
\[
\mathcal B_k:=\mathsf{Acc}_k\cap\{I_k<\Delta_k\},
\qquad \mathbb P(\mathcal B_k)\le\beta_{\mathrm{step},k},
\qquad \sum_{k<K}\beta_{\mathrm{step},k}\le\beta_{\mathrm{run}}.
\]
Then, without requiring independence across steps, with probability at least $1-\beta_{\mathrm{run}}$,
\begin{equation}
J_{\text{user}}^{(K)}(A_K)-J_{\text{user}}^{(0)}(A_0)
\ge\sum_{k\in\mathcal A}\Delta_k-\sum_{k=0}^{K-1}\eta_k.
\label{eq:formula-6-19}
\end{equation}
\end{theorem}

\begin{proof}[Proof of Theorem~\ref{thm:drifting-task-guarantee}.]
Let $\mathcal B:=\bigcup_{k<K}\mathcal B_k$.  The assumed per-step quantitative bounds and a union bound give
$\mathbb P(\mathcal B)\le\sum_{k<K}\beta_{\mathrm{step},k}\le\beta_{\mathrm{run}}$.
Outside $\mathcal B$, at each step add and subtract $J_{\mathrm{user}}^{(k)}(A_{k+1})$:
\[
J_{\mathrm{user}}^{(k+1)}(A_{k+1})-J_{\mathrm{user}}^{(k)}(A_k)
=I_k+\bigl[J_{\mathrm{user}}^{(k+1)}(A_{k+1})-J_{\mathrm{user}}^{(k)}(A_{k+1})\bigr].
\]
Lemma~\ref{lem:distribution-drift} bounds the bracket below by $-\eta_k$ on every step. Definition~\ref{def:evolution-step} makes $I_k=0$ on nonadopted steps, so the within-step increments sum to $\sum_{k\in\mathcal A}I_k\ge\sum_{k\in\mathcal A}\Delta_k$. Summing the displayed identity telescopes and yields Eq.~\eqref{eq:formula-6-19}. No cross-step independence is used. \end{proof}

\begin{corollary}
\label{cor:accepted-step-count}
On the event of Theorem~\ref{thm:finite-run-guarantee}, suppose $\Delta_k\ge\Delta_\star>0$ for every accepted step $k<K$. Then
\begin{equation}
|\mathcal A|\le\frac{1-J_{\mathrm{user}}(A_0)}{\Delta_\star}.
\label{eq:formula-6-24}
\end{equation}
\end{corollary}

The same positive lower bound $\Delta_\star$ must hold on every adopted step to obtain Eq.~\eqref{eq:formula-6-24}.

\begin{proof}[Proof of Corollary~\ref{cor:accepted-step-count}.] Theorem~\ref{thm:finite-run-guarantee} gives
\[
J_{\text{user}}(A_K)-J_{\text{user}}(A_0)
\ge\sum_{k\in\mathcal{A}}\Delta_k
\ge|\mathcal{A}|\Delta_\star.
\]
By Assumption~\ref{ass:bounded-reward}, $J_{\text{user}}(A_K)\le1$, so the same endpoint difference is at most $1-J_{\text{user}}(A_0)$. Combining the two bounds and dividing by $\Delta_\star>0$ gives Eq.~\eqref{eq:formula-6-24}. \end{proof}
\subsection{Cumulative absolute reward changes}
\label{app:cumulative-change}

The superscript $(k)$ identifies the evaluation radius at step $k$. The following bounds control both the accumulated absolute change along the update path and the endpoint change relative to the initial agent.

\paragraph{The symmetric Two-Gate bound.}
Fix $\mathcal D_{\mathrm{user}}$ throughout $K$ steps under Assumptions~\ref{ass:bounded-reward} and \ref{ass:agent-independent-tasks}, with $Z_k\in(0,1)$ on adopted steps. On the joint evaluation-coverage event, Proposition~\ref{prop:cumulative-retained-change} gives Eq.~\eqref{eq:cumulative-change-basic} for the Two-Gate adopted-step set $\mathcal A$.

\begin{proposition}[Cumulative absolute changes under a one-sided retained-loss rule]
\label{prop:one-sided-cumulative-change}
Fix the same task distribution and assumptions, and at each step use the candidate, readiness, selector, and risk conditions of the one-sided rule Eq.~\eqref{eq:one-sided-adoption-rule}. Let $\mathcal A^-:=\{k<K:\mathsf{Acc}^-_k\}$ be this rule's adopted-step set and
$\mathcal V_K^-:=\sum_{k\in\mathcal A^-}\mathbb E_{\mathcal D_{\mathrm{user}}}|V(A_{k+1},\cdot)-V(A_k,\cdot)|$.
On the common one-sided evaluation-coverage event,
\begin{equation}
\begin{aligned}
\mathcal V_K^-
&\le J_{\mathrm{user}}(A_K)-J_{\mathrm{user}}(A_0)
+\sum_{k\in\mathcal A^-}(Z_k-\tau+\varepsilon_L^{(k)})
+2\sum_{k\in\mathcal A^-}(1-Z_k)(\delta^-+\varepsilon_D^{-,(k)})\\
&\le1-J_{\mathrm{user}}(A_0)
+\sum_{k\in\mathcal A^-}(Z_k-\tau+\varepsilon_L^{(k)})
+2\sum_{k\in\mathcal A^-}(1-Z_k)(\delta^-+\varepsilon_D^{-,(k)}).
\end{aligned}
\label{eq:one-sided-cumulative-change}
\end{equation}
Moreover,
$\mathbb E_{\mathcal D_{\mathrm{user}}}|V(A_K,\cdot)-V(A_0,\cdot)|
\le\min\{1,\mathcal V_K^-\}$.
If the deterministic step risks sum to at most $\beta_{\mathrm{run}}$, the common event has probability at least $1-\beta_{\mathrm{run}}$.
\end{proposition}

\begin{proof}[Proof of Propositions~\ref{prop:cumulative-retained-change} and~\ref{prop:one-sided-cumulative-change}.]
For each adopted step write $a_k(t)=\mathrm{Adv}_{A_k}(A_{k+1},t)$ and
$F_k^-:=Z_k\mathbb E_{\mathcal D_{F,k}}(-a_k)_+$.
Lemma~\ref{lem:failure-decomposition}, applied to $|a_k|$, gives
\begin{equation}
\mathbb E_{\mathcal D_{\mathrm{user}}}|V(A_{k+1},\cdot)-V(A_k,\cdot)|
=Z_k\mathbb E_{\mathcal D_{F,k}}|a_k|
 +(1-Z_k)\mathbb E_{\mathcal D_{R,k}}|a_k|.
\label{eq:formula-6-20}
\end{equation}
Writing $|x|=x+2(-x)_+$ yields
\[
Z_k\mathbb E_{\mathcal D_{F,k}}|a_k|=L_{A_k}+2F_k^-.
\]
Since $\mathbb E_{\mathcal D_{F,k}}|a_k|\le1$, the condition $\widehat L_k\ge\tau$ in Eq.~\eqref{eq:formula-4-23} and the failure-task error bound imply
\begin{equation}
2F_k^-\le Z_k-L_{A_k}\le Z_k-\tau+\varepsilon_L^{(k)}.
\label{eq:formula-6-21}
\end{equation}
For Two-Gate, $\widehat D_R\le\delta$ in Eq.~\eqref{eq:formula-4-23} and the retained-task error bound give $D_R\le\delta+\varepsilon_D^{(k)}$. The reward decomposition in Eq.~\eqref{eq:formula-3-9} implies $L_{A_k}\le I_k+(1-Z_k)D_R$; hence Eqs.~\eqref{eq:formula-6-20}--\eqref{eq:formula-6-21} give
\[
\mathbb E_{\mathcal D_{\mathrm{user}}}|a_k|
\le I_k+Z_k-\tau+\varepsilon_L^{(k)}
+2(1-Z_k)(\delta+\varepsilon_D^{(k)}).
\]
On nonadopted steps $I_k=0$, so summing over $\mathcal A$ telescopes to
$\sum_{k\in\mathcal A}I_k=J_{\mathrm{user}}(A_K)-J_{\mathrm{user}}(A_0)\le1-J_{\mathrm{user}}(A_0)$.
This proves Eq.~\eqref{eq:cumulative-change-basic}.

For the one-sided rule, Eq.~\eqref{eq:formula-3-9} gives the exact identity
\[
\begin{aligned}
\mathbb E_{\mathcal D_{\mathrm{user}}}|a_k|
&=L_{A_k}+2F_k^-+(1-Z_k)(D^-_{R,k}+D^+_{R,k})\\
&=I_k+2F_k^-+2(1-Z_k)D^-_{R,k}.
\end{aligned}
\]
On the joint one-sided coverage event, Eq.~\eqref{eq:one-sided-adoption-rule} and $\mathcal E_R^-$ give
$D^-_{R,k}\le\delta^-+\varepsilon_D^{-,(k)}$.
Apply Eq.~\eqref{eq:formula-6-21}, sum only over $\mathcal A^-$, and telescope $\sum_{k\in\mathcal A^-}I_k=J_{\mathrm{user}}(A_K)-J_{\mathrm{user}}(A_0)$. Bounded reward gives $J_{\mathrm{user}}(A_K)\le1$, proving Eq.~\eqref{eq:one-sided-cumulative-change}. The taskwise triangle inequality bounds the endpoint absolute change by the path sum, while bounded rewards bound it by one. Finally, the per-step joint one-sided coverage event $\mathcal G_k^-$ in the proof of Proposition~\ref{prop:measured-margin-guarantee} has failure probability at most $\beta_{\mathrm{step},k}$; intersecting these events and applying a union bound gives the stated common-event probability. \end{proof}

Bounding the failure-task term in Eq.~\eqref{eq:formula-6-20} only by $|a_k|\le1$ contributes $Z_k$ on each adopted step. The condition $\widehat L_k\ge\tau$ instead yields the term $Z_k-\tau+\varepsilon_L^{(k)}$ through Eq.~\eqref{eq:formula-6-21}. The endpoint mean absolute difference is at most one, while the cumulative path change can exceed one.

\begin{remark}
\label{rem:cumulative-bound-residual}
Holding the realized sequence of agents, evaluation radii, and adopted-step set fixed, letting $\delta\downarrow0$ in the right-hand side of Eq.~\eqref{eq:cumulative-change-basic} leaves the retained-task contribution $2\sum_{k\in\mathcal A}(1-Z_k)\varepsilon_D^{(k)}$.
\end{remark}

\subsection{A successful update and later improvement opportunities}\label{app:later-step-proof}
\begin{proposition}[No later-step guarantee from step-0 quantities under unrestricted updates]
\label{prop:no-later-step-guarantee}
In the kernel setting of Appendix~\ref{app:section-A-1}, allow the composition map \((M,C)\mapsto A\) to be any measurable map consistent with Eqs.~\eqref{eq:formula-2-1}--\eqref{eq:formula-2-5}. For every pool size \(N\ge1\) and horizon \(K\ge2\), there exist a common target \(T\) and two instances with the same frozen LLM, initial agent, step-0 generation and evaluation experiment, validation and selection rules, and first successor agent. Both have \(P_0(T)=1\), and the first update is validated and strictly improves overall expected reward. In the first instance, \(P_1^+(T)>0\). In the second, every subsequently generated modification leaves expected reward unchanged on every user task, so, almost surely, \(P_k(T)=0\) for \(1\le k<K\), and
\[
J_{\mathrm{user}}(A_K)-J_{\mathrm{user}}(A_0)
=J_{\mathrm{user}}(A_1)-J_{\mathrm{user}}(A_0).
\]
The composition maps agree at the initial harness state and every step-0 candidate harness state, but may differ at later states.
\end{proposition}
\begin{proof}[Proof of Proposition~\ref{prop:no-later-step-guarantee}.] Take four user tasks $t_0,t_1,t_2,t_R$ of equal probability, binary outputs on these user tasks, reward $r(t,o)=o$, and detector $\phi(t,o)=1-o$. The initial agent has deterministic user-task outputs $A_0:(0,0,0,1)$ and outputs modification $c_0$ on every self-modification task. The common first successor has outputs $A_1:(1,0,0,1)$ and outputs $c_1$ on every self-modification task.

Use code states $C_0,C_1,C_2$ and modification space $\mathcal M=\{c_0,c_1\}$. The code edit $c_0$ sends $C_0$ to $C_1$ and otherwise leaves the code state unchanged; $c_1$ sends $C_1$ to $C_2$ and otherwise leaves it unchanged. Both instances use this same code-edit operation and the same frozen LLM $M$. Their measurable composition maps agree at $C_0$ and $C_1$, realizing the full kernels $A_0$ and $A_1$, respectively. At $C_2$, instance (a) has user-task outputs $(1,1,0,1)$ and self-task output $c_1$, whereas instance (b) has the same full kernel as $A_1$. All spaces and maps are finite and measurable, and Eq.~\eqref{eq:formula-2-5} holds because the current agent's self-task kernel generates the stated modification directly. In particular, all step-0 candidate agents agree between the two instances: $c_0$ gives $A_1$ and $c_1$ leaves $A_0$ unchanged.

Choose the common target $T=(1/8,1/16)$.
At step 0, $Z_0=3/4$, and applying $c_0$ gives
\[
L_0(c_0)=1/4,\qquad D_0(c_0)=0,\qquad
J_{\mathrm{user}}(A_1)-J_{\mathrm{user}}(A_0)=1/4.
\]
Every draw is $c_0$, so $P_0(T)=1$ for any pool size. At $A_1$, $Z_1=1/2$. In instance (a), applying $c_1$ improves expected reward only on $t_1$, giving $L_1(c_1)=1/4$ and $D_1(c_1)=0$.
In instance (b), applying $c_1$ preserves every user-task output, so $L_1(c_1)=D_1(c_1)=0$.

For these finite deterministic instances, the same evaluator checks the four user-task outputs exactly and returns singleton rectangles. Their total-width bounds are zero, so the set $\mathcal P^+$ in Eq.~\eqref{eq:formula-4-7} equals the true qualified set, with conditional coverage one. The validation rule in Eq.~\eqref{eq:formula-4-6}, followed by the same selector choosing a validated candidate whenever one exists, accepts $c_0$ at step 0 in both instances and $c_1$ at step 1 in instance (a). There $c_1$ meets $T$ and is generated with probability one, giving $P_1^+(T)=1$.

In instance (b), every later draw is $c_1$ and leaves expected reward unchanged on every user task, whether or not it is adopted. Under the stated validation rule it is rejected, and the full kernel remains $A_1$. Thus $P_k(T)=0$ for all $1\le k<K$ and
\[
J_{\mathrm{user}}(A_K)-J_{\mathrm{user}}(A_0)
=1/4
=J_{\mathrm{user}}(A_1)-J_{\mathrm{user}}(A_0)
\quad\text{almost surely}.
\]
Instance (a) instead increases expected reward by another $1/4$ at step 1 and then generates modifications that leave expected reward unchanged. The construction works for every $N\ge1$ and $K\ge2$. \end{proof}

\subsection{Agent-responsive task distributions}\label{app:agent-responsive}

All quantities in a same-distribution comparison below are constructed from the declared common reference measure.
\begin{assumption}[Agent-responsive user-task distribution]
\label{ass:agent-responsive-tasks}
At each step $k$ there is a task distribution $\mathcal{D}_{A,k}$ that obtains once the task stream has settled under agent $A$, depending on $A$ and on the ambient conditions at step $k$, but not otherwise on the sequence of updates that produced $A$. Write $\mathcal{D}_{A_k}:=\mathcal{D}_{A_k,k}$.
\end{assumption}

The indices $A,k$ allow both agent response and changes in ambient conditions. Tasks can be sampled under the current agent's distribution; the full distribution need not be known. A candidate's associated distribution is counterfactual before adoption.

\begin{definition}
\label{def:comparative-realized-improvement}
For $\tilde A=A_k\oplus\Delta C$, define improvement on the current agent's task distribution by
\[
\mathrm{Imp}^{\mathrm{cmp}}(\tilde A) := \mathbb{E}_{t\sim\mathcal{D}_{A_k}}\bigl[\mathrm{Adv}_{A_k}(\tilde A,t)\bigr],
\]
and let $J^{\sharp}_k(A):=\mathbb{E}_{t\sim\mathcal{D}_{A,k}}[V(A,t)]$ be the expected reward under the task distribution induced by $A$. Then $\mathrm{Imp}^{\mathrm{rea}}(\tilde A):=J^{\sharp}_k(\tilde A)-J^{\sharp}_k(A_k)$ compares the agents under their respective task distributions. Write $\varrho(\Delta C):=\mathrm{TV}(\mathcal D_{\tilde A,k},\mathcal D_{A_k,k})$ for the total-variation distance between these distributions.

\end{definition}

\begin{proposition}
\label{prop:comparative-improvement}
Under Assumption~\ref{ass:bounded-reward}, for any probability measure $\mathcal{D}$ on $\mathcal{T}_{\text{user}}$ with $Z_k\in(0,1)$ (the standing assumption of \S{}\ref{sec:qualified-modifications}), and with $L_{A_k},Z_k,\mathcal{D}_{F,k},\mathcal{D}_{R,k},D_R$ all constructed from $\mathcal{D}$ by Eq.~\eqref{eq:formula-3-6} and Eq.~\eqref{eq:formula-3-7},
\begin{equation}
\mathbb{E}_{t\sim\mathcal{D}}\bigl[\mathrm{Adv}_{A_k}(\tilde A,t)\bigr] \;\ge\; L_{A_k}(\tilde A)-(1-Z_k)\,D_R(\tilde A;A_k).
\label{eq:formula-H-1}
\end{equation}
In particular Eq.~\eqref{eq:formula-H-1} holds at $\mathcal{D}=\mathcal{D}_{A_k}$, whether or not Assumption~\ref{ass:agent-independent-tasks} does.
\end{proposition}

\begin{corollary}
\label{cor:realized-improvement}
Under Assumptions~\ref{ass:bounded-reward} and \ref{ass:agent-responsive-tasks},
\[
\bigl|\mathrm{Imp}^{\mathrm{rea}}(\tilde A)-\mathrm{Imp}^{\mathrm{cmp}}(\tilde A)\bigr|\le\varrho(\Delta C), \qquad\text{hence}\qquad \mathrm{Imp}^{\mathrm{rea}}(\tilde A)\ge L_{A_k}-(1-Z_k)D_R-\varrho(\Delta C).
\]
\end{corollary}

\begin{remark}\label{rem:responsive-task-substitution}
Under Assumption~\ref{ass:agent-responsive-tasks}, put
$\varrho_k=\mathrm{TV}(\mathcal D_{A_{k+1},k},\mathcal D_{A_k,k})$
and $\eta_k^{\mathrm{exo}}=\sup_A\mathrm{TV}(\mathcal D_{A,k+1},\mathcal D_{A,k})$.
For $\mathcal D_{\mathrm{user}}^{(k)}=\mathcal D_{A_k,k}$ the triangle inequality through $\mathcal D_{A_{k+1},k}$ yields $\eta_k\le\varrho_k+\eta_k^{\mathrm{exo}}$.
Thus the drift in Theorem~\ref{thm:drifting-task-guarantee} includes both terms. With no adopted modification the response at fixed environment state is zero; the environment may still change.

For a single-step post-response expected-gain guarantee, a valid response bound $b_k$ available before adoption gives the sufficient check
$\tau>\varepsilon_L+(1-\widehat Z_k+\varepsilon_Z)(\delta+\varepsilon_D)+b_k$.
If the bound is probabilistic, its failure risk is included in the joint experiment's risk budget.
\end{remark}

\subsubsection{Estimation under the current task distribution}
\label{app:section-H-1}

Under the sampling and coverage conditions of Appendix~\ref{app:proofs-evaluation}, the same estimators compare the agents under $\mathcal D_{A_k,k}$. Independent fresh rollouts evaluate both agents on the same sampled tasks, while the failure weights and retained-task distribution refer to the current agent, as required by Proposition~\ref{prop:comparative-improvement}.

Procedure~\ref{alg:finite-horizon-update} recollects the stored sample after every accepted modification as described in Remark~\ref{rem:stored-label-recollection}. At the next step, new records and fresh evaluations define $\widehat Z_{k+1}$, $\widehat L_{k+1}$, and $\widehat D_{R,k+1}$ relative to the successor's task distribution. A response bound available before adoption is separately needed for a post-response expected-improvement guarantee.

\subsubsection{Bounding the reward response after an update}
\label{app:section-H-2}

The partition bound below controls the effect of post-update task-distribution response on expected reward when the stated uniform reward-resolution condition holds.

Define $\mathcal W=\int V(\widetilde A,t)\,d(\mathcal D_{\widetilde A,k}-\mathcal D_{A_k,k})$. Under the uniform resolution condition of Lemma~\ref{lem:partition-resolution} in Appendix~\ref{app:observation-lag-estimation}, $|\mathcal W|\le\mathrm{TV}_{\mathcal F}+\zeta$.
These partition probabilities compare the two distributions at the same environment state after response. Lemma~\ref{lem:partition-drift-estimation} in the same appendix estimates their cell discrepancy; the uniform within-cell reward condition in Lemma~\ref{lem:partition-resolution} is separate. The samples observed after response do not themselves provide the advance bound $b_k$ required by Remark~\ref{rem:responsive-task-substitution}.

\begin{proof}[Proof of Proposition~\ref{prop:comparative-improvement}.] The pointwise decomposition underlying Lemma~\ref{lem:failure-decomposition} holds for any reference measure $\mathcal D$. Integrating it for $g=\operatorname{Adv}_{A_k}(\tilde A,\cdot)$ gives
\[
\mathbb E_{\mathcal D}[\operatorname{Adv}_{A_k}(\tilde A,\cdot)]
=L_{A_k}(\tilde A)+E_{A_k}(\tilde A),
\]
where $E_{A_k}(\tilde A)=(1-Z_k)\mathbb E_{\mathcal D_{R,k}}[\operatorname{Adv}_{A_k}(\tilde A,t)]$. Since $|\operatorname{Adv}_{A_k}(\tilde A,t)|$ integrates to $D_R(\tilde A;A_k)$ under $\mathcal D_{R,k}$, Jensen's inequality gives
\[
E_{A_k}(\tilde A)\ge-(1-Z_k)D_R(\tilde A;A_k).
\]
Combining the two displays proves Eq.~\eqref{eq:formula-H-1} for the arbitrary reference measure $\mathcal D$. \end{proof}

\begin{proof}[Proof of Corollary~\ref{cor:realized-improvement}.] Add and subtract $\mathbb{E}_{\mathcal{D}_{A_k}}[V(\tilde A,\cdot)]$, so $\mathrm{Imp}^{\mathrm{rea}}=\mathrm{Imp}^{\mathrm{cmp}}+\mathcal{W}$ with $\mathcal{W}:=\int V(\tilde A,\cdot)\,d(\mathcal{D}_{\tilde A,k}-\mathcal{D}_{A_k,k})$. The signed measure has total mass $0$ \textbf{on task space}, so $\mathcal{W}$ is exactly the object Lemma~\ref{lem:distribution-drift} bounds: $|\mathcal{W}|\le\mathrm{osc}(V(\tilde A,\cdot))\cdot\varrho\le\varrho$, since $V\in[0,1]$ by Assumption~\ref{ass:bounded-reward} and Eq.~\eqref{eq:formula-2-7}. Combine with Proposition~\ref{prop:comparative-improvement}. \end{proof}

\subsection{Evaluation and deployment task distributions}
\label{app:section-G-15}

\subsubsection{Definitions}
\label{app:section-G-15-1}

At step $k$, distinguish:
\begin{itemize}
\item $\mathcal{D}^{(k)}_{\text{dep}}$ --- the deployment task distribution at agent state $k$;
\item $\mathcal{D}^{(k)}_{\text{prod}}$ --- the base task distribution represented by the evaluation design, from which failure and retained distributions are formed using their stated weights.
\end{itemize}

Expected reward under the two task distributions:
\[
J^{(k)}_{\text{dep}}(A) := \mathbb{E}_{t \sim \mathcal{D}^{(k)}_{\text{dep}}}[V(A, t)], \quad J^{(k)}_{\text{prod}}(A) := \mathbb{E}_{t \sim \mathcal{D}^{(k)}_{\text{prod}}}[V(A, t)].
\]

Track deployment expected reward as $X_k:=J^{(k)}_{\mathrm{dep}}(A_k)$.

\begin{definition}[Evaluation--deployment distribution discrepancy]
\label{def:observation-lag-drift}
At step $k$, let $\mathcal C_k^{\mathrm{step}}$ be a class containing the current agent and every possible adopted successor in the joint experiment, and define
\[
\eta^{\mathrm{obs}}_k := \sup_{A\in\mathcal C_k^{\mathrm{step}}} \bigl| J^{(k)}_{\mathrm{dep}}(A) - J^{(k)}_{\mathrm{prod}}(A) \bigr|.
\]
\end{definition}

The inter-step drift of Definition~\ref{def:distribution-drift} compares task distributions at adjacent steps; $\eta_k^{\mathrm{obs}}$ compares the evaluation and deployment distributions at the same step. Agent response in Appendix~\ref{app:agent-responsive} is one possible source of distribution change.

Applying Lemma~\ref{lem:distribution-drift} to each agent and taking the supremum gives
\[
\eta^{\text{obs}}_k \leq \mathrm{TV}(\mathcal{D}^{(k)}_{\text{dep}}, \mathcal{D}^{(k)}_{\text{prod}}).
\]
This bound avoids enumerating agents. Appendix~\ref{app:observation-lag-estimation} estimates the probabilities of a declared finite partition and bounds $\eta_k^{\mathrm{obs}}$ under the uniform reward-resolution condition of Lemma~\ref{lem:partition-resolution}.

\subsubsection{One-step expected-reward bound}
\label{app:section-G-15-2}

\begin{lemma}
\label{lem:observation-lag-pursuit}
Under Assumption~\ref{ass:bounded-reward}, let $\mathcal D_{\mathrm{prod}}^{(k)}$ be fixed before evaluating the candidate agent, let $\mathsf{Acc}_k\subseteq\mathsf{Cand}_k\cap\mathsf{Ready}_k$ be the step-$k$ acceptance event of Theorem~\ref{thm:one-step-guarantee} evaluated on that measure. All probabilities and expectations below are under the same joint experiment, including generation, evaluation, updating, and any history-dependent task distributions. Let $I_k:=J_{\mathrm{prod}}^{(k)}(A_{k+1})-J_{\mathrm{prod}}^{(k)}(A_k)$ be improvement on the evaluation distribution and $\Delta_k$ the pre-acceptance quantity Eq.~\eqref{eq:formula-4-28}, set to zero outside $\mathsf{Cand}_k\cap\mathsf{Ready}_k$, and suppose
\[
\mathbb P(\mathsf{Acc}_k\cap\{I_k<\Delta_k\})\le\beta_{\mathrm{step},k},
\qquad
\chi_k:=\mathbb E[\mathbf1\{\mathsf{Acc}_k\}\Delta_k].
\]
If $\mathrm{TV}(\mathcal D_{\mathrm{dep}}^{(k+1)},\mathcal D_{\mathrm{dep}}^{(k)})\le\eta_k^{\mathrm{step}}$ almost surely, with nonnegative integrable $\eta_k^{\mathrm{step}}$, then
\begin{equation}
\mathbb E[X_{k+1}-X_k]\ge\chi_k-2\beta_{\mathrm{step},k}-\mathbb E[\eta_k^{\mathrm{step}}]-2\mathbb E[\eta_k^{\mathrm{obs}}].
\label{eq:formula-G-24}
\end{equation}
\end{lemma}

After separating inter-step distribution movement, the comparison between evaluation and deployment distributions is made once at $A_k$ and once at $A_{k+1}$. Each contributes at most $\eta_k^{\mathrm{obs}}$, giving the factor two. Deterministic drift bounds are a special case of Eq.~\eqref{eq:formula-G-24}.

\subsubsection{Cumulative bound under the two task distributions}
\label{app:section-G-15-3}

\begin{theorem}
\label{thm:observation-lag-cumulative}
Suppose the hypotheses of Lemma~\ref{lem:observation-lag-pursuit} hold at every step under consideration, and define
\begin{equation}
\nu_k:=\chi_k-2\beta_{\mathrm{step},k}-\mathbb E[\eta_k^{\mathrm{step}}]-2\mathbb E[\eta_k^{\mathrm{obs}}].
\label{eq:formula-G-25}
\end{equation}
Then, for any fixed integers $s\ge0$ and $K\ge1$,
\begin{equation}
\sum_{k=s}^{s+K-1}\nu_k
\le\mathbb E[X_{s+K}]-\mathbb E[X_s].
\label{eq:formula-G-26}
\end{equation}
\end{theorem}

The term $\chi_k$ is the expectation of the accepted-step contribution $\mathbf1\{\mathsf{Acc}_k\}\Delta_k$, which is zero on nonaccepted steps. If $\Delta_k\ge\gamma$ on $\mathsf{Acc}_k$, then $\chi_k\ge\gamma\,\mathbb P(\mathsf{Acc}_k)$ under the same joint experiment.

Appendix~\ref{app:observation-lag-estimation} gives a partition-based estimate of $\eta_k^{\mathrm{obs}}$ under its stated resolution and coverage conditions.

\begin{proof}[Proof of Lemma~\ref{lem:observation-lag-pursuit}.]
Write the tracked increment as
\[
\begin{aligned}
X_{k+1}-X_k
={}&\underbrace{J_{\mathrm{prod}}^{(k)}(A_{k+1})-J_{\mathrm{prod}}^{(k)}(A_k)}_{\mathrm{(I)}=I_k}\\
&+\underbrace{J_{\mathrm{dep}}^{(k+1)}(A_{k+1})-J_{\mathrm{dep}}^{(k)}(A_{k+1})}_{\mathrm{(II)}}\\
&+\underbrace{J_{\mathrm{dep}}^{(k)}(A_{k+1})-J_{\mathrm{prod}}^{(k)}(A_{k+1})}_{\mathrm{(III)}}
+\underbrace{J_{\mathrm{prod}}^{(k)}(A_k)-J_{\mathrm{dep}}^{(k)}(A_k)}_{\mathrm{(IV)}}.
\end{aligned}
\]
Let $\mathcal B_k:=\mathsf{Acc}_k\cap\{I_k<\Delta_k\}$. On $\mathcal B_k^c$, Definition~\ref{def:evolution-step} and the accepted-step bound give $I_k\ge\mathbf1\{\mathsf{Acc}_k\}\Delta_k$. On $\mathcal B_k$, bounded rewards give $I_k\ge-1$. The positive-gain check gives $\Delta_k>0$, its definition Eq.~\eqref{eq:formula-4-28} gives $\Delta_k\le\tau$, and $\widehat L_k\ge\tau$ in Eq.~\eqref{eq:formula-4-23} gives $\tau\le\widehat L_k\le1$ on acceptance. Thus $\Delta_k\le1$ there, and
\[
I_k\ge\mathbf1\{\mathsf{Acc}_k\}\Delta_k-2\mathbf1\{\mathcal B_k\},
\qquad
\mathbb E[I_k]\ge\chi_k-2\beta_{\mathrm{step},k}.
\]
Lemma~\ref{lem:distribution-drift} bounds term (II) by $-\eta_k^{\mathrm{step}}$, and Definition~\ref{def:observation-lag-drift} bounds terms (III) and (IV) together by $-2\eta_k^{\mathrm{obs}}$. Thus, pathwise,
\[
X_{k+1}-X_k\ge\mathbf1\{\mathsf{Acc}_k\}\Delta_k
-2\mathbf1\{\mathcal B_k\}-\eta_k^{\mathrm{step}}-2\eta_k^{\mathrm{obs}}.
\]
Taking expectations under the same joint distribution, with nonnegative integrable drift bounds, gives
\[
\mathbb E[X_{k+1}-X_k]\ge\chi_k-2\beta_{\mathrm{step},k}
-\mathbb E[\eta_k^{\mathrm{step}}]-2\mathbb E[\eta_k^{\mathrm{obs}}],
\]
which is Eq.~\eqref{eq:formula-G-24}. \end{proof}

\begin{proof}[Proof of Theorem~\ref{thm:observation-lag-cumulative}.] By Lemma~\ref{lem:observation-lag-pursuit}, for every $k$,
\[
\mathbb E[X_{k+1}-X_k]\ge\nu_k.
\]
Summing this inequality over $k=s,\ldots,s+K-1$ and using linearity of expectation gives
\[
\sum_{k=s}^{s+K-1}\nu_k
\le\mathbb E\!\left[\sum_{k=s}^{s+K-1}(X_{k+1}-X_k)\right]
=\mathbb E[X_{s+K}]-\mathbb E[X_s].
\]
This proves Eq.~\eqref{eq:formula-G-26}. \end{proof}

\subsection{Estimating the evaluation--deployment discrepancy}\label{app:observation-lag-estimation}

A finite partition estimates differences between the evaluation and deployment task distributions at a declared resolution. Lemma~\ref{lem:partition-resolution} states when these differences control expected reward uniformly over the agent class of Definition~\ref{def:observation-lag-drift}.

Fix in advance a finite measurable partition $\mathcal F=\{E_1,\ldots,E_m\}$ of $\mathcal T_{\mathrm{user}}$. For each source $\bullet\in\{P,D\}$, observe $n_\bullet$ i.i.d. tasks from $\mathcal D_{\mathrm{prod}}^{(k)}$ or $\mathcal D_{\mathrm{dep}}^{(k)}$, respectively, and record their cell frequencies. Independence between the two sources is not required.

\textbf{Estimator.} The $\mathcal{F}$-restricted empirical total variation:
\begin{equation}
\widehat{\mathrm{TV}}^{(k)}_{\mathcal{F}} := \tfrac{1}{2}\sum_{i=1}^{m}\Bigl|\,\widehat{\mathcal{D}}^{(k)}_{\text{dep}}(E_i) - \widehat{\mathcal{D}}^{(k)}_{\text{prod}}(E_i)\,\Bigr|,
\label{eq:formula-G-16}
\end{equation}
where $\widehat{\mathcal{D}}(E_i)$ are the empirical frequencies on the two samples.

\begin{lemma}
\label{lem:partition-drift-estimation}
Under this sampling scheme, let $\mathrm{TV}_{\mathcal F}:=\tfrac12\sum_i|\mathcal D_{\mathrm{dep}}^{(k)}(E_i)-\mathcal D_{\mathrm{prod}}^{(k)}(E_i)|$. With probability at least $1-\beta$,
\begin{equation}
\bigl|\widehat{\mathrm{TV}}^{(k)}_{\mathcal{F}} - \mathrm{TV}_{\mathcal{F}}\bigr| \;\le\; \varepsilon^{\text{obs}}(\beta) := \frac{1}{2}\sum_{\bullet\in\{P,D\}}\left(\sqrt{\frac{m}{n_\bullet}}+\sqrt{\frac{2\log(2/\beta)}{n_\bullet}}\right).
\label{eq:formula-G-17}
\end{equation}
\end{lemma}

Under the corresponding condition of Lemma~\ref{lem:partition-resolution}, define
\begin{equation}
\widehat{\eta}^{\text{obs}}_k := \widehat{\mathrm{TV}}^{(k)}_{\mathcal{F}} + \varepsilon^{\text{obs}}(\beta)+\zeta,
\label{eq:formula-G-18}
\end{equation}
with $\zeta=0$ in case (\ref{part:constant-cell-resolution}). On the coverage event of Lemma~\ref{lem:partition-drift-estimation}, $\eta_k^{\mathrm{obs}}\le\widehat\eta_k^{\mathrm{obs}}$.

\begin{lemma}
\label{lem:partition-resolution}
Let $\mathcal F$ be the partition above and let the agent class be the one used in Definition~\ref{def:observation-lag-drift}.
\begin{resultparts}
\item\label{part:constant-cell-resolution} If $V(A,\cdot)$ is constant on every cell for every agent $A$ in this class, then $\eta_k^{\mathrm{obs}}\le\mathrm{TV}_{\mathcal F}$.
\item\label{part:bounded-cell-resolution} If $\sup_{t,t'\in E_i}|V(A,t)-V(A,t')|\le\zeta$ for every agent $A$ in the class and every cell $E_i$, then $\eta_k^{\mathrm{obs}}\le\mathrm{TV}_{\mathcal F}+\zeta$.
\end{resultparts}
\end{lemma}

\begin{proof}[Proof of Lemma~\ref{lem:partition-resolution}.] Fix an agent $A$ in the stated class and put $f(t)=V(A,t)$. On each nonempty cell $E_i$, let $a_i=\inf_{t\in E_i}f(t)$, set $g(t)=a_i$, and write $h=f-g$. Bounded reward and the within-cell condition give $a_i\in[0,1]$ and $h(t)\in[0,\zeta]$; no infimum need be attained. Write $p_i=\mathcal D_{\mathrm{dep}}^{(k)}(E_i)$ and $q_i=\mathcal D_{\mathrm{prod}}^{(k)}(E_i)$. Since $\sum_i(p_i-q_i)=0$ and $a_i\in[0,1]$,
\[
\left|\sum_i a_i(p_i-q_i)\right|\le\tfrac12\sum_i|p_i-q_i|=\mathrm{TV}_{\mathcal F}.
\]
Both expectations of $h$ lie in $[0,\zeta]$, so their difference has absolute value at most $\zeta$. Combining the $g$ and $h$ parts and taking the supremum over the same agent class proves part~(\ref{part:bounded-cell-resolution}). In the constant-cell case $h=0$, giving part~(\ref{part:constant-cell-resolution}). \end{proof}

For dependent records, compare each source's entire record vector with its stated i.i.d. reference vector. If these joint-TV errors are bounded by $d_P,d_D$ and $\widetilde\beta:=\beta-d_P-d_D>0$, using $\widetilde\beta$ in Eq.~\eqref{eq:formula-G-17} gives total risk at most $\beta$.

Since $0\le\eta_k^{\mathrm{obs}}\le1$, coverage with failure probability at most $\beta$ also gives
$\mathbb E[\eta_k^{\mathrm{obs}}]\le\mathbb E[\min\{1,\widehat\eta_k^{\mathrm{obs}}\}]+\beta$
in the same joint experiment. Indeed, for the coverage event $\mathcal E$ and $B=\min\{1,\widehat\eta_k^{\mathrm{obs}}\}$, one has $\eta_k^{\mathrm{obs}}\le B$ on $\mathcal E$ and $\eta_k^{\mathrm{obs}}\le1$ always, so
$\mathbb E[\eta_k^{\mathrm{obs}}]\le\mathbb E[B\mathbf1_{\mathcal E}]+\mathbb P(\mathcal E^c)\le\mathbb E[B]+\beta$.

\begin{proof}[Proof of Lemma~\ref{lem:partition-drift-estimation}.] Write $p^\bullet$ for the true $\mathcal{F}$-marginal of source $\bullet$ and $\hat p^\bullet$ for its empirical counterpart. By the triangle inequality applied to the $\tfrac12\sum_i|\cdot|$ form,
\[
\bigl|\widehat{\mathrm{TV}}_{\mathcal{F}}-\mathrm{TV}_{\mathcal{F}}\bigr|\;\le\;\tfrac12\bigl(\Vert \hat p^{P}-p^{P}\Vert _1+\Vert \hat p^{D}-p^{D}\Vert _1\bigr),
\]
so it suffices to bound the $L^1$ error of each empirical multinomial. For the mean, Cauchy--Schwarz gives
\[
\mathbb{E}\Vert \hat p-p\Vert _1=\sum_i\mathbb{E}|\hat p_i-p_i|\le\sum_i\sqrt{\tfrac{p_i(1-p_i)}{n}}\le\tfrac{1}{\sqrt n}\sum_i\sqrt{p_i}\le\sqrt{\tfrac{m}{n}}.
\]
For the deviation, $\Vert \hat p-p\Vert _1$ is a function of the $n$ samples that changes by at most $2/n$ when one sample is altered (it moves two cells by $1/n$ each), so McDiarmid's inequality gives $\mathbb{P}\bigl[\Vert \hat p-p\Vert _1\ge\mathbb{E}\Vert \hat p-p\Vert _1+t\bigr]\le e^{-nt^2/2}$; setting this to $\beta/2$ yields $t=\sqrt{2\log(2/\beta)/n}$. A union bound over the two sources gives Eq.~\eqref{eq:formula-G-17} with probability $\ge1-\beta$. \end{proof}

\section{Reachability Estimation}\label{app:measurement}
\subsection{Uniform confidence bounds and diagnostic quantities}\label{app:proofs-measurement}\label{app:measurement-precision}
The following restates Theorem~\ref{thm:uniform-reachability-bounds}.

\begin{theoremrestatement}[Uniform confidence bounds for reachability]{thm:uniform-reachability-bounds}
Fix $(\mathcal H_k,\Xi_k=\xi)$ and draw $C_1,\ldots,C_{N_{\mathrm{meas}}}$ independently from $\pi_{k,\xi}$. In this conditional experiment, suppose all evaluation rectangles cover their pairs $(L_k(C_j),D_k(C_j))$ with simultaneous probability at least $1-\beta_{\mathrm{rect}}$. The target sets $\{(l,d)\in\mathbb R^2:l\ge\lambda,d\le\delta\}$, indexed by $(\lambda,\delta)$, form a VC class of constant dimension; let $r_{\mathrm{VC}}(N_{\mathrm{meas}},\beta_{\mathrm{emp}})$ be any valid uniform empirical-process radius for this class. Then, with probability at least $1-\beta_{\mathrm{rect}}-\beta_{\mathrm{emp}}$, the bounds Eq.~\eqref{eq:formula-6-8} hold simultaneously for every target $T$:

\[
\max\{0,\widehat P_{\mathrm{in}}(T)-r_{\mathrm{VC}}\}
\le P_k(T\mid\mathcal H_k,\xi)\le
\min\{1,\widehat P_{\mathrm{out}}(T)+r_{\mathrm{VC}}\}.
\]
\end{theoremrestatement}

\begin{proof}[Proof of Theorem~\ref{thm:uniform-reachability-bounds}.] Write $L_j=L_k(C_j)$ and $D_j=D_k(C_j)$. Let $I_{\mathrm{in}}(j,T):=\mathbf1\{\underline L_j\ge\lambda,\overline D_j\le\delta\}$ and $I_{\mathrm{out}}(j,T):=\mathbf1\{\overline L_j\ge\lambda,\underline D_j\le\delta\}$ denote the indicators averaged in Definition~\ref{def:measurement-indicators}. For a failed output, the true, inner, and outer indicators are all zero because its rectangle is $\{(0,0)\}$ and $\lambda>0$. On the simultaneous rectangle event, for every recorded output $j$ and every target $T$,
\[
I_{\mathrm{in}}(j,T)
\le \mathbf1\{L_j\ge\lambda,D_j\le\delta\}
\le I_{\mathrm{out}}(j,T).
\]
Consequently, on this one rectangle event, their sample averages satisfy
\[
\widehat P_{\mathrm{in}}(T)\le
\frac1{N_{\mathrm{meas}}}\sum_j\mathbf1\{L_j\ge\lambda,D_j\le\delta\}
\le\widehat P_{\mathrm{out}}(T)\qquad\text{for all }T.
\]
Conditional on $(\mathcal H_k,\Xi_k=\xi)$, the latent pairs $(L_j,D_j)$ are i.i.d.\ pushforwards of $C_j\sim\pi_{k,\xi}$. The target sets are lower-right orthants in $\mathbb R^2$, a VC class of constant dimension, so the empirical-process event gives
\[
\sup_T\left|\frac1{N_{\mathrm{meas}}}\sum_j
\mathbf1\{L_j\ge\lambda,D_j\le\delta\}-P_k(T\mid\mathcal H_k,\xi)\right|
\le r_{\mathrm{VC}}.
\]
On the intersection of these two events, the sample-average inequalities and uniform deviation give both bounds in Eq.~\eqref{eq:formula-6-8} for every $T$. A union bound allocates $\beta_{\mathrm{rect}}+\beta_{\mathrm{emp}}$; no union over targets is needed. \end{proof}

\paragraph{Precision.}
Before truncation to \([0,1]\), the generation-probability interval in Eq.~\eqref{eq:formula-6-8} has width
\[
\bigl(\widehat P_{\mathrm{out}}(T)-\widehat P_{\mathrm{in}}(T)\bigr)+2r_{\mathrm{VC}}.
\]
The first term is the fraction of all sampled terminal outputs whose rectangles leave qualification unresolved; it depends on evaluation precision and on the distribution of effects near the target boundary. The second is generation-sampling uncertainty. At fixed confidence a larger generation sample reduces the empirical radius, while requiring all evaluation intervals to contain their true effects with the stated probability can make these intervals wider.

Uniformity covers the specified targets, including targets selected after the data. The conditional and marginal scopes of task-stream measurement are stated in Proposition~\ref{prop:drawn-measurement-coverage} in Appendix~\ref{app:drawn-measurement}.

\begin{definition}
\label{def:certification-selection-probabilities}
For a candidate pool \(C_{1:N}\), conditional on \((\mathcal H_k,\xi)\), let \(G\) be a possibly pool-dependent validation rule, with \(G_i(T)=0\) when \(C_i=\bot\). The expected proportion of terminal outputs that are both qualified and validated is
\[
Q^G_{k,N}(T):=\mathbb E\!\left[\frac1N\sum_{i=1}^N
\mathbf 1\{C_i\in\mathcal P_k(T),G_i(T)=1\}\,\middle|\,\mathcal H_k,\xi\right].
\]
The probability that the pool contains at least one such modification is
\[
H^G_{k,N}(T):=\mathbb P\!\left[\exists i:C_i\in\mathcal P_k(T),G_i(T)=1\,\middle|\,\mathcal H_k,\xi\right].
\]
Write \(R^{G,\sigma}_{k,N}(T)\) for the conditional probability that selector \(\sigma\) selects a member of \(\mathcal P_k(T)\) from this pool, with the conventions of Definition~\ref{def:pool-selector}.
\end{definition}

For selectors restricted to validated modifications, $H-R$ measures the loss due to selection when a qualified validated modification is available. For conditionally i.i.d. generation, $P-Q$ is the expected fraction of all outputs that are qualified but not validated; it differs from the fraction whose qualification remains unresolved by the evaluation intervals. Shared evaluation data can make validation events dependent, so $Q$ alone does not determine $H$.

For one generated modification, write \(\rho_k^G(T):=\mathbb P[G(T)=1\mid\mathcal H_k,\xi]\) for its validation probability.

\begin{proposition}
\label{prop:certification-probability-bounds}
Fix \((\mathcal H_k,\Xi_k=\xi)\) and a target \(T\), and draw \(C\sim\pi_{k,\xi}\). Let \(G\) be the exact-target rule Eq.~\eqref{eq:formula-4-6} with \(N=1\). Suppose its evaluation rectangle covers \((L_k(C),D_k(C))\) with probability at least \(1-\beta_{\mathrm{val}}\) conditional on \((\mathcal H_k,\xi,C)\), and its total widths satisfy the pre-draw deterministic bounds Eq.~\eqref{eq:formula-4-5a} on every evaluation outcome, including any vacuous rectangle recorded on nonreturn. Let \(P_k^+\) be the probability assigned to the set in Eq.~\eqref{eq:formula-4-7} defined by those same bounds. Then
\begin{equation}
P_k^+(T\mid\mathcal H_k,\xi)-\beta_{\mathrm{val}}
\le \rho_k^G(T)\le
P_k(T\mid\mathcal H_k,\xi)+\beta_{\mathrm{val}}.
\label{eq:formula-6-1}
\end{equation}
\end{proposition}

\begin{proof}[Proof of Proposition~\ref{prop:certification-probability-bounds}.] Fix $(\mathcal H_k,\Xi_k=\xi)$ and suppress this conditioning. Let $\mathcal E$ be the event that the rectangle for $C$ covers $(L_k(C),D_k(C))$. On $\mathcal E$, the validation rule in Eq.~\eqref{eq:formula-4-6} validates only qualified modifications. The total-width bounds imply
\[
\underline L(C)\ge L_k(C)-\bar w_L,\qquad
\overline D(C)\le D_k(C)+\bar w_D.
\]
Hence every sampled member of the set in Eq.~\eqref{eq:formula-4-7} is validated on $\mathcal E$, giving
\[
\mathbf1\{C\in\mathcal P_k^+(T)\}\le G(T)
\le\mathbf1\{C\in\mathcal P_k(T)\}\qquad\text{on }\mathcal E.
\]
Using $\mathbb P(\mathcal E\mid C)\ge1-\beta_{\mathrm{val}}$ and integrating over $C$ yields
\[
(1-\beta_{\mathrm{val}})P_k^+
\le \rho_k^G
\le P_k+\beta_{\mathrm{val}}.
\]
Since $(1-\beta_{\mathrm{val}})P_k^+\ge P_k^+-\beta_{\mathrm{val}}$, Eq.~\eqref{eq:formula-6-1} follows. \end{proof}

A high validation probability lower-bounds reachability up to $\beta_{\mathrm{val}}$; a low value can reflect either low reachability or low probability of validating qualified modifications.

\subsection{Fixed-suite effects and rollout bounds}\label{app:section-I-3-1}\label{app:finite-look-measurement}\label{app:absolute-rollout-bound}\label{app:pooled-rollouts}\label{app:section-I-3-2}
\begin{definition}[Fixed evaluation suite]
\label{def:fixed-suite}
A fixed evaluation suite consists of a finite nonempty set $F^{\mathrm{ev}}\subseteq\mathcal{T}_{\text{user}}$ of failure tasks and a finite nonempty set $R^{\mathrm{ev}}\subseteq\mathcal{T}_{\text{user}}$ of retained tasks, both disjoint from $F_k^{\mathrm{gen}}$, the failure batch encoded into $t_{F_{k}^{\mathrm{gen}}}$ and used to generate the modification.
\end{definition}

\begin{definition}
\label{def:fixed-suite-target}
For a candidate agent $\tilde A=A_k\oplus\Delta C$ write
\begin{equation}
L^{\mathrm{ev}}(\tilde A) := \frac{1}{|F^{\mathrm{ev}}|}\sum_{t\in F^{\mathrm{ev}}}\mathrm{Adv}_{A_k}(\tilde A,t), \qquad D^{\mathrm{ev}}_R(\tilde A) := \frac{1}{|R^{\mathrm{ev}}|}\sum_{t\in R^{\mathrm{ev}}}\bigl|\mathrm{Adv}_{A_k}(\tilde A,t)\bigr|,
\label{eq:formula-6-2}
\end{equation}
and write $L^{\mathrm{ev}}(c),D_R^{\mathrm{ev}}(c)$ for these quantities at $\tilde A=A_k\oplus c$ when $c\in\mathcal M$. Set $L^{\mathrm{ev}}(\bot)=D_R^{\mathrm{ev}}(\bot)=0$. For a fixed-suite target $T^{\mathrm{ev}}=(\lambda^{\mathrm{ev}},\delta^{\mathrm{ev}})$ with $\lambda^{\mathrm{ev}}>0$ and $\delta^{\mathrm{ev}}\ge0$, let $\mathcal{P}^{\mathrm{ev}}_k(T^{\mathrm{ev}}) := \{c\in\mathcal M : L^{\mathrm{ev}}(c)\ge\lambda^{\mathrm{ev}} \wedge D^{\mathrm{ev}}_R(c)\le\delta^{\mathrm{ev}}\}$ and
\[
P^{\mathrm{ev}}_k(T^{\mathrm{ev}}\mid\mathcal H_k,\xi)
:=\pi_{k,\xi}(\mathcal{P}^{\mathrm{ev}}_k(T^{\mathrm{ev}})).
\]
\end{definition}

Write $w := |F^{\mathrm{ev}}|/(|F^{\mathrm{ev}}|+|R^{\mathrm{ev}}|)$ for the \textbf{failure fraction of the suite} --- a quantity chosen by the system designer --- and
\[
J^{\mathrm{ev}}(A) \;:=\; w\cdot\frac{1}{|F^{\mathrm{ev}}|}\sum_{t\in F^{\mathrm{ev}}}V(A,t) \;+\; (1-w)\cdot\frac{1}{|R^{\mathrm{ev}}|}\sum_{t\in R^{\mathrm{ev}}}V(A,t)
\]
for the agent's average reward on the fixed suite.

\begin{remark}\label{rem:suite-distribution-interpretation}
The two nonempty task sets are specified by the evaluation design; they may overlap. Their names do not require deterministic failure or success of the current agent on each task. If they overlap, $J^{\mathrm{ev}}$ is the average weighted by each task's group occurrences, not the average over their deduplicated union. The quantity $L^{\mathrm{ev}}$ is an unweighted failure-group mean, while $L=Z_k\mu_F$ includes the failure-rate weight. Each target and its reachability is defined using its respective effects.
\end{remark}
\begin{remark}\label{rem:adaptive-fixed-suite}
If one suite is reused for adoption across steps, later states and generation may depend on it. Conditional on the full history, context, suite and current generation material, Theorem~\ref{thm:uniform-reachability-bounds} applies to the suite-specific effects when conditional i.i.d. generation and simultaneous rectangle coverage hold. The resulting probability describes that realized state and suite.
\end{remark}
The suite effects in Definition~\ref{def:fixed-suite-target} also give the fixed-suite improvement guarantee of Theorem~\ref{thm:fixed-suite-guarantee} in Appendix~\ref{app:suite-distribution-comparison}.

\begin{assumption}[Cross-task rollout independence]\label{ass:cross-task-rollout-independence}
For each evaluated agent, the fresh rollouts used by the pooling and measurement constructions are mutually independent within and across tasks, conditional on that agent and the evaluation tasks.
\end{assumption}
The constructions below additionally use independent candidate and current-agent fresh-rollout samples. This is local to these rollout bounds; the per-task union-bound analysis of Lemma~\ref{lem:improvement-estimation} does not require cross-task independence.

\paragraph{Linear means and shared current-agent evaluation.}
Let each valid candidate use $m$ rollouts per task and the current agent use $m_0$. Under the stated independent-rollout design, Hoeffding applied to the two pooled means gives
\begin{equation}
\varepsilon_L^{\mathrm{ev}}
=\sqrt{\frac{\log(8/\beta_{\mathrm{cls}})}{2m|F^{\mathrm{ev}}|}}
+\sqrt{\frac{\log(8/\beta_{\mathrm{cls}})}{2m_0|F^{\mathrm{ev}}|}}.
\label{eq:formula-6-10}
\end{equation}
The candidate and current-agent pooled means each consist of independent bounded rewards, so their deviations are bounded by the two displayed summands. The triangle inequality bounds the error of their difference by their sum. The current-agent mean can be evaluated once and shared across candidates; its error must be included in each candidate's rectangle to ensure that all rectangles contain their corresponding true values with the stated probability.

\paragraph{Absolute differences.}
Write $\widehat\delta_j=\widehat V(\widetilde A,t_j)-\widehat V(A_k,t_j)$.
Let $\sigma_V^2$ bound both agents' per-task reward variances uniformly over all pairs evaluated by the design, and put
$v\ge |R^{\mathrm{ev}}|^{-1}\sum_j\operatorname{Var}(\widehat\delta_j)$.
Independence between the two agents' samples gives
\[
\operatorname{Var}(\widehat\delta_j)\le\sigma_V^2(1/m+1/m_0).
\]
Proposition~\ref{prop:rollout-jensen-bounds} and Cauchy--Schwarz yield
\[
0\le\mathbb E|\widehat\delta_j|-|\delta_j|
\le\sqrt{\operatorname{Var}(\widehat\delta_j)}
\le\sigma_V\sqrt{1/m+1/m_0}.
\]
Indeed, writing $\widehat\delta_j=\delta_j+(\widehat\delta_j-\delta_j)$, Jensen gives the nonnegative left-hand difference, while the triangle inequality and Cauchy--Schwarz bound it by $\mathbb E|\widehat\delta_j-\delta_j|\le\sqrt{\operatorname{Var}(\widehat\delta_j)}$. Moreover $\operatorname{Var}|X|\le\operatorname{Var}X$, since for an independent copy $X'$,
$\frac12\mathbb E(|X|-|X'|)^2\le\frac12\mathbb E(X-X')^2$.
Thus $v$ also bounds the average true variances of the absolute differences. Bernstein concentration for the independent fixed-task centered absolute differences bounds their random average around its expectation; adding the Jensen bias gives the valid radius
\begin{equation}
\varepsilon_D^{\mathrm{ev}}
=\sigma_V\sqrt{1/m+1/m_0}
+\sqrt{\frac{2v\log(8/\beta_{\mathrm{cls}})}{|R^{\mathrm{ev}}|}}
+\frac{7\log(8/\beta_{\mathrm{cls}})}{3|R^{\mathrm{ev}}|}.
\label{eq:formula-6-11}
\end{equation}
The usual bounded-variable Bernstein additive allowance $2\log(8/\beta_{\mathrm{cls}})/(3|R^{\mathrm{ev}}|)$ is no larger than the displayed allowance. The task set is fixed, so the concentration is over rollout noise around the average task-specific expectations. The per-task Jensen allowance decreases with the per-task rollout budgets; averaging more tasks alone does not reduce that allowance.

\paragraph{Finite evaluation levels.}
Fix rollout budgets $m_1<\cdots<m_K$ in advance and allocate $\beta_{\mathrm{cls}}/K$ to each. If $E_h$ is the coverage event at level $h$, then $\Pr(\bigcap_{h=1}^K E_h)\ge1-\sum_{h=1}^K\beta_{\mathrm{cls}}/K=1-\beta_{\mathrm{cls}}$. Any data-dependent choice of one of these levels, including stopping when its rectangle determines membership in a queried target, therefore retains coverage. In the displayed radii this allocation replaces $\log(8/\beta_{\mathrm{cls}})$ by $\log(8K/\beta_{\mathrm{cls}})$. Other targets retain the chosen rectangle's inner and outer indicators.

\subsection{Measurement from randomly sampled tasks}\label{app:drawn-measurement}
\begin{definition}
\label{def:drawn-measurement}
Fix a stored-sample size $n_R$ before drawing the candidate pool. After that pool has been frozen, collect a disjoint measurement sample $\{(t_j^R,\phi_j^R)\}_{j=1}^{n_R}$.

Its two label counts are $N_1:=\nobreak\sum_{j=1}^{n_R}\phi_j^R$ and $N_0:=n_R-N_1$.
The measurement sample is not used during generation and is never used for a validation decision. In the independent reference experiment, conditional on $(\mathcal H_k,\Xi_k=\xi)$ and the frozen candidate pool, the records are independent with $t_j^R\stackrel{\mathrm{iid}}{\sim}\mathcal D_{\mathrm{user}}$ and $\phi_j^R\mid t_j^R\sim\mathrm{Bernoulli}(\psi_{A_k}(t_j^R))$; the labels are conditionally independent across records. Write $(F^{\mathrm{dr}},R^{\mathrm{dr}})$ for its two subsets, where $F^{\mathrm{dr}}$ contains all $N_1$ records with $\phi^R_j=1$ and $R^{\mathrm{dr}}$ all $N_0$ records with $\phi^R_j=0$; thus $\widehat Z_k=N_1/n_R$. Conditional on the label vector, classification of valid modifications uses count-adaptive radii. If $N_1=0$ or $N_0=0$, the corresponding coordinate for a valid modification receives its full feasible interval --- $[-1,1]$ for $L$ and $[0,1]$ for $D$ --- and the procedure abstains from nonvacuous classification on that coordinate. Failed outputs retain their known singleton rectangles.
\end{definition}

The protocol fixes $n_R$ before sampling and uses every sampled record. Its content is drawn from the task stream; Definition~\ref{def:fixed-suite} instead fixes the suite's task content.

Fix $n_R$ before drawing the candidate pool, draw $C_1,\dots,C_{N_{\mathrm{meas}}}\stackrel{\mathrm{iid}}{\sim}\pi_{k,\xi}$ conditional on $(\mathcal H_k,\Xi_k=\xi)$, freeze that pool, and then evaluate its valid modifications on the disjoint measurement sample of Definition~\ref{def:drawn-measurement}. Failed outputs use the singleton rectangle $\{(0,0)\}$. For valid modifications, on $N_1,N_0>0$, use the count-adaptive radii
\begin{equation}
\varepsilon^{\mathrm{dr}}_L=\varepsilon_Z+(\hat Z_k+\varepsilon_Z)\bigl(\varepsilon^{\mathrm{pool}}_V+\varepsilon^{\mathrm{dr}}_F\bigr),
\qquad
\varepsilon^{\mathrm{dr}}_D=\varepsilon^{\mathrm{Jen}}_V+\varepsilon^{\mathrm{dr}}_R,
\label{eq:formula-6-13}
\end{equation}
where
\[
\varepsilon^{\mathrm{dr}}_F(N_1)=\sqrt{\frac{2\log(8/\beta_{\mathrm{cls}})}{N_1}},
\quad
\varepsilon^{\mathrm{dr}}_R(N_0)=\sqrt{\frac{\log(8/\beta_{\mathrm{cls}})}{2N_0}},
\quad
\varepsilon_Z=\sqrt{\frac{\log(8/\beta_{\mathrm{cls}})}{2n_R}}.
\]
Here $\varepsilon^{\mathrm{pool}}_V$ is the radius in Eq.~\eqref{eq:formula-6-10}, evaluated at the realized count, and $\varepsilon^{\mathrm{Jen}}_V=\sigma_V\sqrt{1/m+1/m_0}$ is the non-pooling Jensen bias of Eq.~\eqref{eq:formula-6-11}. If either count is zero, use a vacuous rectangle on the corresponding coordinate.

\begin{proposition}
\label{prop:drawn-measurement-coverage}
Under this measurement setup, if $N_{\mathrm{meas}}\beta_{\mathrm{cls}}\le\beta_{\mathrm{rect}}$, then:
\begin{enumerate}
\item when, conditional on $(\mathcal H_k,\Xi_k=\xi)$, the frozen pool is independent of the audit records and those records are independent draws from the task-and-label law of Definition~\ref{def:drawn-measurement}, the uniform confidence bounds
\begin{equation}
\max\{0,\widehat P_{\mathrm{in}}(T)-r_{\mathrm{VC}}\}
\le P_k(T\mid\mathcal H_k,\xi)\le
\min\{1,\widehat P_{\mathrm{out}}(T)+r_{\mathrm{VC}}\}
\qquad\text{for every }T
\label{eq:formula-6-14}
\end{equation}
hold conditionally on $(\mathcal H_k,\Xi_k=\xi)$ with probability at least $1-\beta_{\mathrm{rect}}-\beta_{\mathrm{emp}}$;
\item if the joint law of history, context, frozen pool, all audit records and their evaluation marks is within total variation $\beta_{\mathrm{mix}}$ of a reference experiment with the same pre-audit marginal and the independent audit law in the first item, then Eq.~\eqref{eq:formula-6-14}, with the random realized $(\mathcal H_k,\Xi_k)$ on its right-hand side, holds with joint marginal probability at least $1-\beta_{\mathrm{rect}}-\beta_{\mathrm{emp}}-\beta_{\mathrm{mix}}$.
\end{enumerate}
\end{proposition}

The first item is also supplied by a reset followed by an independent task-and-label generator for every record, or by an $m$-dependent stream with all selected records spaced beyond its dependence range and separated from the pre-audit sigma-field. A reset alone removes dependence on pre-reset history but does not make later records mutually independent. For these stream-based examples, fix a deterministic task cutoff $b$: the pre-audit history, context, current agent, and frozen pool are determined by tasks through $b$ and auxiliary randomness independent of the task stream. Audit indices after $b$ are fixed in advance, or their complete schedule is sampled independently of the task stream; the schedule is not changed in response to observed audit tasks or labels. Conditional on the selected tasks and frozen agent, stored labels are generated independently by the Bernoulli kernels in Definition~\ref{def:drawn-measurement}, using randomness separate from the task stream and pre-audit information; fresh evaluation uses its prescribed independent rollout kernels. The second item averages over history, context, pool, and audit data. For a $\beta$-mixing stream with fixed audit indices $s_1<\cdots<s_{n_R}$ and first gap $g_0=s_1-b$, the joint-TV product comparison permits
\begin{equation}
\beta_{\mathrm{mix}}
\le \beta(g_0)+\sum_{j=1}^{n_R-1}\beta(s_{j+1}-s_j).
\label{eq:formula-6-14b}
\end{equation}
If all these gaps are at least $\ell$, this is at most $n_R\beta(\ell)$; the same bound holds for an independent random schedule satisfying those gaps, by conditioning on the schedule. If a reset starts a fresh stationary stream independent of the pre-audit information, it removes only the first term and gives $(n_R-1)\beta(\ell)$.

\begin{proof}[Proof of Proposition~\ref{prop:drawn-measurement-coverage}.] We first prove the result in the conditional reference experiment.  There it is legitimate to condition on $(\mathcal H_k,\Xi_k=\xi)$, which fixes $A_k$, $\pi_{k,\xi}$, $Z_k$, $\mathcal{D}_{F,k}$, and $\mathcal{D}_{R,k}$.  The joint-TV statement is transferred after this conditional proof, under the joint distribution.

\emph{Step 1: the reference sample supplies random-size samples from both conditional distributions.} In the reference experiment, $t^R_j\stackrel{\mathrm{iid}}{\sim}\mathcal{D}_{\text{user}}$, and the stored label satisfies $\phi^R_j\mid t^R_j\sim\mathrm{Bernoulli}(\psi_{A_k}(t^R_j))$ by Eq.~\eqref{eq:formula-3-1}. For measurable $B$,
\[
\mathbb{P}\bigl(t^R_j\in B\mid\phi^R_j=1\bigr)
=\frac{\int_B\psi_{A_k}\,d\mathcal{D}_{\text{user}}}{\int\psi_{A_k}\,d\mathcal{D}_{\text{user}}}
=\frac{\int_B\psi_{A_k}\,d\mathcal{D}_{\text{user}}}{Z_k}
=\mathcal{D}_{F,k}(B)
\]
by Eq.~\eqref{eq:formula-3-2} and Eq.~\eqref{eq:formula-3-7}, and the $\phi^R_j=0$ computation gives $\mathcal{D}_{R,k}$ likewise. Both require $Z_k\in(0,1)$, the standing assumption of \S{}\ref{sec:qualified-modifications}. Conditional on the full label vector, the tasks factorise: the $N_1$ failure-labelled tasks are i.i.d.\ from $\mathcal{D}_{F,k}$, the $N_0$ retained-labelled tasks are i.i.d.\ from $\mathcal{D}_{R,k}$, and the two groups are independent. Moreover $N_1\sim\mathrm{Binomial}(n_R,Z_k)$ and $\widehat Z_k=N_1/n_R$.
\begin{equation}
\Pr(t_j^R\in dt\mid\phi_j^R=1)=\mathcal D_{F,k}(dt),\qquad
\Pr(t_j^R\in dt\mid\phi_j^R=0)=\mathcal D_{R,k}(dt).
\label{eq:formula-6-12}
\end{equation}

\emph{Step 2: which implementations inherit the reference experiment.} An independent audit generator, including one started after a reset, gives Step 1 directly. An $m$-dependent stream sampled at gaps exceeding $m$ and separated by such a gap from the pre-audit sigma-field also gives exact independence. A reset without independent post-reset sampling does not.

For ordinary $\beta$-mixing under the cutoff and advance-scheduling conditions above, compare the whole joint experiment to the reference measure retaining the same pre-audit history, context and pool, with conditionally independent reference audit tasks. Successive product-measure comparisons over the initial and record-to-record gaps give Eq.~\eqref{eq:formula-6-14b} for fixed indices. For an independent random schedule, condition on its realized indices and average the comparison; a common lower gap bound gives the displayed $n_R\beta(\ell)$ allowance. Applying the common independent marking and fresh-evaluation kernels does not increase total variation. Each reference event therefore transfers to the actual joint experiment with one loss of at most $\beta_{\mathrm{mix}}$. A supplied whole-window coupling with that defect is also sufficient. This argument gives marginal event bounds.

Conditional on the reference label vector, the random-size groups earn the $q_F=1$ radii: no stopping rule has selected their sizes, and no blocking factor appears inside either conditional mean.  If a group is empty, assigning the corresponding coordinate its full feasible interval makes the rectangle valid without claiming informativeness.

\emph{Step 3: per-modification radii.} Failed outputs have deterministic rectangle coverage. Fix a valid drawn modification $\Delta C_i\in\mathcal M$; condition on the frozen pool in the stated reference experiment, so $\mathrm{Adv}_{A_k}(\tilde A_i,\cdot)$ is a fixed function on the measurement sample. Conditional on the label vector:

\emph{(a) $L$.} On $N_1>0$, $\frac1{N_1}\sum_{t\in F^{\mathrm{dr}}}\mathrm{Adv}_{A_k}(\tilde A_i,t)$ is, conditional on the label vector, an i.i.d.\ mean of a range-$2$ variable, so Hoeffding gives $\varepsilon^{\mathrm{dr}}_F(N_1)=\sqrt{2\log(8/\beta_{\mathrm{cls}})/N_1}$. There are $mN_1$ candidate rewards and $m_0N_1$ current-agent rewards. Given the tasks and frozen pool, their independent signed weighted sum has squared range lengths totaling $a+c$, where $a=(mN_1)^{-1}$ and $c=(m_0N_1)^{-1}$. For $h=\log(8/\beta_{\mathrm{cls}})$ its direct two-sided Hoeffding radius $\sqrt{h(a+c)/2}$ has failure risk $\beta_{\mathrm{cls}}/4$ and is at most the retained radius $\sqrt{ha/2}+\sqrt{hc/2}=\varepsilon_V^{\mathrm{pool}}$. Thus the entire pooled difference uses one $\beta_{\mathrm{cls}}/4$ event. Hence $|\hat\mu^{(i)}_F-\mu^{(i)}_F|\le\varepsilon^{\mathrm{pool}}_V+\varepsilon^{\mathrm{dr}}_F(N_1)$. Since $\hat L^{(i)}=\hat Z_k\hat\mu^{(i)}_F$ and $L^{(i)}=Z_k\mu^{(i)}_F$, the decomposition in the proof of Lemma~\ref{lem:improvement-estimation} applies verbatim --- $|\hat Z_k-Z_k|\,|\hat\mu_F|\le\varepsilon_Z$ against $|\hat\mu_F|\le1$, plus $Z_k|\hat\mu_F-\mu_F|$ bounded on $\mathcal{E}_Z$ by $(\hat Z_k+\varepsilon_Z)(\varepsilon^{\mathrm{pool}}_V+\varepsilon^{\mathrm{dr}}_F(N_1))$ --- which is the first display of Eq.~\eqref{eq:formula-6-13}.  For $N_1=0$ the declared vacuous interval for $L$ covers by construction.

\emph{(b) $D_R$.} On $N_0>0$, $\hat D^{(i)}=\frac1{N_0}\sum_{t\in R^{\mathrm{dr}}}|\hat\delta_i(t)|$ is, conditional on the label vector, an i.i.d.\ mean of a range-$1$ variable over $\mathcal{D}_{R,k}$, which by Definition~\ref{def:retained-reward-change} is the measure $D_R$ is defined against. Proposition~\ref{prop:rollout-jensen-bounds}(\ref{part:jensen-direction})--(\ref{part:jensen-magnitude}) gives $|\delta_i(t)|\le\mathbb{E}|\hat\delta_i(t)|\le|\delta_i(t)|+\sigma_V\sqrt{1/m+1/m_0}$ pointwise, so $\mathbb{E}[\hat D^{(i)}]\in[D^{(i)}_R,\,D^{(i)}_R+\varepsilon^{\mathrm{Jen}}_V]$, and Hoeffding around that mean gives $\varepsilon^{\mathrm{dr}}_R(N_0)=\sqrt{\log(8/\beta_{\mathrm{cls}})/(2N_0)}$. Hence $|\hat D^{(i)}-D^{(i)}_R|\le\varepsilon^{\mathrm{Jen}}_V+\varepsilon^{\mathrm{dr}}_R(N_0)$, the second display of Eq.~\eqref{eq:formula-6-13}.  For $N_0=0$ the vacuous interval for $D_R$ covers by construction.

\emph{Step 4: union, confidence bounds, and transfer.} The failure-task mean, the complete pooled rollout difference, the retained absolute-difference mean, and the failure-rate estimate each have failure risk at most $\beta_{\mathrm{cls}}/4$. The first three statements are integrated over the random label vector; the failure-rate event is bounded in the original reference sampling experiment. Hence each candidate rectangle has failure risk at most $\beta_{\mathrm{cls}}$. A union over the $N_{\mathrm{meas}}$ modification-wise rectangle events costs at most $N_{\mathrm{meas}}\beta_{\mathrm{cls}}\le\beta_{\mathrm{rect}}$; the shared $\widehat Z_k$ event need only be counted once, so this allocation is conservative.  On their intersection, the inside/true/outside indicator containment of Theorem~\ref{thm:uniform-reachability-bounds} holds simultaneously for all targets.  The latent pairs are i.i.d.\ pushforwards of $C_i\sim\pi_{k,\xi}$, so the VC event supplies the uniform empirical radius $r_{\mathrm{VC}}$ at error probability $\beta_{\mathrm{emp}}$.  This proves the first claim conditionally in each of the exact reference implementations.  In the ordinary-mixing case, transfer this \emph{joint event} through the joint-TV comparison of Step 2, subtracting $\beta_{\mathrm{mix}}$ once.  The result is marginal coverage of the random conditional reachability, not a pointwise-in-history guarantee. \end{proof}

\subsection{Additional fixed-suite guarantees}\label{app:suite-distribution-comparison}
The fixed-suite validation rule validates a modification when $\hat L^{\mathrm{ev}}\ge\tau$ and $\hat D_R^{\mathrm{ev}}\le\delta$. Let $\mathcal E_{\mathrm{suite}}$ be the joint event that, for every modification evaluated by this rule, $|\hat L^{\mathrm{ev}}-L^{\mathrm{ev}}|\le\varepsilon_L^{\mathrm{ev}}$ and $|\hat D_R^{\mathrm{ev}}-D_R^{\mathrm{ev}}|\le\varepsilon_D^{\mathrm{ev}}$. The following result uses radii for which $\mathbb P(\mathcal E_{\mathrm{suite}})\ge1-\beta_{\mathrm{gate}}$.
\begin{theorem}[Fixed-suite guarantee]
\label{thm:fixed-suite-guarantee}
For every candidate agent \(\widetilde A\),
\begin{equation}
J^{\mathrm{ev}}(\widetilde A)-J^{\mathrm{ev}}(A_k)
\ge wL^{\mathrm{ev}}(\widetilde A)-(1-w)D_R^{\mathrm{ev}}(\widetilde A).
\label{eq:formula-6-3}
\end{equation}
On $\mathcal E_{\mathrm{suite}}$, every modification validated by this validation rule satisfies
\begin{equation}
J^{\mathrm{ev}}(\tilde A)>J^{\mathrm{ev}}(A_k) \qquad\text{whenever}\qquad \tau \;>\; \varepsilon^{\mathrm{ev}}_L+\frac{1-w}{w}\bigl(\delta+\varepsilon^{\mathrm{ev}}_D\bigr).
\label{eq:formula-6-4}
\end{equation}
\end{theorem}

\begin{proof}[Proof of Theorem~\ref{thm:fixed-suite-guarantee}.] Splitting the suite average by Eq.~\eqref{eq:formula-6-2},
\[
J^{\mathrm{ev}}(\tilde A)-J^{\mathrm{ev}}(A_k)=w\,L^{\mathrm{ev}}+(1-w)\cdot\frac{1}{|R^{\mathrm{ev}}|}\sum_{t\in R^{\mathrm{ev}}}\mathrm{Adv}_{A_k}(\tilde A,t),
\]
and the second sum is bounded below by $-\frac{1}{|R^{\mathrm{ev}}|}\sum|\mathrm{Adv}|=-D^{\mathrm{ev}}_R$, giving Eq.~\eqref{eq:formula-6-3}. On the validated event $\hat L^{\mathrm{ev}}\ge\tau$ and $\hat D^{\mathrm{ev}}_R\le\delta$; on $\mathcal E_{\mathrm{suite}}$, $L^{\mathrm{ev}}\ge\tau-\varepsilon^{\mathrm{ev}}_L$ and $D^{\mathrm{ev}}_R\le\delta+\varepsilon^{\mathrm{ev}}_D$. Hence Eq.~\eqref{eq:formula-6-3} gives
\[
J^{\mathrm{ev}}(\tilde A)-J^{\mathrm{ev}}(A_k)
\ge w(\tau-\varepsilon^{\mathrm{ev}}_L)-(1-w)(\delta+\varepsilon^{\mathrm{ev}}_D),
\]
which is positive precisely under Eq.~\eqref{eq:formula-6-4}. \end{proof}

\section{Experimental Settings and Supplementary Results}
\label{app:experiments}

\subsection{DS-1000 repeated evaluation and diagnosis}
\label{app:section-I-3-4}

\paragraph{Setting.}
We study persistent Python solver-harness modifications on DS-1000 under three generation processes: independent generation, revision with label-free traces, and revision with outcome feedback. Each modification edits the solver harness while the LLM remains fixed. Label-free revision uses generated solutions and runtime status without correctness labels; outcome-feedback revision additionally receives pass/fail labels on generation tasks. Generation has access to the public task prompt and these traces, but not hidden answers, validation records, or audit results.

We use the API model \texttt{gpt-5.4-mini-2026-03-17}. Generation uses temperature $0.8$ and an 8192-token response limit; the seed harness makes one solver call at temperature $0$ with a 4096-token limit. Modifications may change solver prompts, decoding parameters, extraction, call count, and orchestration, so candidates are evaluated under their resulting execution policies rather than an equalized inference budget.

The task assignment was fixed before inspecting outcomes. Of the 1000 DS-1000 tasks, 516 used in the pilot studies are excluded; 480 of the remaining 484 are assigned once, with library composition proportional to the remaining task bank. Each of four contexts contains 8 generation tasks, 8 low-budget validation tasks, 48 high-budget validation tasks, and 56 independent audit tasks. For each context and generation process, six independent pools each contain four terminal outputs. The same pools are assessed using both validation budgets. The primary target is $T_0=(0.05,0.50)$, and the prespecified target grid is $\{0.05,0.10\}\times\{0.25,0.50\}$.

We repeat the audit using three fresh valid runs for every current-agent--task and valid-candidate--task pair, keeping the candidates and task assignments fixed. The original current-agent outcomes determine the failure and retained groups; they are not included in the new three-run means. Within each context, all candidates share the same three-run current-agent reference. The repeated audit contains 48,720 valid outcomes: $4\times56\times3$ for the current agent and $286\times56\times3$ for valid candidates. Wrong answers count as valid failures; infrastructure-invalid runs are replenished according to outcome-independent validity criteria. Two generated outputs are invalid and remain in the generation denominators with zero target-membership and validation indicators.

\paragraph{Repeated audit estimates.}
Fix one context and a valid candidate. Let $F$ and $R$ be the failure and retained groups determined by the original current-agent outcomes, with $n=|F|+|R|=56$ and $w=|F|/n$. For task $t$, write $X_{tr},Y_{tr}\in\{0,1\}$ for the success outcomes of the current and candidate agents on fresh run $r\in\{1,2,3\}$, and define their means $\bar x_t=3^{-1}\sum_{r=1}^3X_{tr}$ and $\bar y_t=3^{-1}\sum_{r=1}^3Y_{tr}$. Equation~\eqref{eq:ds1000-repeated-effects} gives the point estimates. Conditional on the fixed tasks, agents, and original groups, let $p_t$ and $q_t$ denote the two agents' success probabilities. The expected-reward effects are
\begin{equation}
L=\frac1n\sum_{t\in F}(q_t-p_t),\qquad
D=\frac1{|R|}\sum_{t\in R}|q_t-p_t|,\qquad
\Delta=\frac1n\sum_t(q_t-p_t).
\label{eq:ds1000-fixed-effects}
\end{equation}
Thus $L=wL^{\mathrm{ev}}$ and $D=D_R^{\mathrm{ev}}$ for the fixed-suite effects in Definition~\ref{def:fixed-suite-target}, Appendix~\ref{app:section-I-3-1}. The target $L\ge 0.05$ uses the weighted contribution; its unweighted equivalent is $L^{\mathrm{ev}}\ge 0.05/w$. The confidence analysis assumes fresh runs are independent within and across tasks and between the two agents, conditional on the fixed tasks, agents, and original records. Candidates may share current-agent observations; the joint bounds below do not require independence between candidates.

\paragraph{Original validation and adoption decisions.}
Validation retains its original data and rule. For a valid candidate, let $x_i,y_i\in\{0,1\}$ be the original current-agent and candidate outcomes on validation task $i$. With $n_{\mathrm v}$ validation tasks and $n_{R,\mathrm v}=\sum_i x_i$, the rule uses the observed improvement proportion $n_{\mathrm v}^{-1}\sum_i(1-x_i)y_i$ and retained-loss proportion $n_{R,\mathrm v}^{-1}\sum_i x_i(1-y_i)$. If $n_{R,\mathrm v}=0$, the retained-loss point value is zero and its interval is $[0,1]$. Clopper--Pearson intervals are applied to these proportions; a Hoeffding interval of range width two is also computed for mean reward change. The familywise risk is $0.05$ per four-output pool, divided by three times the output count, including invalid outputs. At target $T=(\lambda,\delta)$, the rule validates a modification when the improvement lower endpoint reaches $\lambda$ and the retained-loss upper endpoint is at most $\delta$. Among modifications passing the high-budget rule, selection maximizes observed mean reward change on the validation tasks. Repeated audit results do not alter these recorded decisions.

\paragraph{Point summaries and uncertainty.}
In the summaries below, write $G$ for the experimental validation rule and $\sigma$ for its selector based on the recorded validation scores. For output $j$ in pool $b$, let
\begin{align*}
Y_{bj}(T)&=\mathbf1\{\mathrm{valid}_{bj},\
\widehat L_{bj}\ge\lambda,\
\widehat D_{bj}\le\delta\},\\
A^G_{bj}(T)&=\mathbf1\{\underline L^G_{bj}\ge\lambda,\
\overline D^G_{bj}\le\delta\},
\end{align*}
where $\widehat L_{bj},\widehat D_{bj}$ are the repeated-audit estimates, and $\underline L^G_{bj},\overline D^G_{bj}$ are the original validation endpoints under rule $G$. Both indicators are zero for invalid outputs. In a context--process cell with $B=6$ pools and $M=4B=24$ outputs, compute
\begin{align*}
\widehat P_{\rm pt}(T)&=\frac1M\sum_{b,j}Y_{bj}(T),
&\widehat Q^G(T)&=\frac1M\sum_{b,j}Y_{bj}(T)A^G_{bj}(T),\\
\widehat H^G(T)&=\frac1B\sum_b\mathbf1\{\max_jY_{bj}(T)A^G_{bj}(T)=1\},
&\widehat R^{G,\sigma}(T)&=\frac1B\sum_bY_{b,\sigma(b)}(T),
\end{align*}
where the selected-output indicator is zero if no output is selected. The fraction of pools containing an output meeting the audit point target replaces $Y_{bj}A^G_{bj}$ by $Y_{bj}$ in the definition of $\widehat H^G$. Each reported summary averages the four context values.

For the empirical summary intervals, 5000 hierarchical-bootstrap replicates resample contexts, pools within context, and outputs within pool; each replicate recomputes the pool event and selection. We report 90\% percentile intervals using seed 20260815. Validation-budget comparisons use the same resampled pools and outputs for both budgets. These intervals describe variation in the empirical summaries obtained from the three-run means, rather than confidence bounds for the true generation probability. Resampling an all-zero summary produces $[0,0]$, which does not bound an unseen event probability.

\paragraph{Candidate-level diagnosis.}
In a retrospective analysis, we select all candidates whose original audit improvement lower bound is at least 0.05 and whose original retained-loss upper bound is at most 0.50. This selects three outcome-feedback candidates, labeled A, B, and C in Table~\ref{tab:ds1000-audit-diagnosis}. The original audit intervals use the original binary-outcome proportions with risk $0.05$ per context--process group of 24 outputs, divided by three times the output count. For this specified screen based on the original records, we use the fresh runs to calculate nine one-sided bounds: a lower bound for $L$, an upper bound for $D$, and a lower bound for $\Delta$ for each candidate. Each endpoint receives error probability $a=0.05/9$.

To account for different success probabilities across tasks, consider a sum $K$ of $N$ independent Bernoulli variables with mean success probability $\bar p$. Hoeffding's binomial comparison \citep[Theorem~4]{hoeffding1956distribution} gives
\[
\Pr(K\ge k)\le\Pr\{\operatorname{Bin}(N,p_0)\ge k\}
\quad\text{when }\bar p\le p_0\text{ and }k\ge Np_0+1.
\]
Writing $B^{-1}(a;s,t)$ for the $a$-quantile of the $\operatorname{Beta}(s,t)$ distribution, inversion gives a lower confidence bound for $\bar p$,
\begin{equation}
b(k,N,a)=
\begin{cases}
0,&k=0,\\
\min\{B^{-1}(a;k,N-k+1),(k-1)/N\},&k>0.
\end{cases}
\label{eq:ds1000-bernoulli-bound}
\end{equation}
The cap $(k-1)/N$ enforces the comparison's integer condition; an upper bound is $1-b(N-k,N,a)$, obtained by applying the same construction to $N-K$. These bounds allow different Bernoulli probabilities rather than requiring identically distributed observations.

For a task group $S$, put $h_S=3|S|$ and $K_S=\sum_{t\in S,r}\{Y_{tr}+(1-X_{tr})\}$. This is a sum of $2h_S$ independent Bernoulli variables with mean success probability $(1+|S|^{-1}\sum_{t\in S}(q_t-p_t))/2$. With $S=F$ and $S=F\cup R$, respectively, the lower bounds are
\[
\underline L=w\{2b(K_F,2h_F,a)-1\},\qquad
\underline\Delta=2b(K_{F\cup R},6n,a)-1.
\]
For retained tasks, pair current-agent and candidate runs by the fixed run index and let $K_D=\sum_{t\in R,r}|Y_{tr}-X_{tr}|$. Since $|q_t-p_t|\le\mathbb E|Y_{tr}-X_{tr}|$, an upper bound on the mean disagreement probability is also an upper bound on $D$:
\[
\overline D=1-b(3|R|-K_D,3|R|,a).
\]
A union bound over the nine endpoints gives at least 95\% joint coverage for this fixed screening procedure, conditional on the original records and the stated independence model. Shared current-agent observations across the three candidates do not affect this guarantee.

\begin{table}[H]
\centering
\small
\setlength{\tabcolsep}{4pt}
\caption{\textbf{Repeated-audit effects for the diagnostic candidates.} Point estimates use three fresh runs per agent and task on the fixed audit set. The nine one-sided bounds allocate total error probability 0.05 across the three candidates selected from the original audit records; displayed endpoints are rounded to the nearest four decimal places. Candidate A meets $L\ge 0.05$, $D\le 0.50$, and $\Delta>0$ under these bounds.}
\label{tab:ds1000-audit-diagnosis}
\begin{tabular*}{\linewidth}{@{\extracolsep{\fill}}lrrrrrr@{}}
\toprule
Candidate & $\widehat L$ & Lower bound & $\widehat D$ & Upper bound & $\widehat\Delta$ & Lower bound\\
\midrule
A & $0.1726$ & $0.0551$ & $0.1176$ & $0.2765$ & $0.1964$ & $0.0556$\\
B & $0.1071$ & $-0.0113$ & $0.0980$ & $0.3012$ & $0.1012$ & $-0.0402$\\
C & $0.1369$ & $0.0266$ & $0.1515$ & $0.3653$ & $0.1607$ & $0.0195$\\
\bottomrule
\end{tabular*}
\end{table}

For candidate A, the failure group has 39 tasks: across 117 fresh runs per agent, the current agent succeeds 12 times and the candidate 41 times. The retained group has 17 tasks and 6 disagreements among 51 paired runs. Across all 168 runs per agent, the current agent succeeds 56 times and the candidate 89 times. Bounds stated as inequalities in Section~\ref{sec:empirical-studies} are rounded outward to four decimal places to preserve the guarantees.

Candidate A's original high-budget validation improvement estimate is 0.1458, but its lower endpoint is 0.0378, below 0.05; the retained-loss upper endpoint is 0.4196, below 0.50. Its pool adopts no modification. B and C also fail the original validation rule. Thus repeated evaluation supports A's expected-reward improvement on the fixed audit set, and the original decision record identifies the improvement condition that prevented adoption. Validation and audit use different task sets; the example diagnoses this recorded decision and does not estimate reachability over the full user-task distribution.

\paragraph{Supplementary results.}
At $T_0$, the repeated-audit point estimates meet the target for $37/96$, $25/96$, and $33/96$ outputs under independent generation, label-free revision, and outcome-feedback revision. The numbers of pools containing such an output are $16/24$, $13/24$, and $16/24$. The high-budget rule validates two independently generated outputs, both meeting the audit point target and both selected from their respective pools. Table~\ref{tab:ds1000-pqhr-full} gives the full summaries.

\begin{table}[H]
\centering
\small
\setlength{\tabcolsep}{4pt}
\caption{\textbf{DS-1000 summaries from repeated audit estimates at $T_0=(0.05,0.50)$.} Target membership uses the mean of three fresh runs per agent and audit task; validation retains the original high-budget decisions. Brackets give 90\% hierarchical-bootstrap intervals for the empirical summaries.}
\label{tab:ds1000-pqhr-full}
\begin{tabular*}{\linewidth}{@{\extracolsep{\fill}}lcccc@{}}
\toprule
Generation process & $\widehat P_{\rm pt}$ & $\widehat Q$ & $\widehat H$ & $\widehat R$ \\
\midrule
Independent & $0.385$ & $0.021$ & $0.083$ & $0.083$ \\
& $[0.135,0.646]$ & $[0.000,0.073]$ & $[0.000,0.167]$ & $[0.000,0.167]$ \\
Label-free revision & $0.260$ & $0.000$ & $0.000$ & $0.000$ \\
& $[0.031,0.479]$ & $[0.000,0.000]$ & $[0.000,0.000]$ & $[0.000,0.000]$ \\
Outcome-feedback revision & $0.344$ & $0.000$ & $0.000$ & $0.000$ \\
& $[0.125,0.563]$ & $[0.000,0.000]$ & $[0.000,0.000]$ & $[0.000,0.000]$ \\
\bottomrule
\end{tabular*}
\end{table}

At the stricter target $T_1=(0.10,0.25)$, estimated target membership is $8/96$, $4/96$, and $14/96$, with fractions $0.083\ [0.021,0.156]$, $0.042\ [0.000,0.094]$, and $0.146\ [0.021,0.302]$. The corresponding numbers of pools are $8/24$, $4/24$, and $8/24$; $\widehat Q$, $\widehat H$, and $\widehat R$ are zero for all three processes. Figure~\ref{fig:ds1000-pqhr} presents the context-specific estimates and their comparison with validation outcomes.

\begin{figure}[htbp]
\centering
\includegraphics[width=\textwidth]{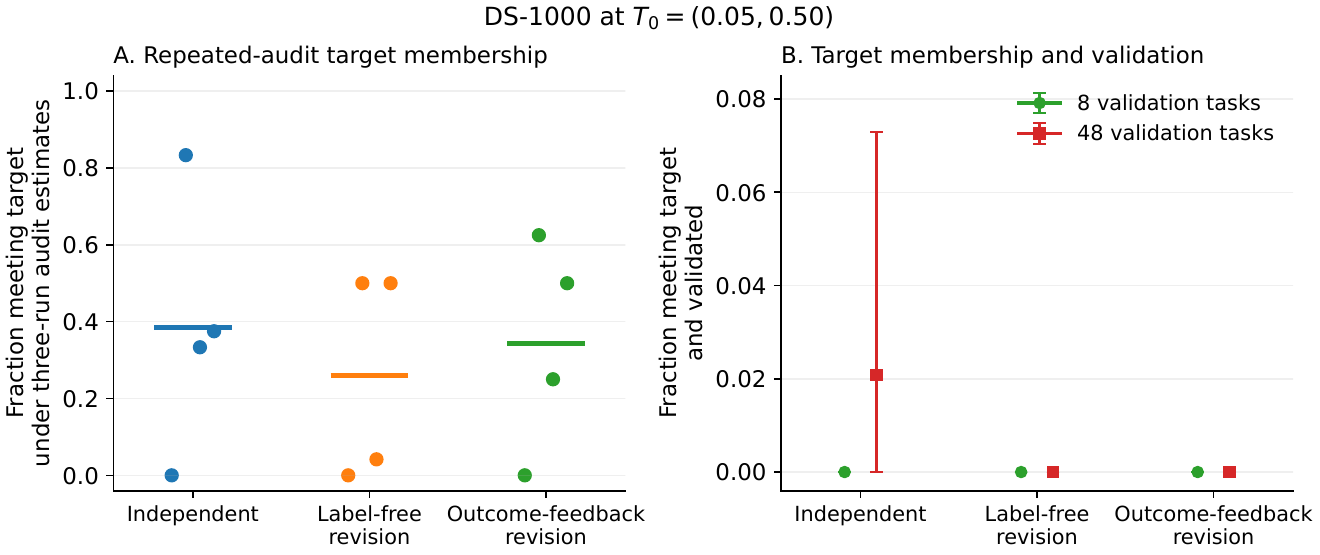}
\caption{\textbf{Repeated evaluation and validation budgets on DS-1000.} At $T_0=(0.05,0.50)$, \textbf{(A)} shows each context's fraction of outputs meeting the target under three-run audit estimates, with a horizontal line for the mean; \textbf{(B)} shows the fraction also validated using 8 or 48 validation tasks. Error bars give 90\% hierarchical-bootstrap intervals for the empirical summaries.}
\label{fig:ds1000-pqhr}
\end{figure}

\section{Further Directions for Harness Self-Evolution}
\label{app:scope-open-problems}

\paragraph{Generation and structured updates.}
The candidate-pool analysis permits complete generation runs with internal feedback and repair (Appendices~\ref{app:reachability-scope} and~\ref{app:proofs-generation}). Adapting later runs to earlier pool outputs or evaluation results raises the further problem of tracking their probabilities of producing qualified modifications. For fixed repair and finite assembly, estimating those output probabilities and exploiting dependence among assembled outputs could sharpen the available bounds for concrete search procedures (Appendix~\ref{app:section-G-6}).

After adoption, the modified harness affects user-task performance and the generation and effects of later modifications. Structural relations between these changes could support lower bounds on subsequent reachability under specified update mechanisms; the unrestricted-update construction in Appendix~\ref{app:later-step-proof} identifies why such relations matter.

\paragraph{Validation and measurement.}
The unresolved fraction in the reachability estimate depends on evaluation precision and on candidate effects near the improvement and retained-change requirements (Appendix~\ref{app:measurement-precision}). Quantifying this dependence could inform allocation between generating candidates and evaluating their effects. A related question is the validation probability of a given qualified modification at fixed error risk and evaluation budget, as its improvement and retained change vary relative to those requirements.

Variance-adaptive bounds and pooled rollout estimates (Appendices~\ref{app:section-D} and~\ref{app:pooled-rollouts}) provide starting points for allocating task samples and repeated rollouts under a fixed budget while controlling error risk and the validation probability of qualified candidates. The same allocation question extends to dependent task samples under the separation and coverage conditions in Appendix~\ref{app:proofs-evaluation}.

Repeated use of a fixed evaluation suite can make later agent states and generation depend on that suite. The conditional guarantee for the current state and suite is described in Remark~\ref{rem:adaptive-fixed-suite}. Extending these measurements to fresh tasks requires control of the difference between suite-specific effects and effects under the intended task distribution, with sampling and representativeness conditions matched to that claim. Allocating data between evolution and independent evaluation remains part of this question.

\paragraph{Retention requirements and application conditions.}
Tail constraints control average absolute expected-reward changes within every group having at least a specified probability under the retained-task distribution (Appendix~\ref{app:tail-retention}). Their implementation with finite rollouts and dependent task samples calls for corresponding estimation guarantees and evaluation requirements.

Applications also require ways to obtain the reward--utility discrepancy bound used in Appendix~\ref{app:reward-fidelity} and to model how user-task demand responds to agent changes. Partition-frequency estimates combined with reward-resolution conditions bound the expected-reward effect of task-distribution differences (Appendix~\ref{app:section-H-2}); predicting the response before adoption remains an application-modeling problem.

\end{document}